\documentclass{article}

\usepackage[preprint,nonatbib]{neurips_2020}

\usepackage{url}
\usepackage[hidelinks]{hyperref}
\usepackage{amsmath,amsthm,amssymb,amsbsy,bm}
\usepackage{paralist}
\usepackage{xcolor}
\usepackage{color}
\usepackage{graphicx}
\graphicspath{{./figs/}}
\usepackage{algorithm}
\usepackage{algorithmic}
\usepackage{comment}
\usepackage{multirow}
\usepackage{wrapfig}
\usepackage{booktabs}

\usepackage{fancyhdr}
\usepackage{cite}
\usepackage{cleveref}
\usepackage{subcaption}

\usepackage[page]{appendix}

\usepackage{array}
\newcommand{\PreserveBackslash}[1]{\let\temp=\\#1\let\\=\temp}
\newcolumntype{C}[1]{>{\PreserveBackslash\centering}p{#1}}
\newcolumntype{R}[1]{>{\PreserveBackslash\raggedleft}p{#1}}
\newcolumntype{L}[1]{>{\PreserveBackslash\raggedright}p{#1}}

\newtheorem{thm}{Theorem}
\newtheorem{lemma}{Lemma}

\newtheorem*{obs}{Observation}

\newtheorem{theorem}{Theorem}

\theoremstyle{remark}
\newtheorem{remark}{Remark}

\newcommand{\R}{\mathbb{R}}

\newcommand{\e}{\begin{equation}}
\newcommand{\ee}{\end{equation}}
\newcommand{\en}{\begin{equation*}}
\newcommand{\een}{\end{equation*}}
\newcommand{\eqn}{\begin{eqnarray}}
\newcommand{\eeqn}{\end{eqnarray}}
\newcommand{\bmat}{\begin{bmatrix}}
\newcommand{\emat}{\end{bmatrix}}

\DeclareMathAlphabet\mathbfcal{OMS}{cmsy}{b}{n}

\newcommand{\dif}{\operatorname{d}}

\newcommand{\mb}{\mathbf}
\newcommand{\mc}{\mathcal}

\newcommand{\bb}{\mathbb}

\newcommand{\vct}[1]{\mathbf{#1}}
\newcommand{\mtx}[1]{\mathbf{#1}}

\newcommand{\T}{\mathrm{T}}

\newcommand{\trace}{\operatorname{trace}}

\newcommand{\rank}{\operatorname{rank}}

\newcommand{\sign}{\operatorname{sign}}

\def \st {\operatorname*{s.t.\ }}

\newcommand{\wh}{\widehat}

\newcommand{\wt}{\widetilde}

\newcommand{\norm}[2]{\left\| #1 \right\|_{#2}}

\newcommand{\abs}[1]{\left| #1 \right|}

\newcommand{\parans}[1]{\left(#1\right)}
\newcommand{\innerprod}[2]{\left\langle #1,  #2 \right\rangle}

\newcommand{\calA}{\mathcal{A}}

\newcommand{\calL}{\mathcal{L}}

\newcommand{\vb}{\vct{b}}

\newcommand{\vg}{\vct{g}}
\newcommand{\vh}{\vct{h}}

\newcommand{\vo}{\vct{o}}

\newcommand{\vr}{\vct{r}}
\newcommand{\vs}{\vct{s}}

\newcommand{\vy}{\vct{y}}

\newcommand{\vnu}{\vct{\nu}}

\newcommand{\vxi}{\vct{\xi}}

\newcommand{\vzero}{\vct{0}}
\newcommand{\vone}{\vct{1}}

\newcommand{\mA}{\mtx{A}}

\newcommand{\mG}{\mtx{G}}
\newcommand{\mH}{\mtx{H}}

\newcommand{\mO}{\mtx{O}}

\newcommand{\mS}{\mtx{S}}

\newcommand{\mU}{\mtx{U}}

\newcommand{\mX}{\mtx{X}}
\newcommand{\mY}{\mtx{Y}}
\newcommand{\mZ}{\mtx{Z}}
\newcommand{\mGamma}{\mtx{\Gamma}}

\newcommand{\mLambda}{\mtx{\Lambda}}

\newcommand{\mId}{{\bf I}}

\newcommand{\mzero}{{\bf 0}}

\graphicspath{{./figs/}}

\newlength{\imgwidth}
\newboolean{twoColVersion}
\setboolean{twoColVersion}{false}
\newcommand{\twoCol}[2]{\ifthenelse{\boolean{twoColVersion}} {#1} {#2} }

\usepackage{enumitem}
\hypersetup{
    colorlinks=true,%
    citecolor=blue,%
    filecolor=blue,%
    linkcolor=blue,%
    urlcolor=blue
}

\newcommand{ \Brac }[1]{\left\lbrace #1 \right\rbrace}
\newcommand{ \brac }[1]{\left[ #1 \right]}
\newcommand{ \paren }[1]{ \left( #1 \right) }

\newcommand{\myparagraph}[1]{\smallskip\noindent\textbf{#1.}}

\newcommand\blfootnote[1]{%
	\begingroup
	\renewcommand\thefootnote{}\footnote{#1}%
	\addtocounter{footnote}{-1}%
	\endgroup
}

\title{Robust Recovery via Implicit Bias of Discrepant Learning Rates for Double Over-parameterization}

\author{%
Chong You$^\dagger$ \;\;\;\;\;\;\;\;\;
Zhihui Zhu$^\ddagger$ \;\;\;\;\;\;\;
Qing Qu$^\sharp$ \;\;\;\;\;\;\;
Yi Ma$^\dagger$
\\
\vspace*{-0.05in} 
\\
$^\dagger$Department of EECS, University of California, Berkeley\\
$^\ddagger$Department of Electrical and Computer Engineering, University of Denver\\
$^\sharp$Center for Data Science, New York University\\
}

\begin{document}

\maketitle
\blfootnote{\hspace*{-1.4mm}$^*$The first two authors contributed equally to this work.}

\begin{abstract}
Recent advances have shown that implicit bias of gradient descent on over-parameterized models enables the recovery of low-rank matrices from linear measurements, even with no prior knowledge on the intrinsic rank. In contrast, for \emph{robust} low-rank matrix recovery from \emph{grossly corrupted} measurements, over-parameterization leads to overfitting without prior knowledge on both the intrinsic rank and sparsity of corruption. This paper shows that with a \emph{double over-parameterization} for both the low-rank matrix and sparse corruption, gradient descent with {\em discrepant learning rates} provably recovers the underlying matrix even without prior knowledge on neither rank of the matrix nor sparsity of the corruption. We further extend our approach for the robust recovery of natural images by over-parameterizing images with deep convolutional networks. Experiments show that our method handles different test images and varying corruption levels with a single learning pipeline where the network width and termination conditions do not need to be adjusted on a case-by-case basis. Underlying the success is again the implicit bias with discrepant learning rates on different over-parameterized parameters, which may bear on broader applications.

\end{abstract}


\section{Introduction}\label{sec:intro}

Learning \emph{over-parameterized models}, which have more parameters than the problem's intrinsic dimension, is becoming a crucial topic in machine learning \cite{soudry2018implicit,vaskevicius2019implicit,zhao2019implicit,wu2020hadamard,oymak2019overparameterized,gunasekar2017implicit,li2018algorithmic,arora2019implicit,belabbas2020implicit,gunasekar2018implicit,gidel2019implicit}. While classical learning theories suggest that over-parameterized models tend to \emph{overfit} \cite{alpaydin2020introduction}, recent advances showed that an optimization algorithm may produce an \emph{implicit bias} that regularizes the solution with desired properties. This type of results has led to new insights and better understandings on gradient descent for solving several fundamental problems, including logistic regression on linearly separated data \cite{soudry2018implicit}, compressive sensing \cite{vaskevicius2019implicit,zhao2019implicit}, sparse phase retrieval \cite{wu2020hadamard}, nonlinear least-squares \cite{oymak2019overparameterized}, low-rank (deep) matrix factorization \cite{gunasekar2017implicit,li2018algorithmic,arora2019implicit,belabbas2020implicit}, and deep linear neural networks \cite{gunasekar2018implicit,gidel2019implicit}, etc.

Inspired by these recent advances \cite{soudry2018implicit,vaskevicius2019implicit,zhao2019implicit,wu2020hadamard,oymak2019overparameterized,gunasekar2017implicit,li2018algorithmic,arora2019implicit,belabbas2020implicit,gunasekar2018implicit,gidel2019implicit}, in this work we present a new type of practical methods for \emph{robust} recovery of \emph{structured} signals via model over-parameterization. In particular, we aim to learn an unknown signal $\mb X_\star \in \bb R^{n \times n}$ from its \emph{grossly corrupted} linear measurements
\begin{align}\label{eqn:problem}
    \mb y \;=\; \mc A(\mb X_\star) \;+\; \mb s_\star,
\end{align}
where the linear operator $\calA (\cdot):\R^{n\times n}\mapsto \R^m$, and $\mb s_\star \in \bb R^m $ is a (sparse) corruption vector. Variants of the problem appear ubiquitously in signal processing and machine learning \cite{candes2011robust,weller2015undersampled,li2016low,duchi2019solving,chi2019nonconvex,li2020nonconvex}.

\myparagraph{Robust recovery of low-rank matrices}
Low-rank matrix recovery has broad applications in face recognition \cite{candes2011robust} (where self-shadowing, specularity, or saturation in brightness can be modeled as outliers), video surveillance \cite{candes2011robust} (where the foreground objects are usually modeled as outliers) and beyond. A classical method for low-rank matrix recovery is via 
\emph{nuclear norm}\footnote{Nuclear norm 
(the tightest convex envelope to matrix rank) is defined as the sum of singular values.} minimization.
Such a method is provably correct under certain incoherent conditions \cite{candes2011robust,chandrasekaran2011rank}. However, minimizing nuclear norm involves expensive computations of singular value decomposition (SVD) of large matrices \cite{lin2010augmented} (when $n$ is large), which prohibits its application to problem size of practical interest. 


The computational challenge has been addressed by recent development of matrix factorization (MF) methods \cite{jain2017non,chi2019nonconvex}. 
Such methods are based on parameterizing the signal $\mb X \in \bb R^{n \times n}$ via the factorization $\mb X= \mb U \mb U^\top$ \cite{burer2003nonlinear}, and solving the associated nonconvex optimization problem with respect to (\emph{w.r.t.}) $\mb U\in \bb R^{n\times r}$, where $r$ is the rank of $\mb X_\star$. Since with $r \ll n$ the problem is of much smaller size, it leads to more scalable methods while still enjoys guarantees for correctness \cite{bhojanapalli2016global,ge2016matrix,zhu2018global,chi2019nonconvex}. However, the effectiveness of such \emph{exact-parameterization} based method hinges on the exact information of the rank of $\mb X_\star$ and the percentage of corrupted entries in $\mb s_\star$ (see \Cref{fig:RPCA-overview}), which are usually \emph{not} available \emph{a priori} in practice.



\begin{figure}
	\begin{subfigure}{0.47\linewidth}
		\centerline{
			\includegraphics[width=\linewidth]{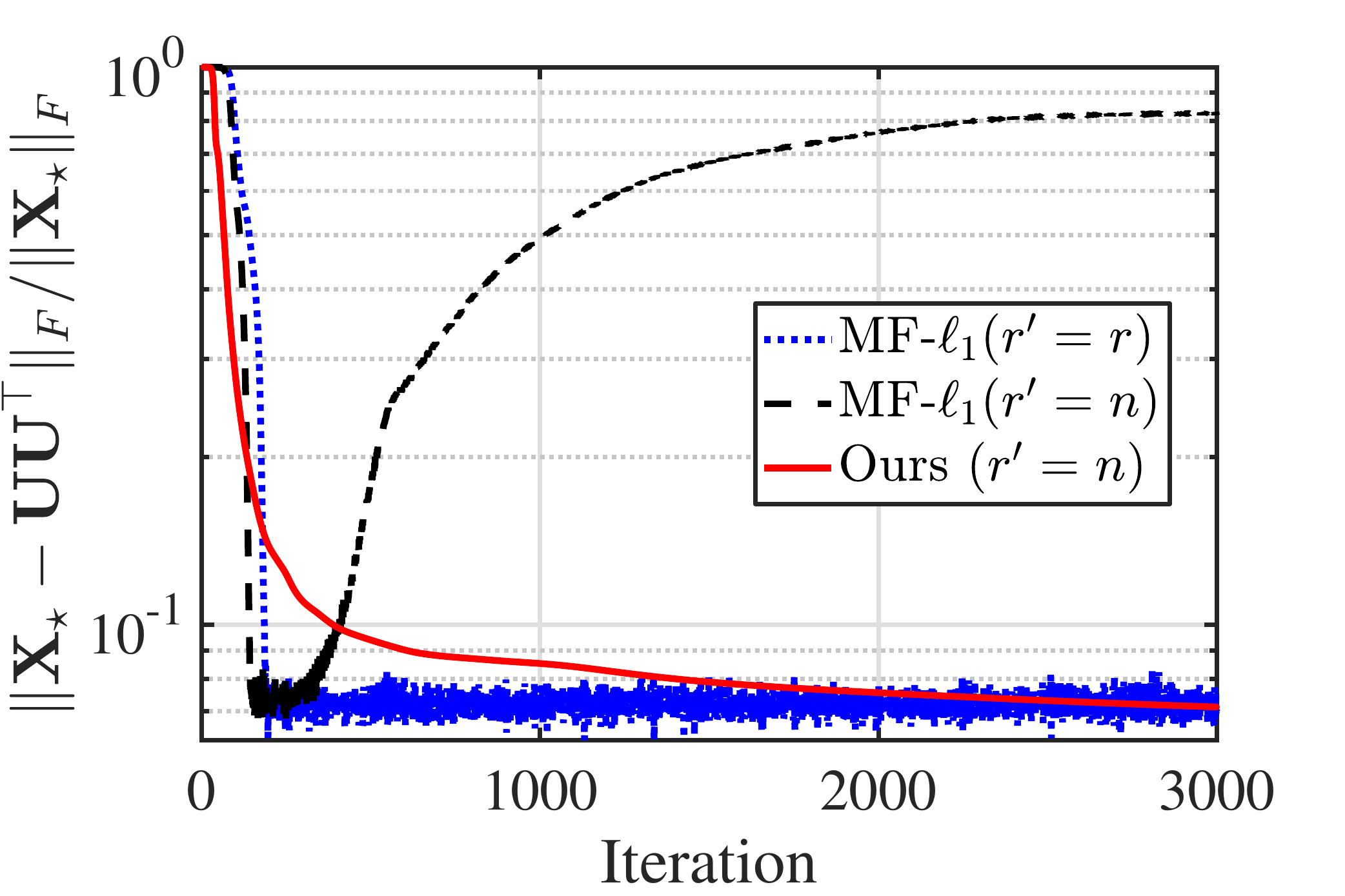}}
		\caption{Robust Matrix Recovery}\label{fig:RPCA-overview}
\end{subfigure}
\hfill
	\begin{subfigure}{0.47\linewidth}
		\centerline{\includegraphics[width=\linewidth]{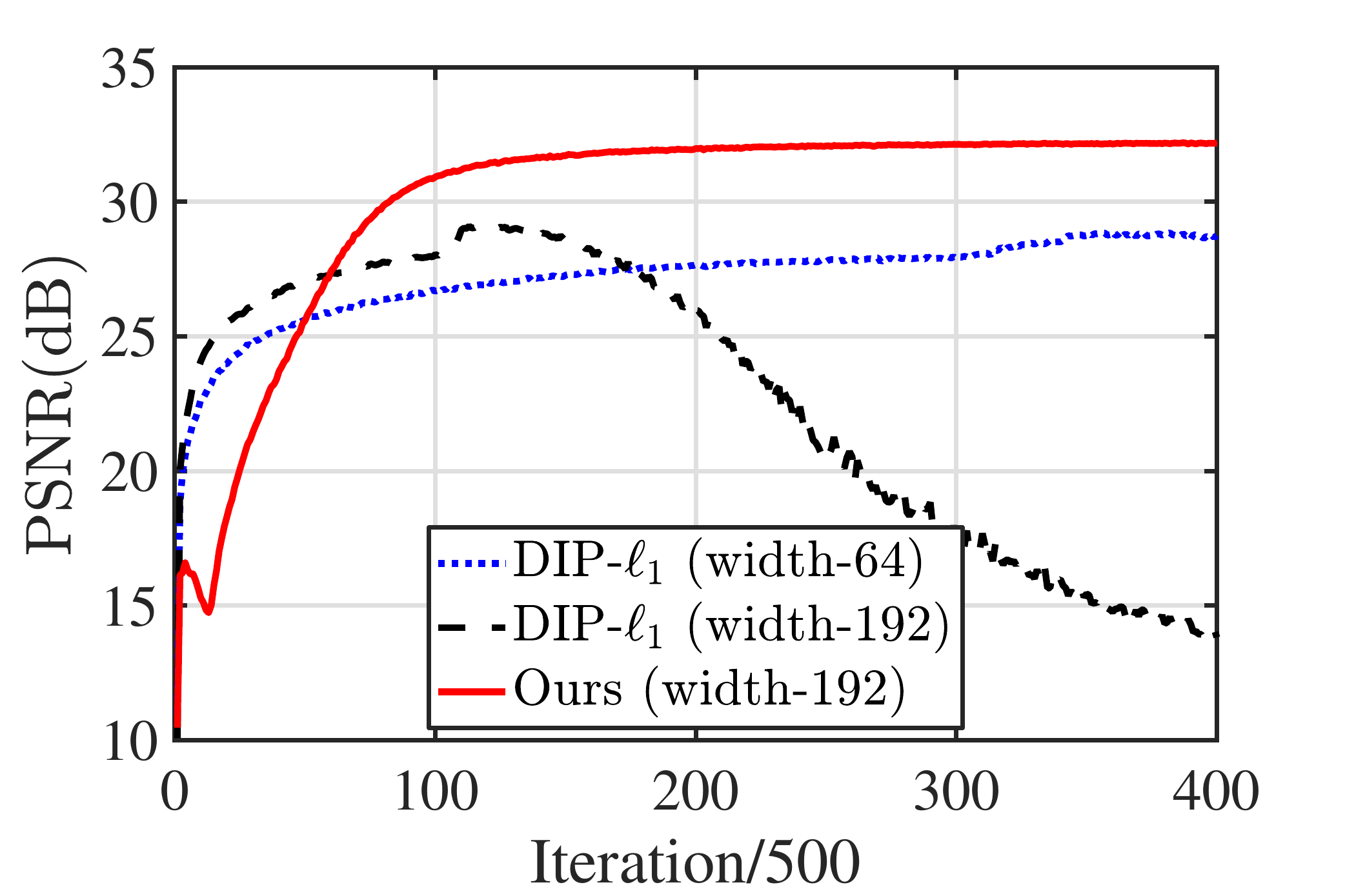}}
\caption{Robust Image Recovery}\label{fig:DIP-overview}
	\end{subfigure}\vspace{-5pt}
	\caption{\footnotesize \textbf{Learning curves for robust recovery of low-rank matrices (a) and natural images (b).} (a) Classical matrix factorization (MF) method with $\ell_1$ penalty requires exact parameterization (left blue), otherwise over-parameterization leads to overfitting without early termination (left black). 
	(b) Previous DIP method with $\ell_1$ penalty requires tuning network width (with width\;=\;$64$, right blue) or early termination (with width\;=\;$192$, right black). For both problems, our method  with double over-parameterization achieves superior performance and requires neither early termination nor precise parameterization (red curves). 
	}\label{fig:result-overview}
\vspace{-14pt}
\end{figure}


\myparagraph{Robust recovery of natural images} 
Robust recovery of natural images is often considered as a challenging task due to the lack of \emph{universal} mathematical modeling for natural images. While sparse and low-rank based methods have been demonstrated to be effective for years \cite{mairal2009non,dabov2007image,dabov2009bm3d,gu2014weighted,xu2018trilateral,lecouat2020fully}, the state-of-the-art performance is obtained by learned priors with deep convolutional neural networks. 
Such methods operate by end-to-end training of neural networks from pairs of corrupted/clean images \cite{zhang2017beyond,yue2019variational,zamir2020cycleisp}, and often \emph{cannot} effectively handle test cases with corruption type and noise level that are different from those of the training data. 

Recently, this challenge has been partially addressed by the so-called \emph{deep image prior} (DIP) \cite{ulyanov2018deep}, which has shown impressive results on many image processing tasks. The method is based on parameterizing an image $\mb X$ by a deep network $\mb X = \phi(\bm \theta)$, which turns out to be a flexible prior for modeling the underlying distributions of a variety of natural images. 
The network $\phi(\bm \theta)$ has a U-shaped architecture and may be viewed as a multi-layer, nonlinear extension of the low-rank matrix factorization $\mb U \mb U^\top$. 
Hence, it may not come as a surprise that DIP also inherits the drawbacks of the exact-parameterization approach for low-rank matrix recovery. Namely, it requires either a meticulous choice of network width \cite{heckel2018deep} or early stopping of the learning process \cite{heckel2019denoising} (see \Cref{fig:DIP-overview}).

\myparagraph{Overview of our methods and contributions} 
In this work, we show that the challenges associated with the exact-parameterization methods can be simply and effectively dealt with via \emph{over-parameterization} and {\em discrepant learning rates}. Our work is inspired by the recent results \cite{gunasekar2017implicit,li2018algorithmic} for low-rank matrix recovery with $\vs_\star = \vzero$, which state that the gradient descent on
\begin{equation}\label{eq:over-low-rank}
    \min_{\mU \in \bb R^{n \times r'}} \frac{1}{2} \norm{\calA(\mU \mU^\top) - \vy}{2}^2
\end{equation}
converges to a low-rank regularized solution with $\rank(\mb U) = \rank(\mX_\star)$ even when $\mb U\in \bb R^{n \times r'} $ over-parameterizes $\mX_\star$ with $r' = n \gg r$. 
In the presence of sparse corruption $\vs_\star$, it is tempting to replace the $\ell_2$-norm in \eqref{eq:over-low-rank} with an $\ell_1$-norm\footnote{This is a commonly used approach for robust optimization problems, such as robust regression~\cite{rousseeuw2005robust}, robust subspace learning \cite{zhu2018dual,qu2020finding,jiang2018nonconvex}, robust phase retrieval~\cite{duchi2019solving,davis2017nonsmooth}, robust matrix recovery \cite{li2016low,li2020nonconvex}, and many more.} and solve
\begin{equation}\label{eq:single-over-param}
    \min_{\mU \in \bb R^{n \times r'}}\; \norm{\calA(\mU \mU^\top) - \vy}{1}.
\end{equation}
Unfortunately, such a naive approach easily \emph{overfits} the corruptions with $r' = n$ (see \Cref{fig:RPCA-overview}). 

In this work, we introduce a \emph{double over-parameterization} method for robust recovery. 
Our method is based on over-parameterizing both the low-rank matrix $\mX_\star$ and the outliers $\vs_\star$, and leveraging implicit algorithmic bias to find the correct solution $(\mb X_\star, \mb s_\star)$.
The benefit of such an approach \emph{w.r.t.} the state of the art is summarized as follows (see also Table~\ref{table:overview}):
 

\begin{itemize}[leftmargin=*,topsep=-0.5pt,itemsep=-0.5pt]
    \item \emph{More scalable.} Our method is based on gradient descent only and does not require computing an SVD per iteration as in convex approaches \cite{candes2011robust}. Hence it is significantly more scalable.
    \item \emph{Prior knowledge free.} Our method requires \emph{no} prior knowledge on neither the rank of $\mX_\star$ nor the sparsity of $\vs_\star$, and is free of parameter tuning. This is \emph{in contrast} to existing nonconvex approaches \cite{li2016low,ge2016matrix,li2020nonconvex} for which the true rank and/or sparsity are required \emph{a priori}.
    \item \emph{Provably correct.} 
    Under similar conditions of the convex approach, our method converges to the ground truth $(\mX_\star, \vs_\star)$ asymptotically. 
\end{itemize}

\begin{wraptable}{r}{0.59\linewidth}
\vspace{-9pt}
    \begin{footnotesize}
        \caption{\footnotesize \label{table:overview} Comparison of different approaches to matrix recovery.}
        \begin{tabular}{c|@{}C{1.5cm}@{}C{1.9cm}@{}C{1.5cm}@{}}
        \toprule
        Methods & 
        \begin{tabular}{c} Convex \\ (Eq.~\eqref{eq:rpca-cvx-weight}) 
        \end{tabular}
        & \begin{tabular}{c} Nonconvex\\  (Eq.~\eqref{eq:single-over-param}) 
        \end{tabular}
        & \begin{tabular}{c} Ours\\  (Eq.~\eqref{eq:rms-over-1}) 
        \end{tabular}\\
        \midrule
        Provably correct? & $\checkmark$ & $\checkmark$ & $\checkmark$ \\
        \midrule
        Prior knowledge free? & $\checkmark$ &   & $\checkmark$   \\
        \midrule
        Scalable? &  & $\checkmark$ & $\checkmark$  \\
        \bottomrule
        \end{tabular}
    \end{footnotesize}
\vspace{-8pt}
\end{wraptable} 
Underlying the success of our method is the notion of \emph{implicit bias of discrepant learning rates}. The idea is that the algorithmic low-rank and sparse regularizations need to be balanced for the purpose of identifying the underlying rank and sparsity. 
With a lack of means of tuning a regularization parameter, we show in \Cref{sec:implicit-bias-learn-rate} that the desired balance can be obtained by using \emph{different} learning rates for different optimization parameters.
Such a property may be of separate interest and bear on a broader range of problems.


Finally, we demonstrate the broad applicability of our ideas on the task of the robust recovery of natural images. 
We only need to replace the over-parameterization $\mU \mU^\top$ 
for low-rank matrices by a U-shaped network $\phi(\bm \theta)$ for natural images (as adopted in DIP \cite{ulyanov2018deep}) and solve the optimization problem by gradient descent. 
This produces a powerful and easy-to-use method with favorable properties when compared to the original DIP (see \Cref{fig:DIP-overview}). To be precise, our method handles different test images and varying corruption levels with a single learning pipeline, in which network width and termination conditions do not need to be adjusted on a case-by-case basis.

\section{Main Results and Algorithms}\label{sec:main}



\subsection{A Double Over-Parameterization Formulation}\label{subsec:double-formulation}
\vspace{-0.3em}
As precluded in \eqref{eqn:problem}, we first consider the problem of recovering a rank-$r$ ($r \ll n$) positive semidefinite  matrix\footnote{By a lifting trick such as \cite{tu2016low,park2017non}, our method can be extended to handling arbitrary rectangular matrices.} $\mb X_\star \in \bb R^{n \times n}$ from its corrupted linear measurements $\mb y = \mc A(\mb X_\star) + \mb s_\star$, where $\mb s_\star \in \bb R^m$ is a vector of sparse corruptions. Recent advances on algorithmic implicit bias for optimizing over-parameterized models \cite{gunasekar2017implicit,li2018algorithmic,arora2019implicit,belabbas2020implicit,vaskevicius2019implicit,zhao2019implicit} motivate us to consider a nonlinear least squares for robust matrix recovery, with double over-parameterization $\mb X= \mb U\mb U^\top$ and $\mb s = \mb g\circ \mb g - \mb h \circ \mb h$:
\begin{align}\label{eq:rms-over-1}
    \min_{\mb U \in \bb R^{n \times r'} , \{\mb g, \mb h\} \subseteq \bb R^m }\; f(\mb U, \mb g,\mb h) \;:=\; \frac{1}{4}  \norm{  \mc A\paren{\mb U \mb U^\top} + \left(\mb g \circ \mb g - \mb h \circ \mb h\right) - \mb y }{2}^2,
\end{align}
where the dimensional parameter $r'\geq r$. In practice, the choice of $r'$ depends on how much prior information we have for $\mb X_\star$: $r'$ can be either taken as an estimated upper bound for $r$, or taken as $r' = n$ with no prior knowledge. In the following, we provide more backgrounds for the choice of our formulation \eqref{eq:rms-over-1}.

\begin{itemize}[leftmargin=*,topsep=-2pt]
\item \emph{Implicit low-rank prior via matrix factorization.} For the \emph{vanilla} low rank matrix recovery problem with no outlier (i.e., $\mb s_\star=\mb 0$), the low-rank matrix $\mb X_\star$ can be recovered via over-parameterization $ \mb X = \mb U\mb U^\top$ \cite{gunasekar2017implicit,li2018algorithmic,arora2019implicit,belabbas2020implicit}. In particular, the work \cite{gunasekar2017implicit} showed that  with small initialization and infinitesimal learning rate, gradient descent on \eqref{eq:over-low-rank} converges to the minimum nuclear norm solution under certain commute conditions for $\calA(\cdot)$. 
\item \emph{Implicit sparse prior via Hadamard multiplication.} For the classical sparse recovery problem \cite{candes2006robust,candes2008introduction} which aims to recover a sparse $\vs_\star \in \mathbb R^m$ from its linear measurement $\vb = \mA \vs_\star$, recent work \cite{vaskevicius2019implicit,zhao2019implicit} showed that it can also be dealt with via over-parameterization $\vs = \vg \circ \vg - \vh \circ \vh$. In particular, the work \cite{vaskevicius2019implicit} showed that with small initialization and infinitesimal learning rate, gradient descent on 
\begin{equation}
    \min_{\{\vg, \vh\} \subseteq \bb R^{m}} \; \norm{\mA (\vg \circ \vg - \vh \circ \vh) - \vb}{2}^2
\end{equation}
correctly recovers the sparse $\vs_\star$ when $\mA$ satisfies certain restricted isometric properties \cite{candes2008restricted}.
\end{itemize}

The benefit of the double over-parameterization formulation in \eqref{eq:rms-over-1}, as we shall see, is that it allows specific algorithms to \emph{automatically} identify both the intrinsic rank of $\mX_\star$ and the sparsity level of $\vs_\star$ without any prior knowledge.

\subsection{Algorithmic Regularizations via Gradient Descent}\label{sec:gd}
\vspace{-0.3em}
Obviously, over-parameterization leads to \emph{under-determined} problems which can have infinite number of solutions, so that not all solutions of \eqref{eq:rms-over-1} correspond to the desired $(\mX_\star,\vs_\star)$. For example, for any given $\mb U$, one can always construct a pair $(\mb g,\mb h)$ for \eqref{eq:rms-over-1} such that $(\mb U,\mb g, \mb h)$ achieves the global minimum value $f=0$. Nonetheless, as we see in this work, the gradient descent iteration on \eqref{eq:rms-over-1},
\begin{align}\label{eqn:grad-nls}
\begin{split}
        \mb U_{k+1} \;&=\; \mb U_{k} \;-\; \tau \cdot \nabla_{\mb U}  f(\mb U_{k},\mb g_{k},\mb h_{k}) \;=\; \mb U_k \;-\; \tau \cdot \mc A^* \paren{ \mb r_k  } \mb U_k, \\
   \begin{bmatrix}
   \mb g_{k+1} \\ \mb h_{k+1}
   \end{bmatrix} \;&=\; \begin{bmatrix} \mb g_{k} \\ \mb h_{k} \end{bmatrix}
    \;-\;  \alpha \cdot \tau \cdot \begin{bmatrix}  \nabla_{\mb g}  f(\mb U_{k},\mb g_{k},\mb h_{k}) \\ \nabla_{\mb h}  f(\mb U_{k},\mb g_{k},\mb h_{k})  \end{bmatrix} \;=\; \begin{bmatrix}  \mb g_{k} \\ \mb h_{k} \end{bmatrix} \; -\; \alpha \cdot \tau \cdot \begin{bmatrix}
      \mb r_{k} \circ \mb g_{k} \\
     -  \mb r_{k} \circ \mb h_{k} 
    \end{bmatrix},
\end{split} 
\end{align}
with properly chosen learning rates $(\tau,\; \alpha \cdot \tau)$ enforces implicit bias on the solution path, that it automatically identifies the desired, regularized solution $(\mX_\star,\vs_\star)$. Here, in \eqref{eqn:grad-nls} we have $\mc A^*(\cdot): \bb R^m \mapsto \bb R^{n \times n}$ being the conjugate operator of $\mc A(\cdot)$ and 
$\mb r_{k} = \mc A( \mb U_{k} \mb U_{k}^\top ) + \mb g_{k} \circ \mb g_{k} - \mb h_{k} \circ \mb h_{k} - \mb y$.

It should be noted that the scalar $\alpha > 0$, which controls the ratio of the learning rates for $\mU$ and $(\vg,\vh)$, plays a crucial role on the quality of the solution that the iterate in \eqref{eqn:grad-nls} converges to (see \Cref{fig:RPCA-vary-stepsize}). We will discuss this in more details in the next subsection (i.e., \Cref{sec:implicit-bias-learn-rate}).

\myparagraph{Convergence to low-rank \& sparse solutions}
Based on our discussion in \Cref{subsec:double-formulation}, it is expected that the gradient descent \eqref{eqn:grad-nls} converges to a solution $(\mU, \vg, \vh)$ such that
\begin{align}\label{eqn:nls-solution}
  \mb X \;=\; \mb U \mb U^\top\qquad\text{and}\qquad \mb s \;=\; \mb g \circ  \mb g \;-\; \mb h \circ \mb h
\end{align}
have the minimum nuclear norm and $\ell_1$ norm, respectively. More specifically, we expect the solution $(\mX, \vs)$ in \eqref{eqn:nls-solution} of the nonlinear least squares \eqref{eq:rms-over-1} also serves as a \emph{global} solution to a convex problem
\begin{align}\label{eq:rpca-cvx-weight}
\min_{\mX\in\R^{n\times n}, \; \mb s\in\R^{m}}\; \norm{\mX}{*} \;+\; \lambda \cdot  \norm{ \vs }{1}, \quad \st \;\; \calA(\mX) + \vs \;=\; \mb y,\;\; \mb X \succeq \mb 0,
\end{align}
for which we state more rigorously in \Cref{sec:dynamical} under a particular setting. However, it should be noted that obtaining the global solution of \eqref{eq:rpca-cvx-weight} does not necessarily mean that we find the desired solution $(\mb X_\star, \mb s_\star)$ --- the penalty $\lambda > 0$ in \eqref{eq:rpca-cvx-weight}, which balances the low-rank and the sparse regularizations, plays a crucial role on the quality of the solution to \eqref{eq:rpca-cvx-weight}. For instance, when $\calA(\cdot)$ is the identity operator, under proper conditions of $(\mb X_\star,\mb s_\star)$, the work \cite{candes2011robust} showed that the global solution of \eqref{eq:rpca-cvx-weight} is exactly the target solution $\paren{ \mb X_\star, \mb s_\star }$ \emph{only} when $\lambda = 1/\sqrt{n}$. Therefore, choosing the right $\lambda$ is critical for a  successful recovery of $(\mb X_\star,\mb s_\star)$. 


\subsection{Implicit Bias with Discrepant Learning Rates}\label{sec:implicit-bias-learn-rate}
\vspace{-0.3em}
As noted above, a remaining challenge is to control the implicit regularization of the gradient descent \eqref{eqn:grad-nls} so that the algorithm converges to the solution to \eqref{eq:rpca-cvx-weight} with the desired value $\lambda = 1 / \sqrt{n}$. 
Without explicit regularization terms in our new objective \eqref{eq:rms-over-1}, at first glance this might seem impossible. Nonetheless, we show that this can simply be achieved by adapting the ratio of learning rates $\alpha$ between $\mU$ and $(\vg, \vh)$ in our gradient step \eqref{eqn:grad-nls}, which is one of our key contributions in this work that could also bear broader interest. More specifically, this phenomenon can be captured by the following observation.




\begin{obs} 
With small enough learning rate $\tau$ and initialization $(\mU_{0}, \vg_{0}, \vh_{0})$ close enough to the origin, gradient descent \eqref{eqn:grad-nls} converges to a solution of \eqref{eq:rpca-cvx-weight} with $\lambda = 1/\alpha$.	
\end{obs}

In comparison to the classical optimization theory \cite{nocedal2006numerical} where learning rates \emph{only} affect algorithm convergence rate but not the quality of the solution, our observation here is surprisingly different --- using discrepant learning rates on different over-parameterized variables actually affects the specific solution that the algorithm converges to\footnote{Our result also differs from \cite{li2019towards,you2019does,nakkiran2020learning} which analyze the effect of initial large learning rates.}. In the following, let us provide some intuitions of why this happens and discuss its implications. We leave a more rigorous mathematical treatment to \Cref{sec:dynamical}.


\myparagraph{Intuitions}
The relation $\lambda = 1/\alpha$ implies that a large learning rate for one particular optimization variable in \eqref{eqn:grad-nls} leads to a small penalty on its implicit regularization counterpart in \eqref{eq:rpca-cvx-weight}. From a high-level perspective, this happens because a larger learning rate allows the optimization variable to move faster away from its initial point (which is close to the origin), resulting in a weaker regularization effect (which penalizes the distance of the variable to the origin) on its solution path.


\myparagraph{Implications}
The implicit bias of discrepant learning rates provides a new and powerful way for controlling the regularization effects in over-parametrized models. 
For robust matrix recovery in particular, it reveals an equivalency between our method in \eqref{eq:rms-over-1} and the convex method in \eqref{eq:rpca-cvx-weight} with one-to-one correspondence between learning rate ratio $\alpha$ and the regularization parameter $\lambda$. 
By picking $\alpha = \sqrt{n}$, we may directly quote results from \cite{candes2011robust} and conclude that our method correctly recovers $\mX_\star$ with information-theoretically optimal sampling complexity and sparsity levels (see \Cref{fig:RPCA-phase}). 
Note that this is achieved with no prior knowledge on the rank of $\mX_\star$ and sparsity of $\vs_\star$. 
Next, we show that such benefits have implications beyond robust low-rank matrix recovery.

\subsection{Extension to Robust Recovery of Natural Images}
\vspace{-0.3em}
Finally, we address the problem of robust recovery of a natural image\footnote{Here, $H, W$ are height and width of the image, respectively. $C$ denotes the number of channels of the image, where $C=1$ for a greyscale image and $C=3$ for a color image with RGB channels.} $\mX_\star \in \mathbb R^{C \times H \times W}$ from its corrupted observation $\mb y = \mb X_\star + \mb s_\star$, for which the structure of $\mX_\star$ cannot be fully captured by a low-rank model. 
Inspired by the work \cite{ulyanov2018deep} on showing the effectiveness of an image prior from a deep convolutional network $\mb X = \phi(\bm \theta)$ of particular architectures, where $\bm \theta \in \mathbb R^c$ denotes network parameters, we use the following formulation for image recovery:
\begin{align}\label{eq:dip-double}
    \min_{\bm \theta  \in \mathbb R^c,\; \{\mb g, \mb h\} \subseteq \bb R^{C\times H \times W} }\; f(\bm \theta, \vg, \vh) \;=\; \frac{1}{4}  \norm{  \phi(\bm \theta) + \paren{\mb g \circ \mb g - \mb h \circ \mb h} - \mb y }{F}^2.
\end{align}
As we will empirically demonstrate in \Cref{sec:experiment-image}, gradient descent on \eqref{eq:dip-double}
\begin{align}\begin{split}
        \bm \theta_{k+1} \;&=\; \bm \theta_{k} \;-\; \tau \cdot \nabla_{\bm \theta}  f(\bm \theta_{k},\mb g_{k},\mb h_{k}), \\
   \begin{bmatrix}
   \mb g_{k+1} \\ \mb h_{k+1}
   \end{bmatrix} \;&=\; \begin{bmatrix} \mb g_{k} \\ \mb h_{k} \end{bmatrix}
    \;-\;  \alpha \cdot \tau \cdot \begin{bmatrix}  \nabla_{\mb g}  f(\bm \theta_{k},\mb g_{k},\mb h_{k}) \\ \nabla_{\mb h}  f(\bm \theta_{k},\mb g_{k},\mb h_{k})  \end{bmatrix},
\end{split} 
\end{align}
with a balanced learning rate ratio $\alpha$ also enforces implicit regularizations on the solution path to the desired solution. It should be noted that this happens even that the over-parameterization $\mb X = \phi(\bm \theta)$ is a highly \emph{nonlinear} network (in comparison with shallow linear network $\mb X = \mb U\mb U^\top$ \cite{gunasekar2017implicit} or deep linear network \cite{arora2019implicit}), which raises several intriguing theoretical questions to be further investigated.

\section{Convergence Analysis of Gradient Flow Dynamics}\label{sec:dynamical}

We provide a dynamical analysis certifying our observation in \Cref{sec:implicit-bias-learn-rate}. Similar to \cite{gunasekar2017implicit,arora2019implicit}, we consider a special case where the measurement matrices $\Brac{ \mb A_i }_{i=1}^m$ associated with $\calA$ commute\footnote{Any $\calA: \mathbb R^{n \times n} \to \mathbb R^m$ can be written as $\calA(\mZ) = [\langle\mA_1, \mZ\rangle, \ldots, \langle\mA_m, \mZ\rangle]$ for some $\{\mA_i \in \mathbb R^{n\times n}\}_{i=1}^m$.}, and investigate the trajectories of the discrete gradient iterate of $\mb U_k,\;\mb g_k$, and $\mb h_k$ in \eqref{eqn:grad-nls} with \emph{infinitesimal} learning rate $\tau$ (i.e., $\tau \rightarrow 0$). In other words, we study their \emph{continuous} dynamics counterparts $\mb U_t(\gamma)$, $\mb g_t(\gamma)$, and $\mb h_t(\gamma)$, which are parameterized by $t\in [0,+\infty) $ and initialized at $t=0$ with
\begin{align}\label{eqn:flow-initialization}
    \mb U_0(\gamma) \;=\; \gamma \mb I, \quad \mb g_0(\gamma) \;=\; \gamma \mb 1,\quad h_0(\gamma) \;=\; \gamma \mb 1,
\end{align}
where $\gamma>0$. Thus, analogous to \eqref{eqn:grad-nls}, the behaviors of the continuous gradient flows can be captured via the following differential equations 
\begin{align}
    \dot{ \mb U }_t(\gamma) \;&=\; \lim_{ \tau \rightarrow 0 } 
    \big(\mb U_{t +\tau }(\gamma) - \mb U_{t }(\gamma)\big) /{\tau } \;=\; - \mc A^* \paren{ \mb r_t(\gamma)  } \mb U_t(\gamma), \label{eqn:U-t} \\
     \begin{bmatrix} \dot{\mb g}_t(\gamma) \\ \dot{\mb h}_t(\gamma) \end{bmatrix} \;&=\; \lim_{ \tau \rightarrow 0 } \paren{ \begin{bmatrix} \mb g_{t+\tau}(\gamma) \\ \mb h_{t+\tau}(\gamma) \end{bmatrix} - \begin{bmatrix} \mb g_t(\gamma) \\ \mb h_t(\gamma) \end{bmatrix} } \;/\; {\tau } \;=\;  -\; \alpha \cdot \begin{bmatrix}
     \mb r_t(\gamma) \circ \mb g_t(\gamma) \\
     - \mb r_t(\gamma) \circ \mb h_t(\gamma),
    \end{bmatrix}, \label{eqn:g-h-t}
\end{align}
with $\mb r_t(\gamma) \;=\; \mc A( \mb U_t(\gamma) \mb U_t^\top(\gamma) ) + \mb g_t(\gamma) \circ \mb g_t(\gamma) - \mb h_t(\gamma) \circ \mb h_t(\gamma) - \mb y$. Let $ \mb X_t (\gamma) = \mb U_t(\gamma) \mb U_t^\top (\gamma)$, then we can derive the gradient flow for $\mb X_t(\gamma)$ via chain rule
\begin{align}\label{eqn:X-dot}
    \dot{\mb X}_t(\gamma) \;=\; \dot{\mb U}_t(\gamma) \mb U_t^\top(\gamma) \;+\; \mb U_t(\gamma) \dot{\mb U}_t^\top(\gamma) \;=\; - \mc A^* (\mb r_t(\gamma)) \mb X_t(\gamma) - \mb X_t(\gamma) \mc A^*(\mb r_t(\gamma)).
\end{align}
For any $\gamma>0$, suppose the limits of $ \mb X_t (\gamma)$, $\mb g_t(\gamma)$, and $\mb h_t(\gamma)$ as $t \rightarrow +\infty $ exist and denote by
\begin{align}\label{eqn:X-h-g-infty}
    \mb X_\infty\paren{ \gamma } \;:=\; \lim_{ t \rightarrow +\infty } \mb X_t(\gamma),\quad \mb g_\infty(\gamma) \;:=\; \lim_{ t \rightarrow +\infty } \mb g_t(\gamma),\quad \mb h_\infty(\gamma) \;:=\; \lim_{ t \rightarrow +\infty } \mb h_t(\gamma).
\end{align}
Then when the initialization is infinitesimally small with $\gamma \rightarrow 0 $, we show that the following holds.

\begin{theorem}\label{thm:dynamic-rms}
Assume that the measurement matrices $\Brac{\mb A_i}_{i=1}^m$ commute with $ \mb A_i \mb A_j \; =\; \mb A_j \mb A_i$ for $\forall \; 1 \leq i \not= j \leq m$, and the gradient flows of $\mb U_t(\gamma)$, $\mb g_t(\gamma)$, and $\mb h_t(\gamma)$ satisfy \eqref{eqn:U-t} and \eqref{eqn:g-h-t} and are initialized by \eqref{eqn:flow-initialization}.
Let $ \mb X_\infty\paren{ \gamma }$, $\mb g_\infty(\gamma)$, and $\mb h_\infty(\gamma)$ be the limit points as defined in \eqref{eqn:X-h-g-infty}. Suppose that our initialization is infinitesimally small such that
\begin{align*}
    \wh{\mb X} \;:=\; \lim_{\gamma \rightarrow 0 } \mb X_{\infty}(\gamma),\quad \wh{\mb g} \;:=\; \lim_{\gamma \rightarrow 0 } \mb g_{\infty}(\gamma),\quad \wh{\mb h} \;:=\; \lim_{\gamma \rightarrow 0 } \mb h_{\infty}(\gamma)
\end{align*}
exist and $( \wh{\mb X}, \wh{\mb g}, \wh{\mb h} )$ is a global optimal solution to \eqref{eq:rms-over-1} with
\begin{align*}
    \mc A(\wh{\mb X}) + \wh{\mb s} \;=\; \mb y \quad \text{and}\quad \wh{\mb s}\;=\; \wh{\mb g} \circ \wh{\mb g} - \wh{\mb h} \circ \wh{\mb h}.
\end{align*}
Then we have $\wh{\mb g} \circ \wh{\mb h} = \mb 0$, and $(\wh{\mb X},\wh{\mb s} )$ is also a global optimal solution to \eqref{eq:rpca-cvx-weight}, with $\lambda = \alpha^{-1}$.
\end{theorem}
A detailed proof of Theorem~\ref{thm:dynamic-rms} is deferred to Appendix~\ref{sec:proof}. Our result shows that among the infinite many global solutions to \eqref{eq:rms-over-1}, gradient descent biases towards the one with the minimum nuclear norm (for $\mX$) and $\ell_1$ norm (for $\vs$) with relative weight controlled by $\alpha$. In the following, we discuss the rationality of the assumptions for Theorem~\ref{thm:dynamic-rms}.
\begin{itemize}[leftmargin=*]
    \item Since gradient descent almost surely converges to a local minimum \cite{lee2016gradient} and any local minimum is likely to be a global minimum in low-rank matrix optimization with over-parameterization (i.e., $\mU\in\R^{n\times n}$) \cite{journee2010low}, the assumption that $( \wh{\mb X}, \wh{\mb g}, \wh{\mb h} )$ is a global optimal solution is generally satisfied. 
    \item The commutative assumption is commonly adopted in the analysis of over-parameterized low-rank matrix recovery and is empirically observed to be non-essential  \cite{gunasekar2017implicit,arora2019implicit}. A recent work \cite{li2018algorithmic} provides a more sophisticated analysis of the discrete dynamic under the restricted isometric assumption where the commutative assumption is not needed. We believe such analysis can be extended to our double over-parameterization setting as well, but leave it as future work. 
\end{itemize}

\section{Experiments}

In this section, we provide experimental evidence for the implicit bias of the learning rate discussed in \Cref{sec:implicit-bias-learn-rate}. Through experiments for low-rank matrix recovery, \Cref{sec:experiment-low-rank} shows that the learning rate ratio $\alpha$ affects the solution that gradient descent converges to, and that an optimal $\alpha$ does not depend on the rank of matrix and sparsity of corruption. Furthermore, \Cref{sec:experiment-image} shows the effectiveness of implicit bias of learning rate in alleviating overfitting for robust image recovery, and demonstrates that our method produces better recovery quality when compared to DIP for varying test images and corruption levels, all with a single model and set of learning parameters.


\begin{figure}
	\begin{subfigure}{0.32\linewidth}
		\centerline{
    \includegraphics[width=1.8in]{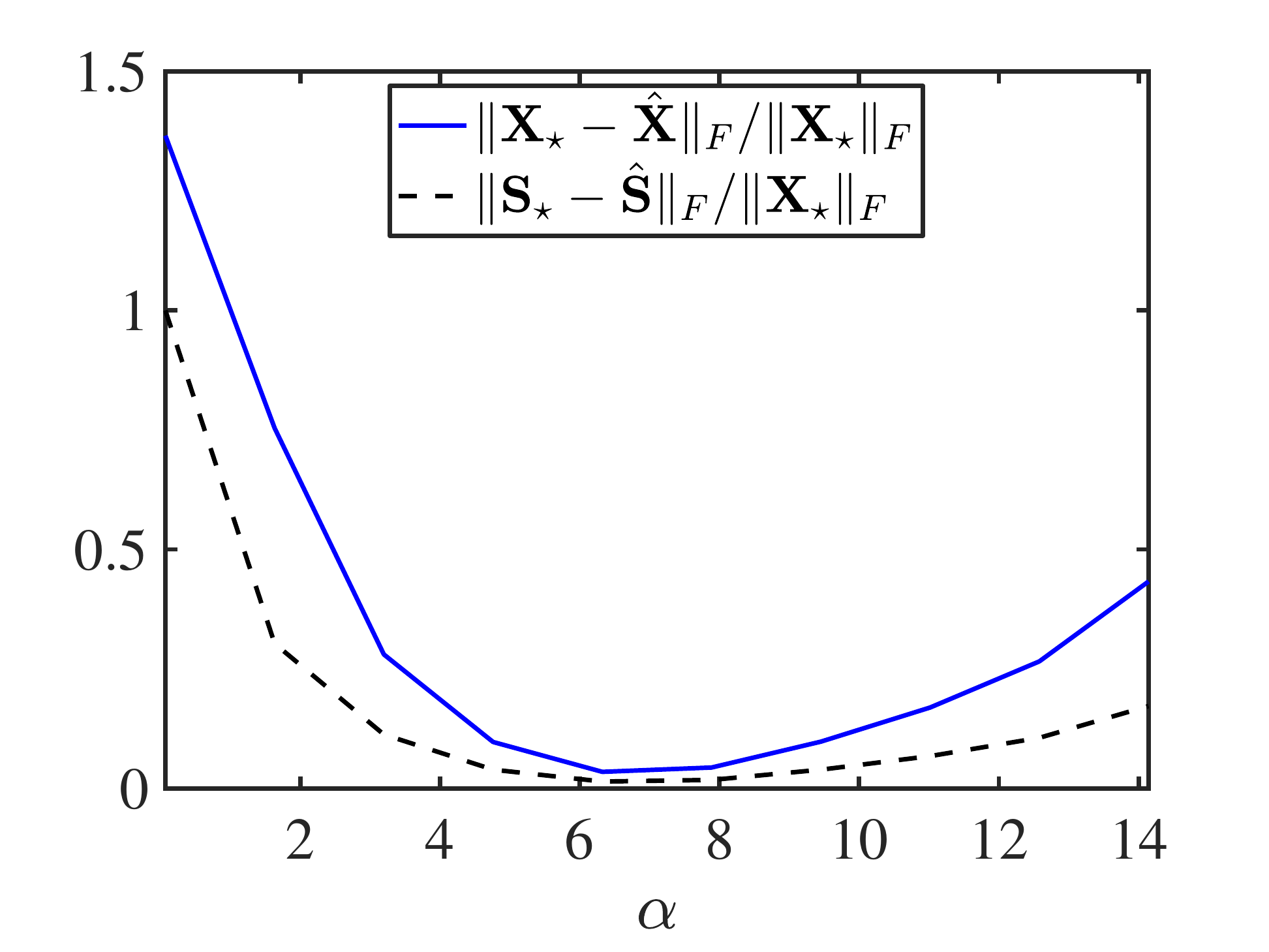}}
  \caption{Effect of learning rate ratio $\alpha$}
  \label{fig:RPCA-vary-stepsize}
  \end{subfigure}
  \hfill
	\begin{subfigure}{0.32\linewidth}
		\centerline{
			\includegraphics[width=1.8in]{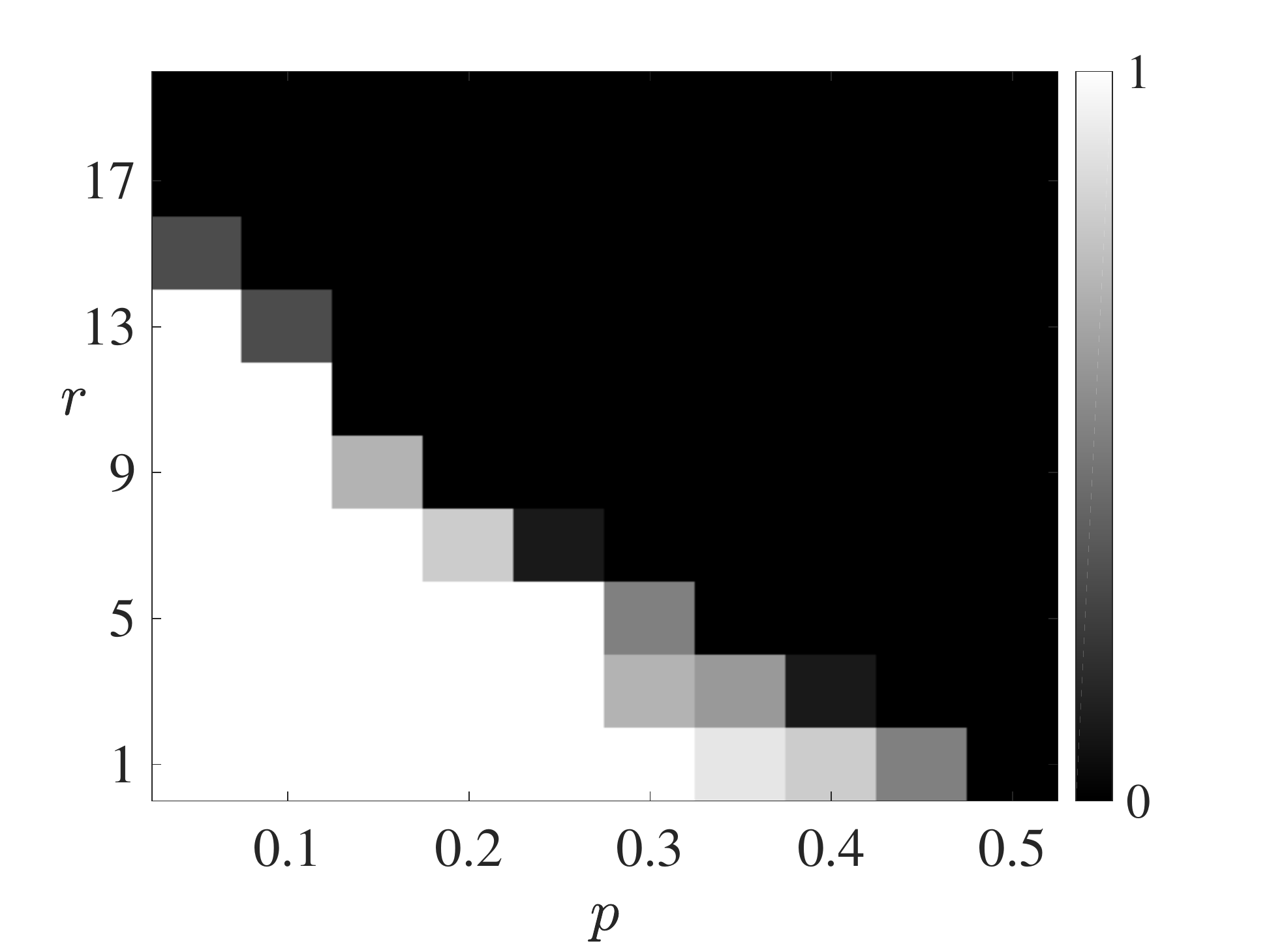}}
			\caption{Convex method \cite{lin2009fast}}
			  \label{fig:RPCA-convex}
	\end{subfigure}
	\hfill
	\begin{subfigure}{0.32\linewidth}
		\centerline{
			\includegraphics[width=1.8in]{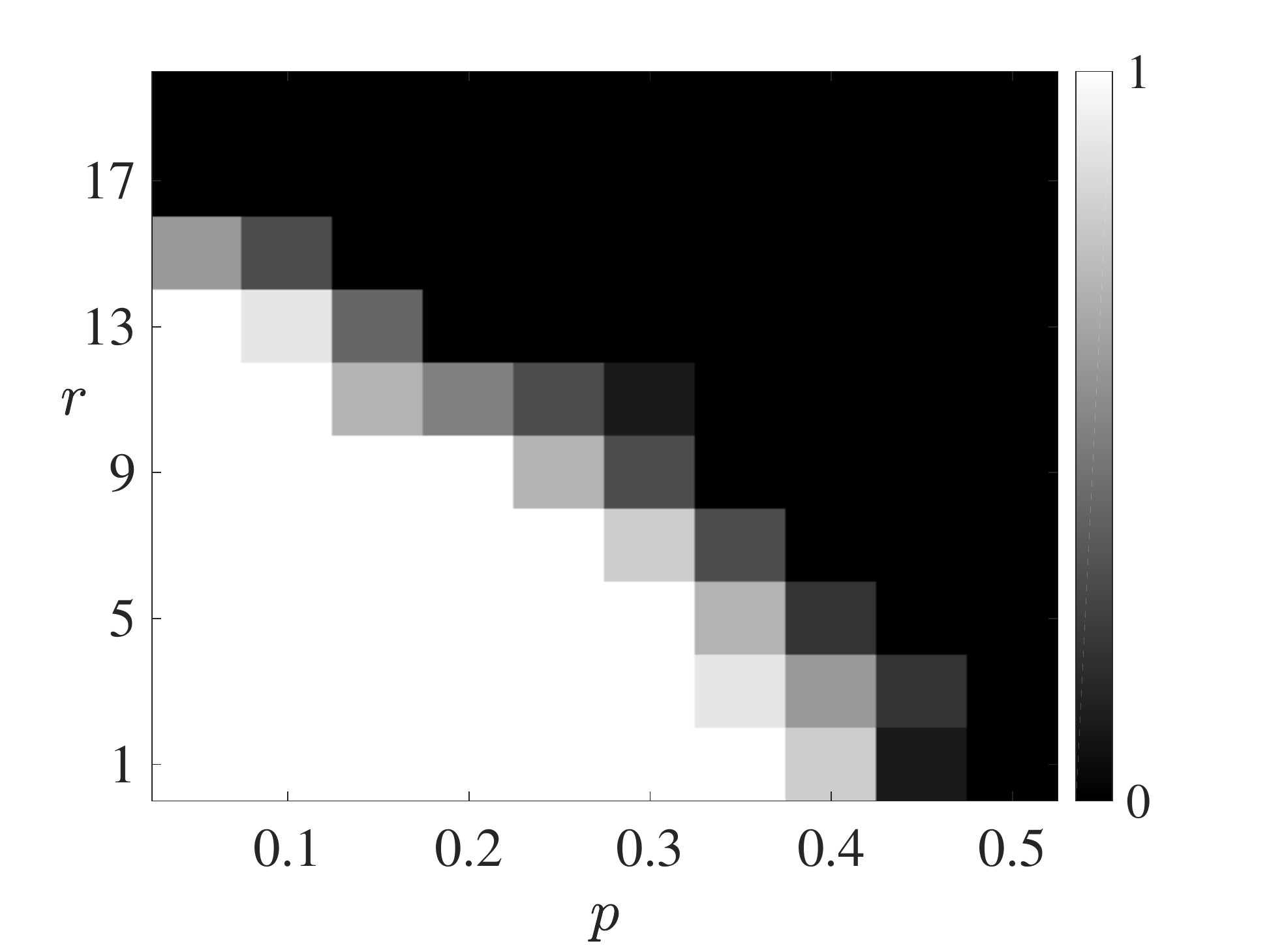}}
			\caption{Our method}
			\label{fig:RPCA-ours}
	\end{subfigure}
	\caption{\small \textbf{Numerical results for robust PCA} (with $n = 50$). (a) Relative reconstruction error for the output ($\widehat\mX, \widehat\mS$) of our method with varying step size ratio $\alpha$. Here, we set $r=5$ and $p=30\%$. (b, c) Probability of successful recovery over $10$ trials with varying $r$ (y-axis) and $p$ (x-axis). Here, we fixed $\alpha=\sqrt{n}$. White indicates always successful recovery, while black means total failure. 
	}\label{fig:RPCA-phase}
\vspace{-14pt}
\end{figure}

\subsection{Robust Recovery of Low-rank Matrices}
\label{sec:experiment-low-rank}
\vspace{-0.3em}
We conduct experiments for a particular case of low-rank matrix recovery problem, namely the robust PCA problem, in which the operator $\mathcal{A}$ is the identity operator. More specifically, the goal is to recovery a low-rank matrix $\mX_\star$ and a sparse matrix $\mS_\star$ from the mixture $\mY = \mX_\star + \mS_\star$, possibly without prior knowledge on the rank $r$ of $\mX_\star$ and the percentage of nonzero entries $p$ of $\mS_\star$. For any given $r$, we generate $\mX_\star \in \mathbb R^{n \times n}$ by setting $\mX_\star = \mU_\star\mU_\star^\top$, where $\mU_\star$ is an $n\times r$ matrix with i.i.d. entries drawn from standard Gaussian distribution. For any given $p$, we generate $\mS_\star \in \mathbb R^{n \times n}$ by sampling uniformly at random $n^2 \cdot p$ locations from the matrix and setting those entries by sampling i.i.d. from a zero-mean Gaussian distribution with standard deviation $10$. We use $n = 50$.

We apply our double over-parameterization method in \eqref{eq:rms-over-1} for the robust PCA problem, 
Specifically, we initialize $\mU$ and $\vg$ by drawing i.i.d. entries from zero mean Gaussian distribution with standard deviation $10^{-4}$, and initialize $\vh$ to be the same as $\vg$. The learning rate for $\mU$ as well as for $\{\vg, \vh\}$ are set to $\tau$ and $\alpha \cdot \tau$, respectively, where $\tau = 10^{-4}$ for all experiments.



\myparagraph{Effect of learning rate ratio $\alpha$}
We set $r = 5$ and $p = 30\%$, perform $2\times 10^4$ gradient descent steps and compute reconstruction errors for the output of our algorithm $(\wh\mX, \wh\mS)$ relative to the ground truth $(\mX_\star, \mS_\star)$. \Cref{fig:RPCA-vary-stepsize} illustrates the performance with varying values of $\alpha$. We observe that $\alpha$ affects the solution that the algorithm converges to, which verifies that the learning rate ratio has an implicit regularization effect. Moreover, the best performance is given by $\alpha = \sqrt{n}$, which is in accordance with our theoretical result in \Cref{thm:dynamic-rms}.

\vspace{-0.05in}
\myparagraph{Effect of rank $r$ and outlier ratio $p$ and phase transition} We now fix $\alpha = \sqrt{n}$ and study the ability of our method for recovering matrices of varying rank $r$ with varying percentage of corruption $p$. For each pair $(r, p)$, we apply our method and declare the recovery to be successful if the relative reconstruction error $\frac{\norm{\wh\mX - \mX_\star}{F}}{\norm{\mX_\star}{F}}$ is less than $0.1$. \Cref{fig:RPCA-ours} displays the fraction of successful recovery in $10$ Monte Carlo trials. It shows that a single value of the parameter $\alpha$ leads to correct recovery for a wide range of $r$ and $p$. \Cref{fig:RPCA-convex} shows the fraction of successful recovery for the convex method in \eqref{eq:rpca-cvx-weight} (with the Accelerated Proximal Gradient \cite{lin2009fast} solver \footnote{\scriptsize\url{http://people.eecs.berkeley.edu/~yima/matrix-rank/sample_code.html}}).


\begin{figure}
\captionsetup[subfigure]{labelformat=empty,font=footnotesize,skip=-0.9em}
	\begin{subfigure}{0.195\linewidth}
			\includegraphics[trim=13cm 20.2cm 2.5cm 20.5cm, clip,width=\linewidth]{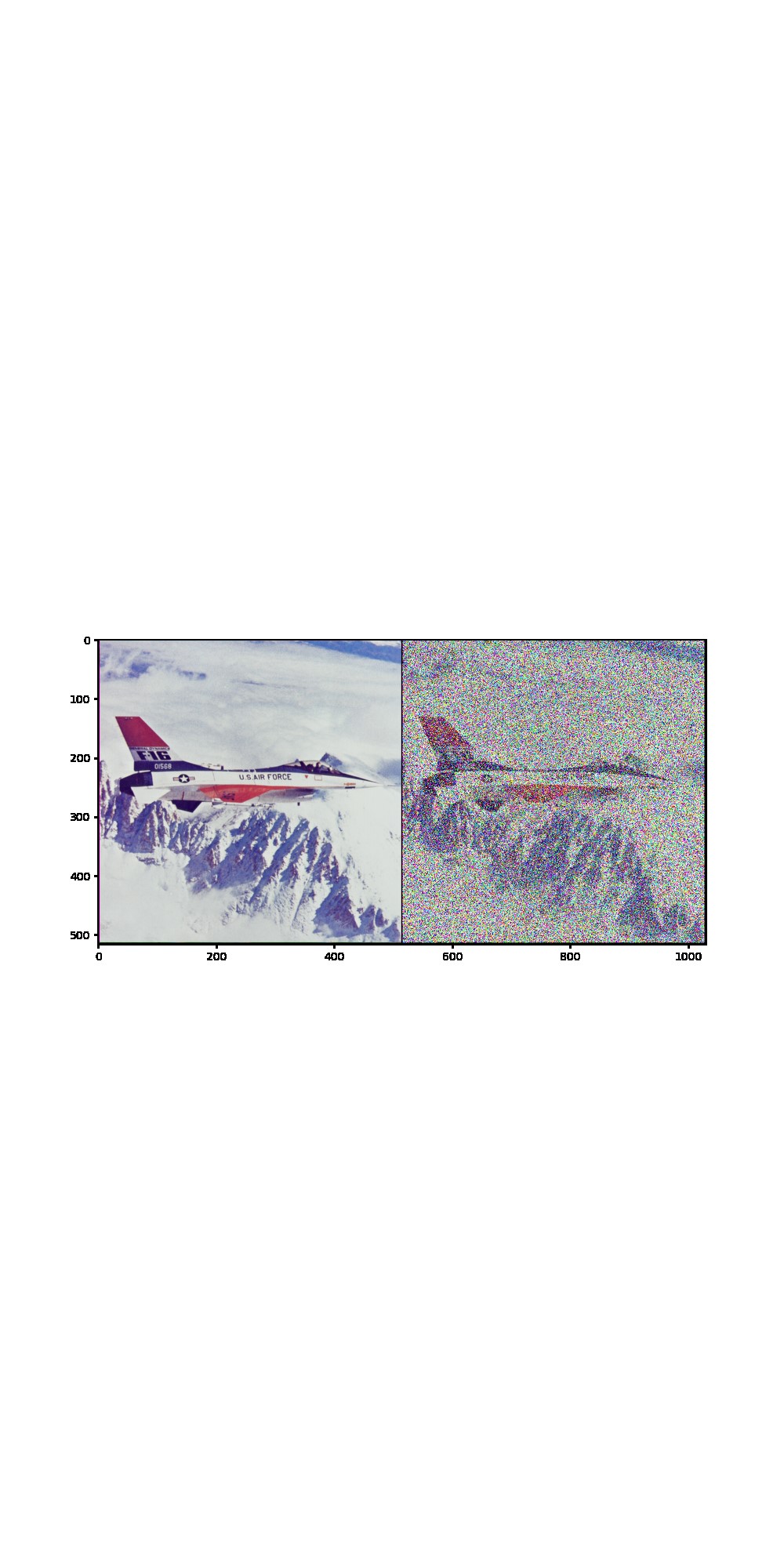}
	\end{subfigure}
    \begin{subfigure}{0.195\linewidth}
			\includegraphics[trim=3.2cm 20.2cm 12.4cm 20.5cm, clip,width=\linewidth]{figs/Ours_n05_true.jpg}
	\end{subfigure}
	\begin{subfigure}{0.195\linewidth}
	        \includegraphics[trim=2.8cm 16.2cm 2.2cm 16.3cm, clip,width=\linewidth]{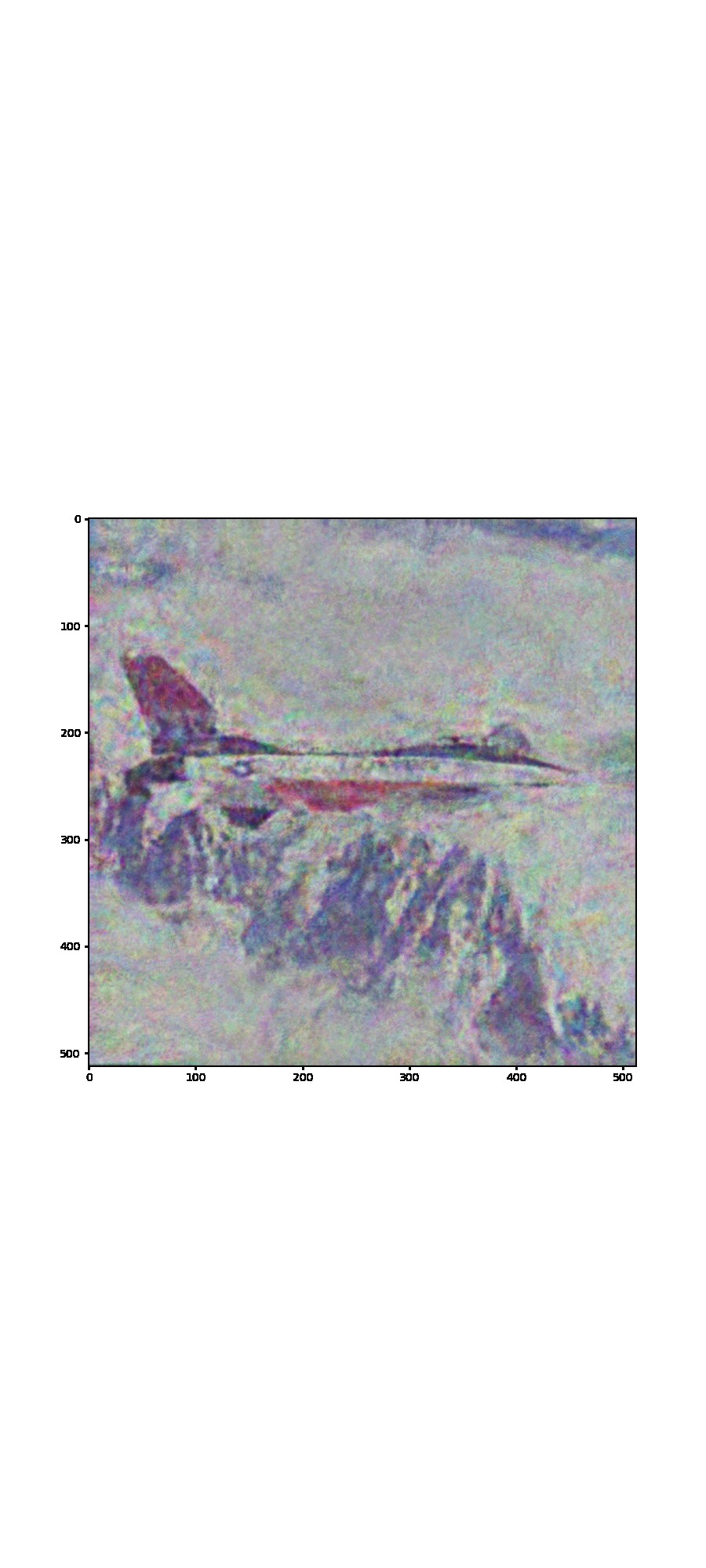}
	        \caption{PSNR $=16.46$}
	\end{subfigure}
	\begin{subfigure}{0.195\linewidth}
	        \includegraphics[trim=2.8cm 16.2cm 2.2cm 16.3cm, clip,width=\linewidth]{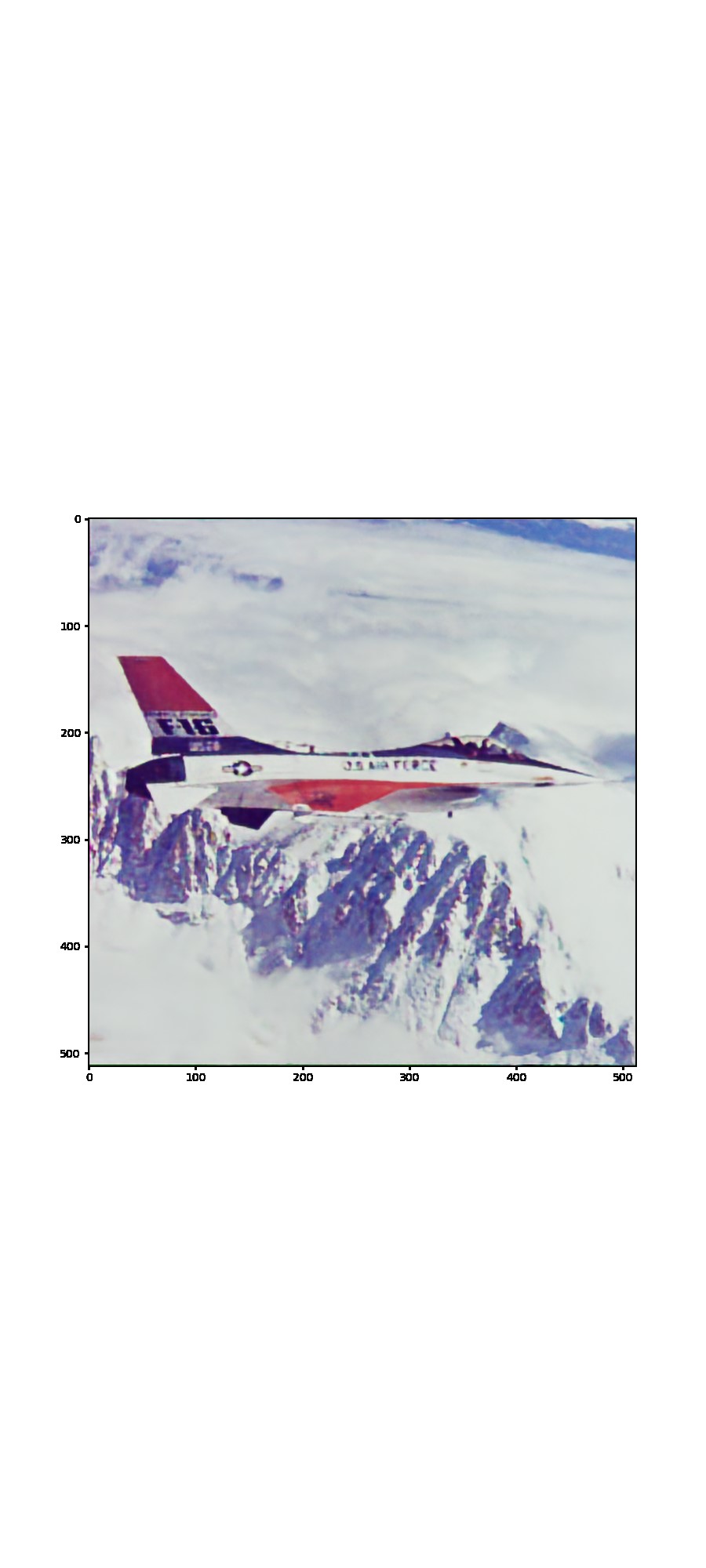}
	        \caption{PSNR $=29.07$}
	\end{subfigure}
	\begin{subfigure}{0.195\linewidth}
	        \includegraphics[trim=2.8cm 16.2cm 2.2cm 16.3cm, clip,width=\linewidth]{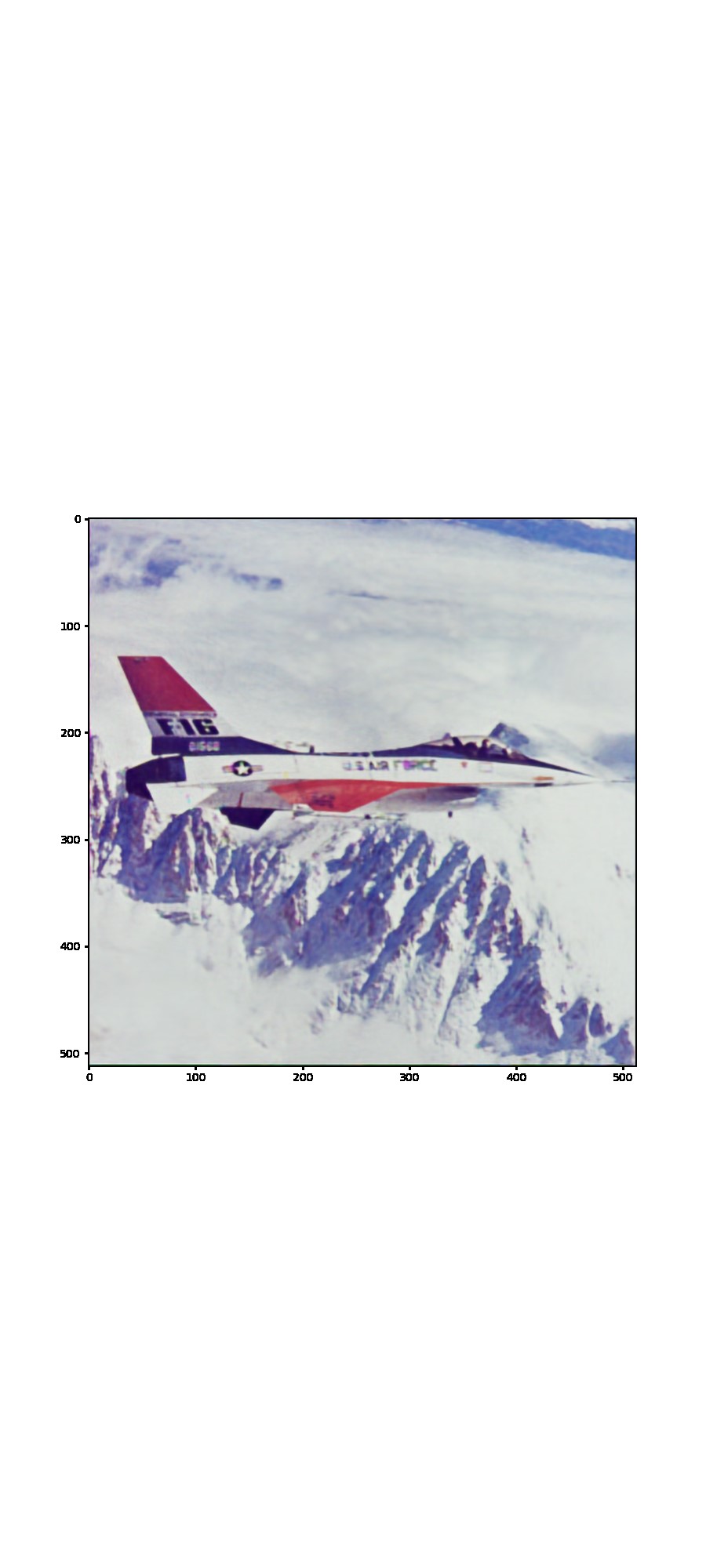}
	        \caption{PSNR $=31.60$}
	\end{subfigure}
	\\[-0.2em]
	\begin{subfigure}{0.195\linewidth}
			\includegraphics[trim=13cm 20.2cm 2.5cm 20.5cm, clip,width=\linewidth]{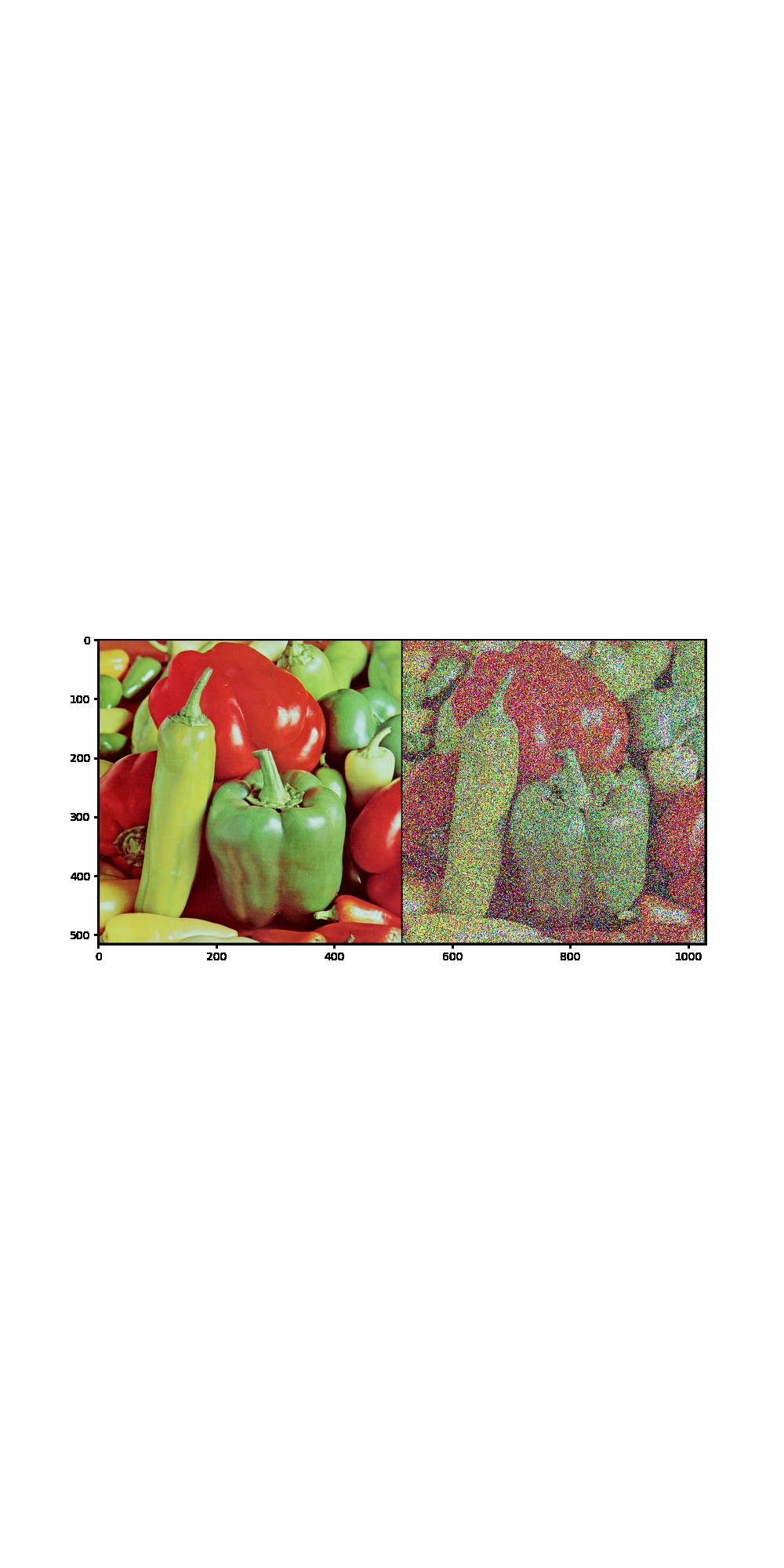}
	\end{subfigure}
	\begin{subfigure}{0.195\linewidth}
			\includegraphics[trim=3.2cm 20.2cm 12.4cm 20.5cm, clip,width=\linewidth]{figs/Ours_peppers_true.jpg}
	\end{subfigure}
	\begin{subfigure}{0.195\linewidth}
			\includegraphics[trim=2.8cm 16.2cm 2.2cm 16.3cm, clip,width=\linewidth]{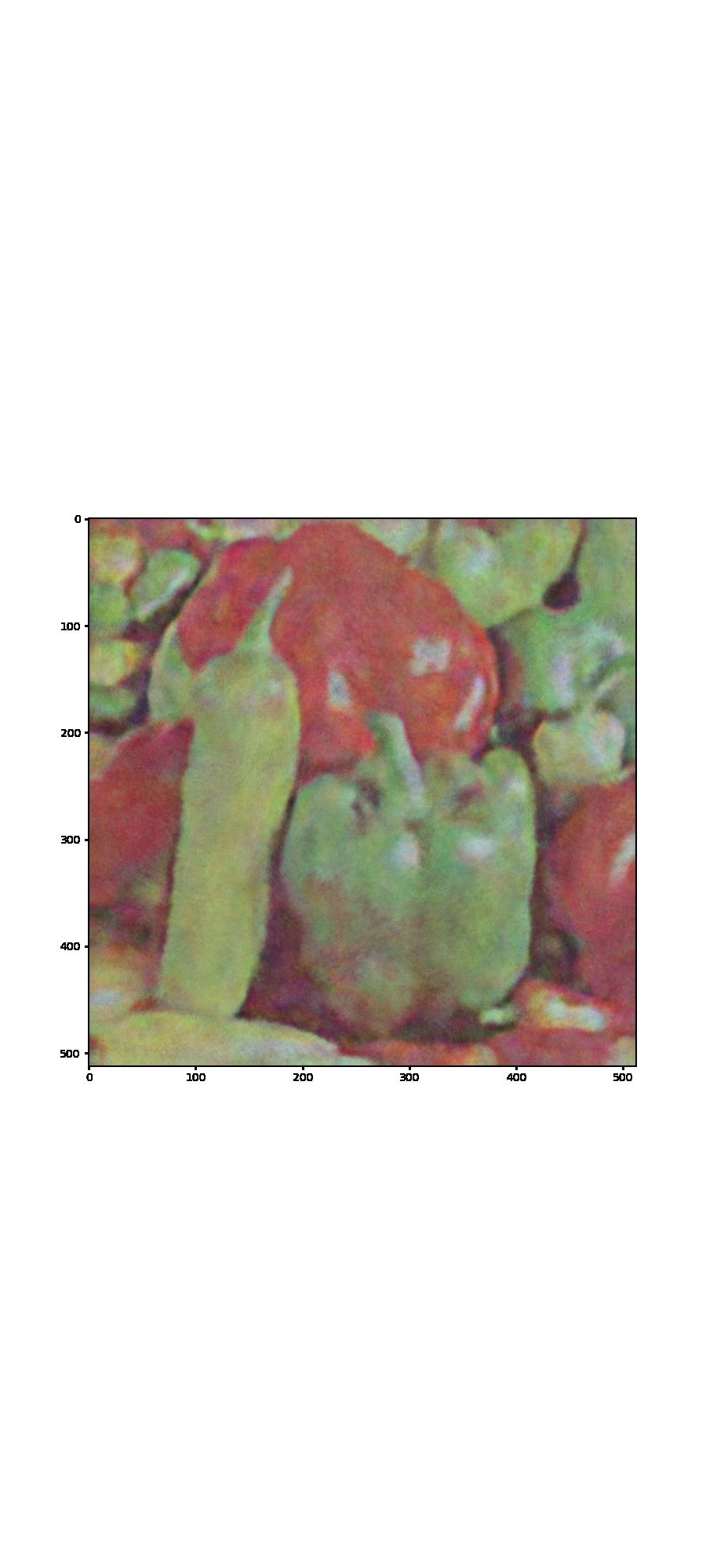}
		    \caption{PSNR $=16.63$}
	\end{subfigure}
	\begin{subfigure}{0.195\linewidth}
			\includegraphics[trim=2.8cm 16.2cm 2.2cm 16.3cm, clip,width=\linewidth]{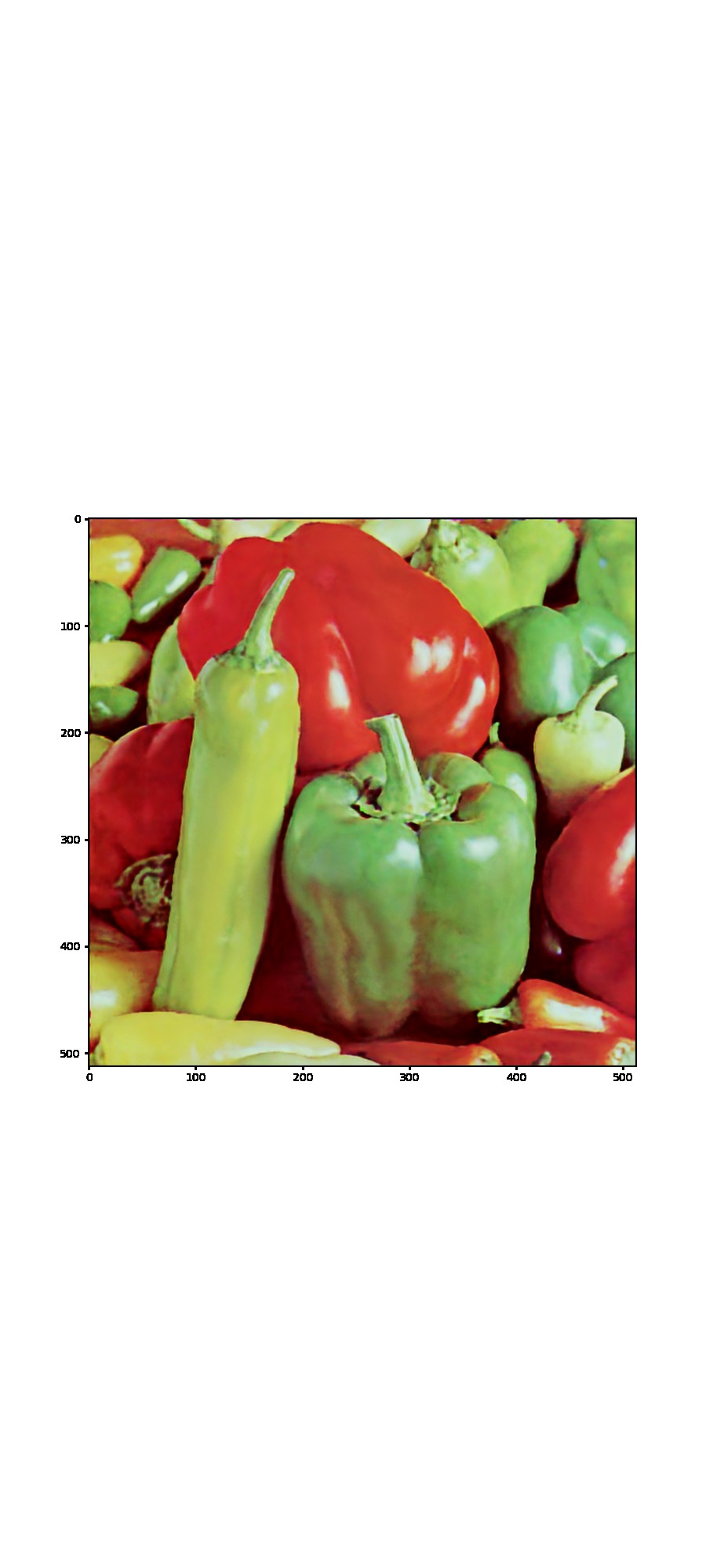}
			\caption{PSNR $=29.13$}
	\end{subfigure}
	\begin{subfigure}{0.195\linewidth}
			\includegraphics[trim=2.8cm 16.2cm 2.2cm 16.3cm, clip,width=\linewidth]{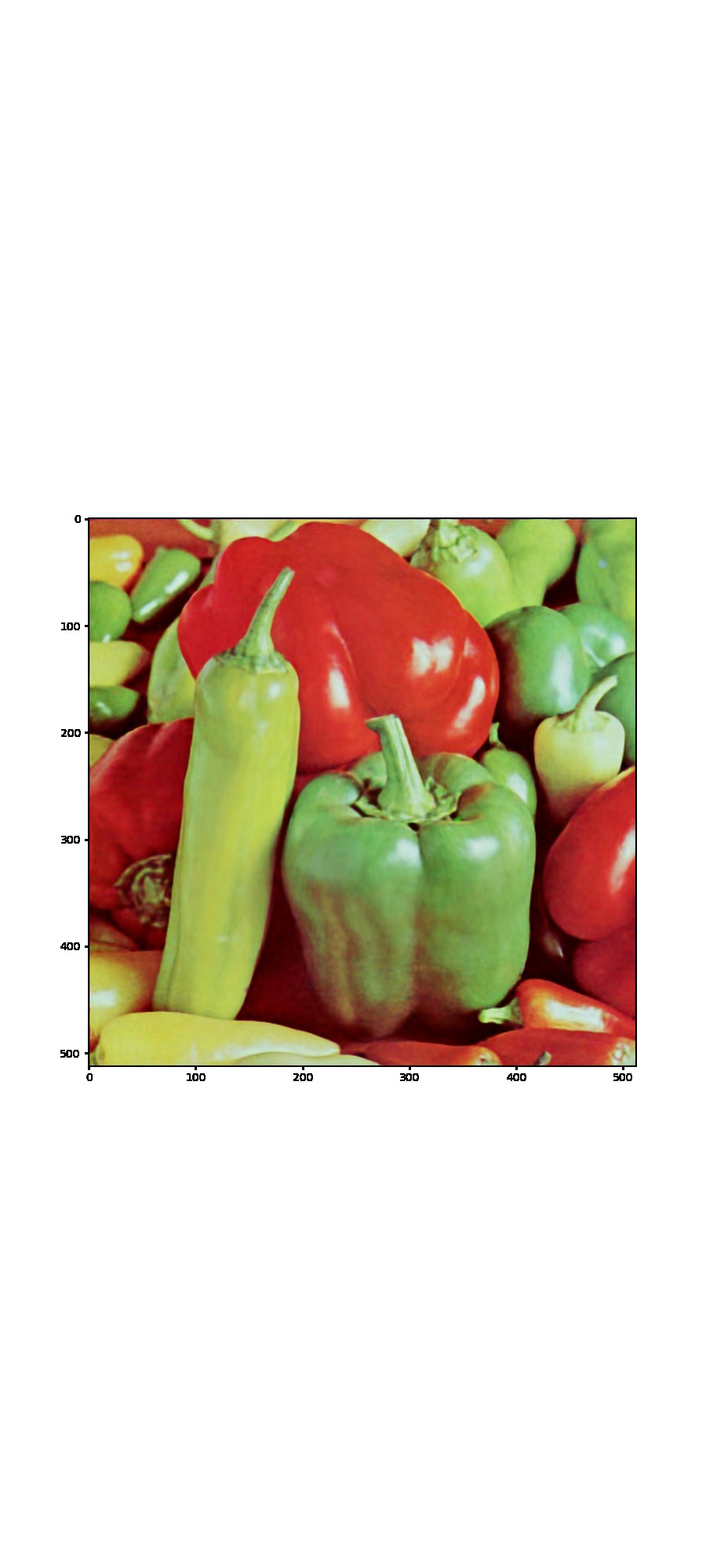}
			\caption{PSNR $=30.76$}
	\end{subfigure}
	\\[-0.2em]
	\begin{subfigure}{0.195\linewidth}
			\includegraphics[trim=13cm 20.2cm 2.5cm 20.5cm, clip,width=\linewidth]{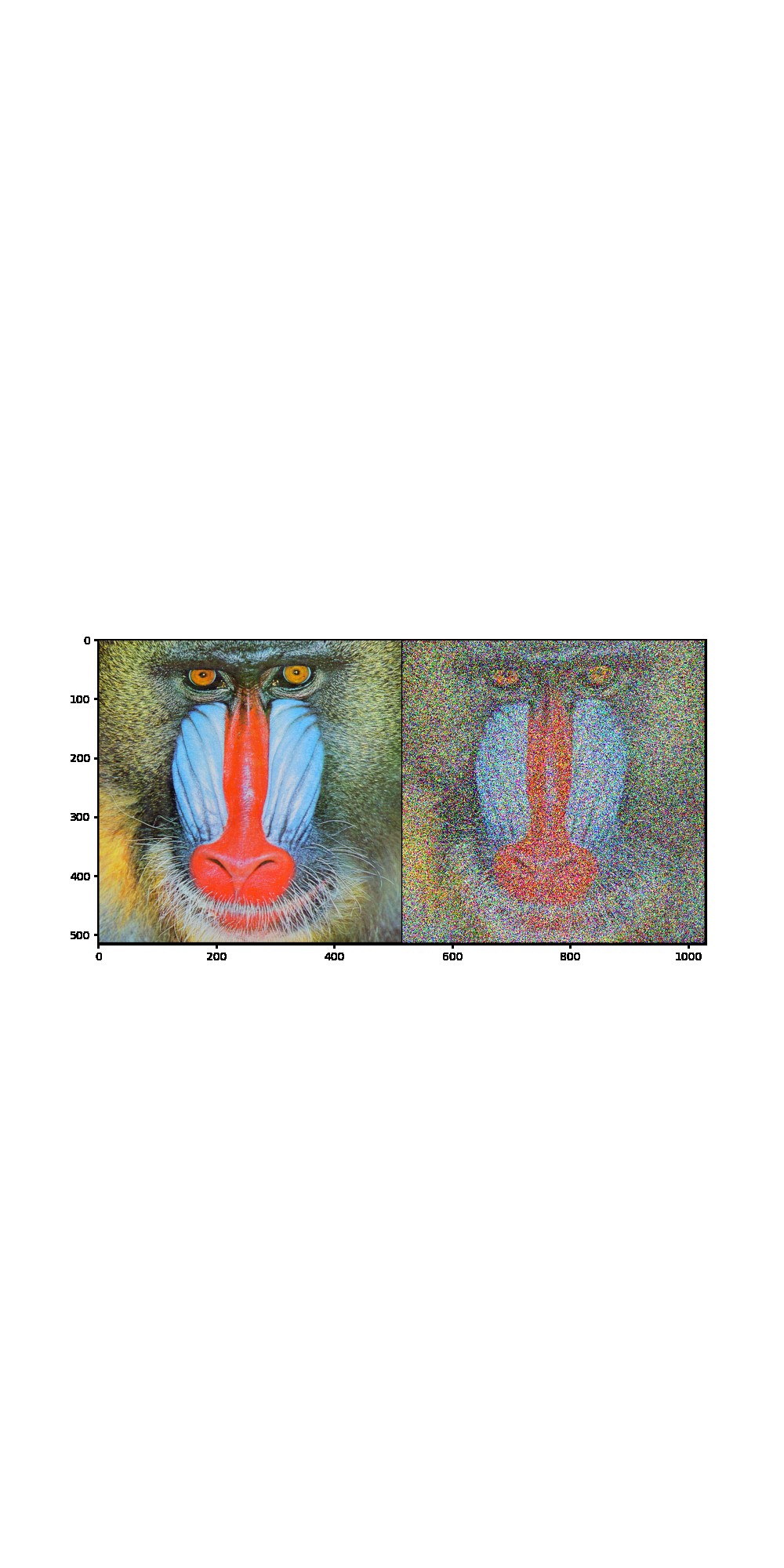}
	\end{subfigure}
	\begin{subfigure}{0.195\linewidth}
			\includegraphics[trim=3.2cm 20.2cm 12.4cm 20.5cm, clip,width=\linewidth]{figs/Ours_baboon_true.jpg}
	\end{subfigure}
	\begin{subfigure}{0.195\linewidth}
			\includegraphics[trim=2.8cm 16.2cm 2.2cm 16.3cm, clip,width=\linewidth]{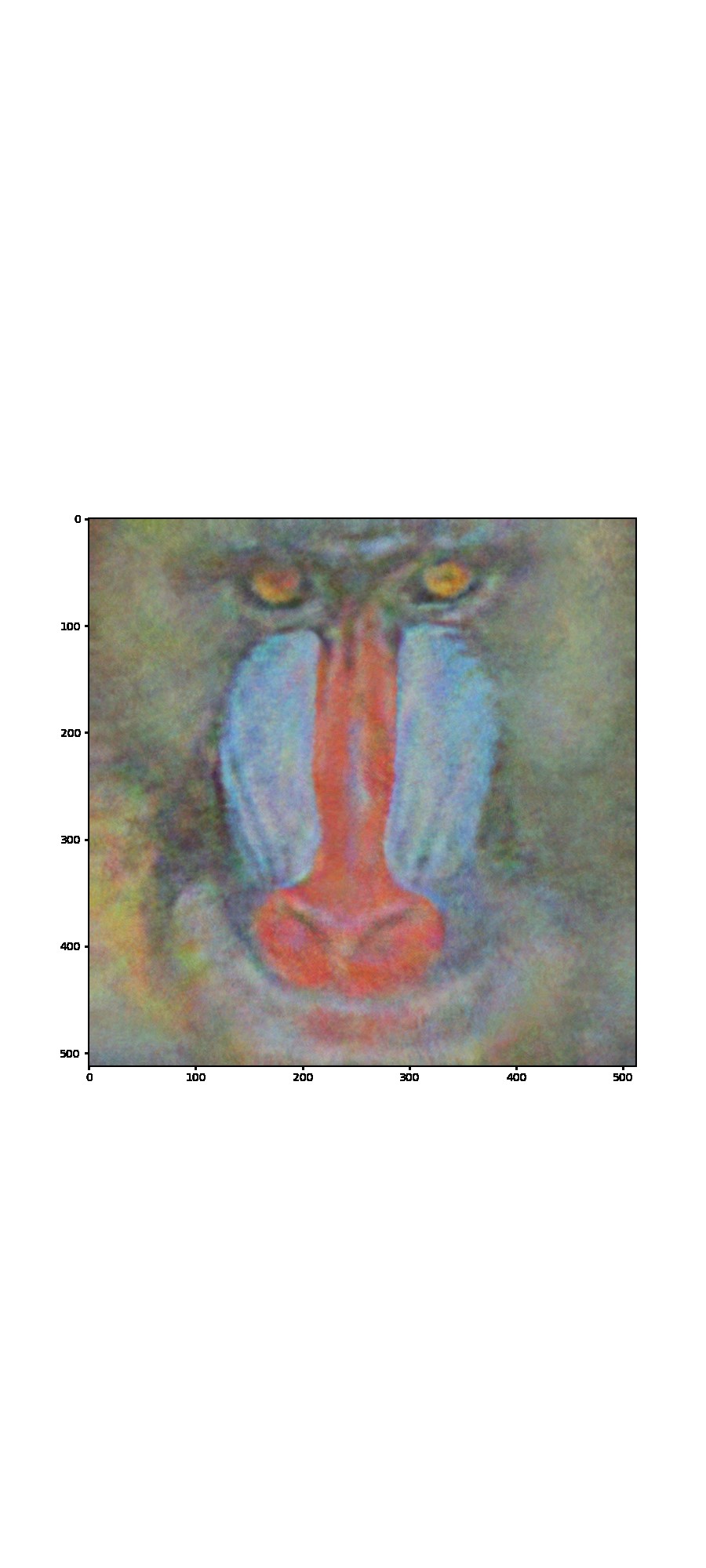}
			\caption{PSNR $=18.56$}
	\end{subfigure}
	\begin{subfigure}{0.195\linewidth}
			\includegraphics[trim=2.8cm 16.2cm 2.2cm 16.3cm, clip,width=\linewidth]{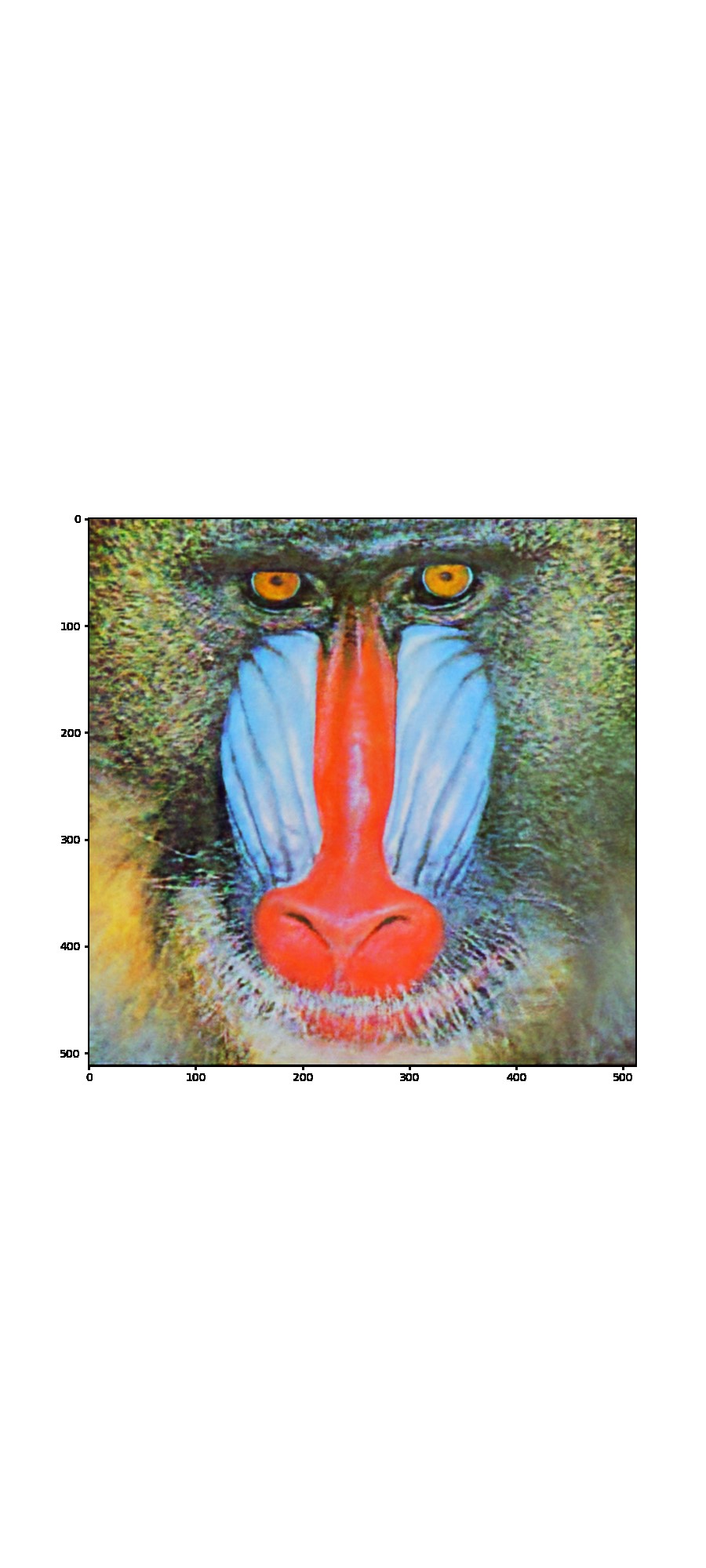}
			\caption{PSNR $=20.37$}
	\end{subfigure}
	\begin{subfigure}{0.195\linewidth}
			\includegraphics[trim=2.8cm 16.2cm 2.2cm 16.3cm, clip,width=\linewidth]{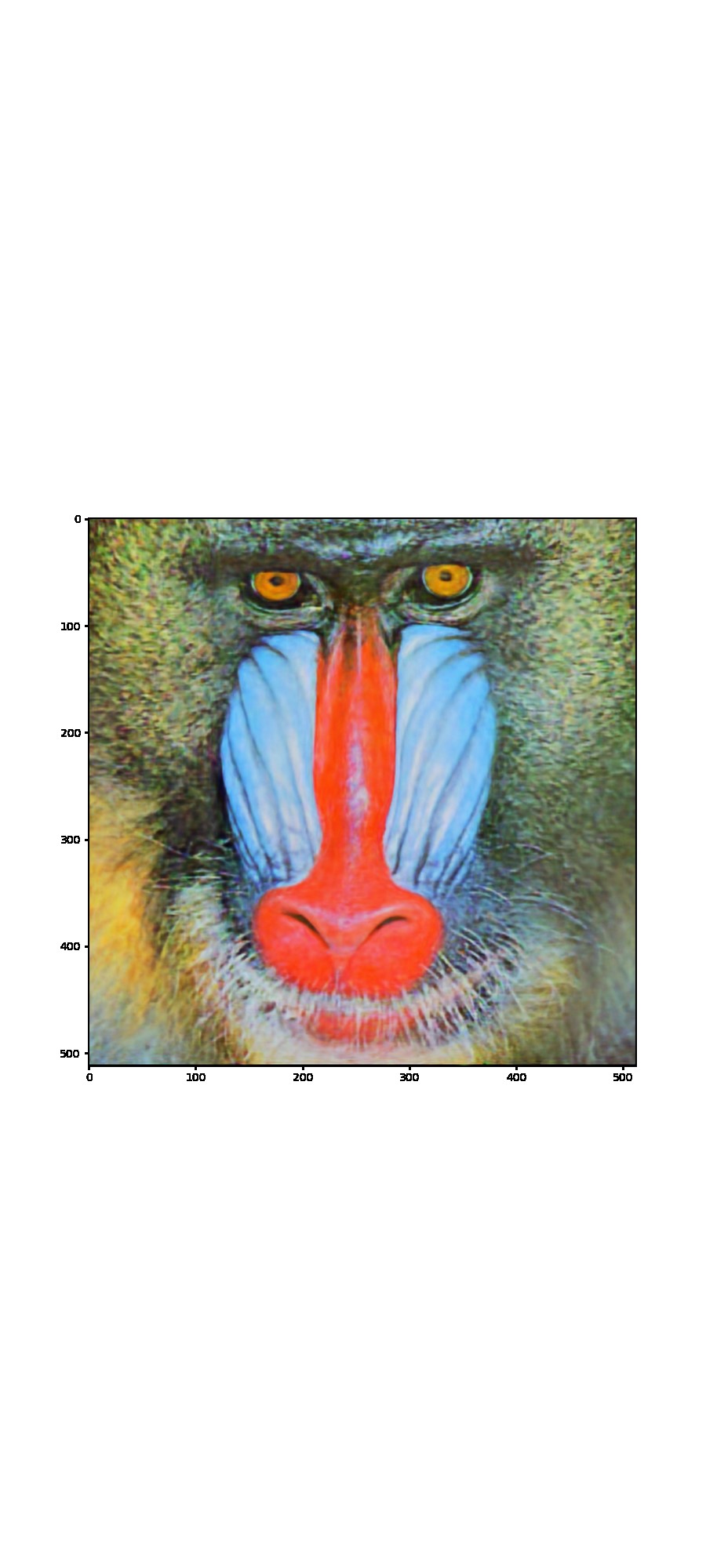}
			\caption{PSNR $=21.38$}
	\end{subfigure}
	\\[-0.2em]
	\begin{subfigure}{0.195\linewidth}
			\includegraphics[trim=14.7cm 22.1cm 4.2cm 22.2cm, clip,width=\linewidth]{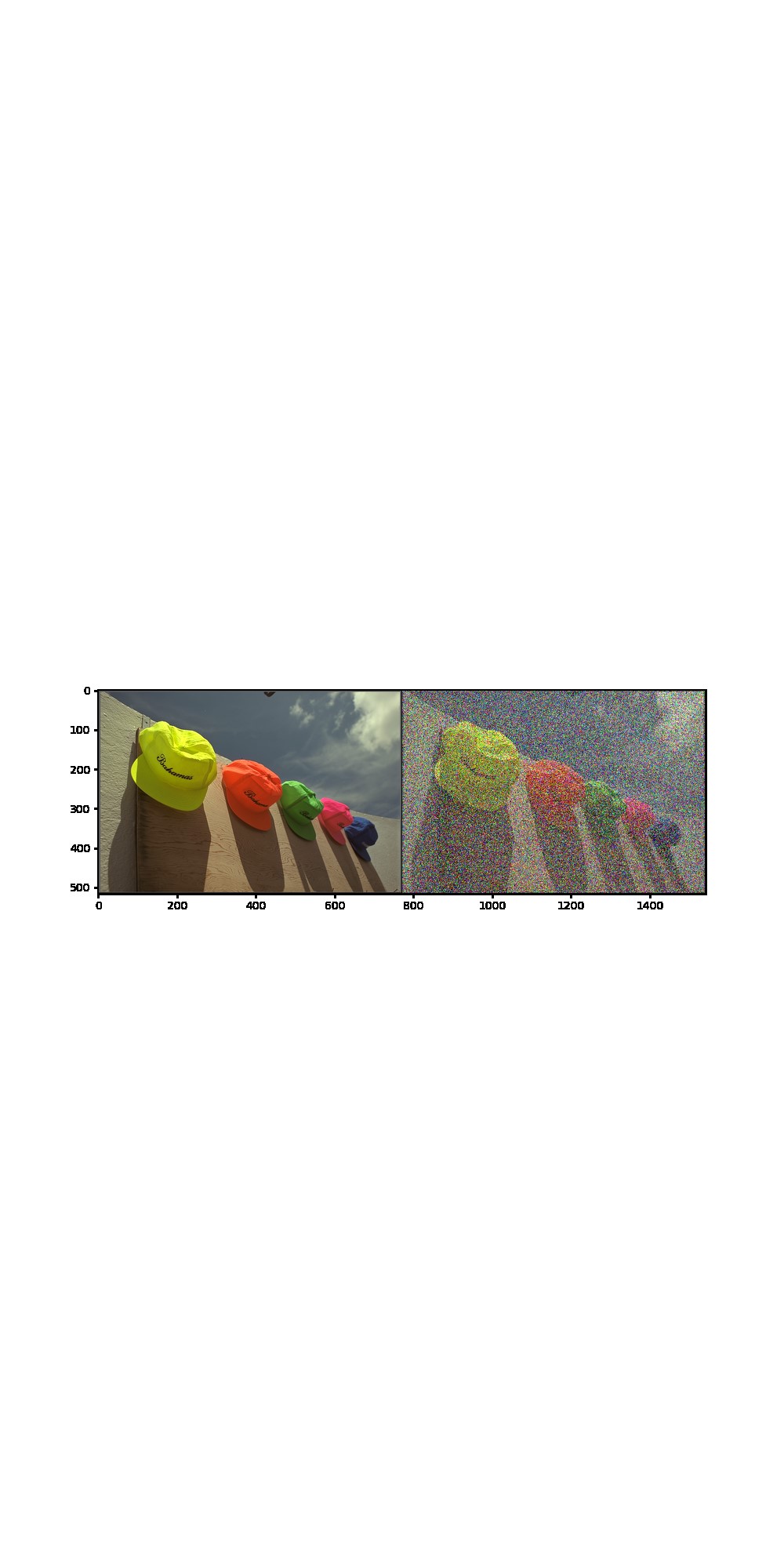}
	\end{subfigure}
	\begin{subfigure}{0.195\linewidth}
			\includegraphics[trim=4.9cm 22.1cm 14.1cm 22.2cm, clip,width=\linewidth]{figs/Ours_kodim03_true.jpg}
	\end{subfigure}
	\begin{subfigure}{0.195\linewidth}
			\includegraphics[trim=5.8cm 19.2cm 5.2cm 19.3cm, clip,width=\linewidth]{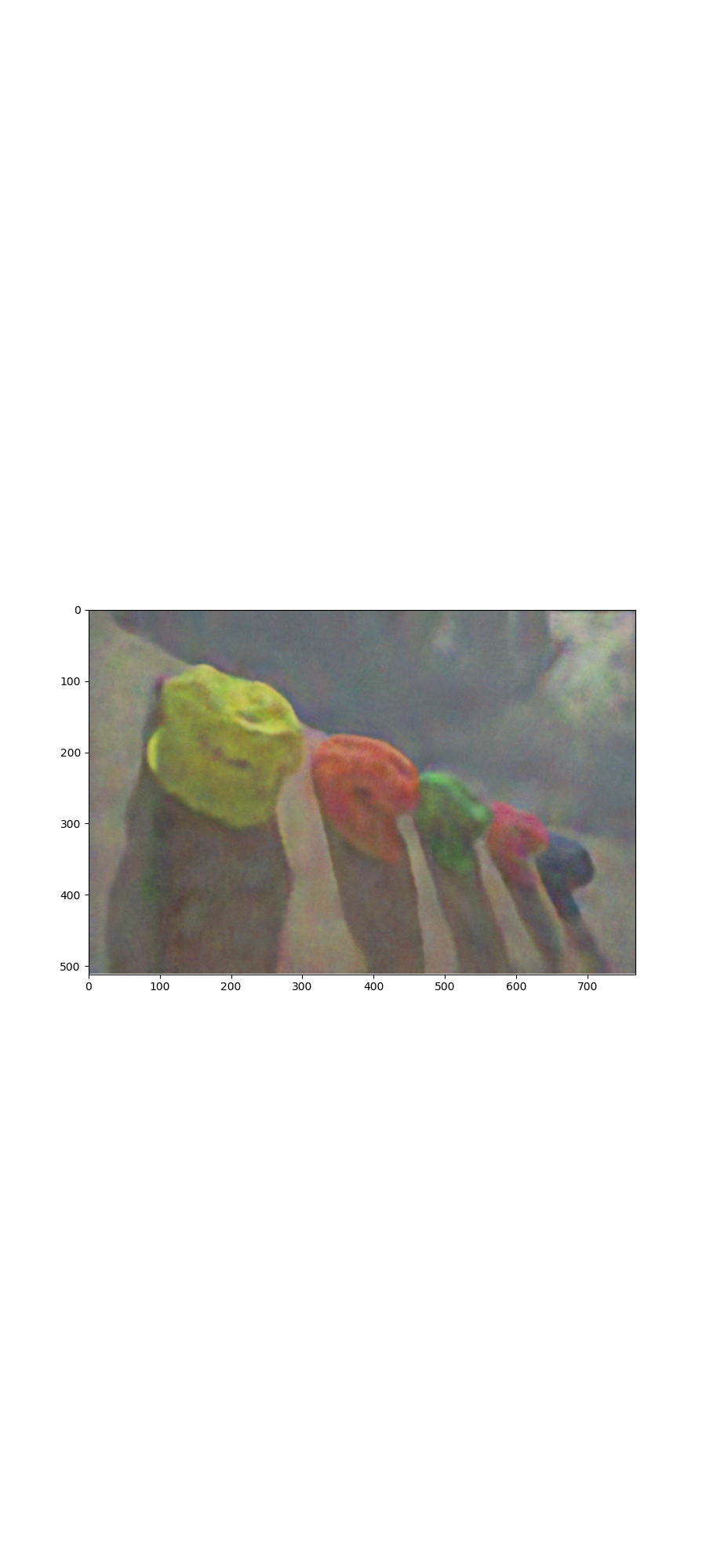}
			\caption{PSNR $=16.57$}
	\end{subfigure}
	\begin{subfigure}{0.195\linewidth}
			\includegraphics[trim=5.8cm 19.2cm 5.2cm 19.3cm, clip,width=\linewidth]{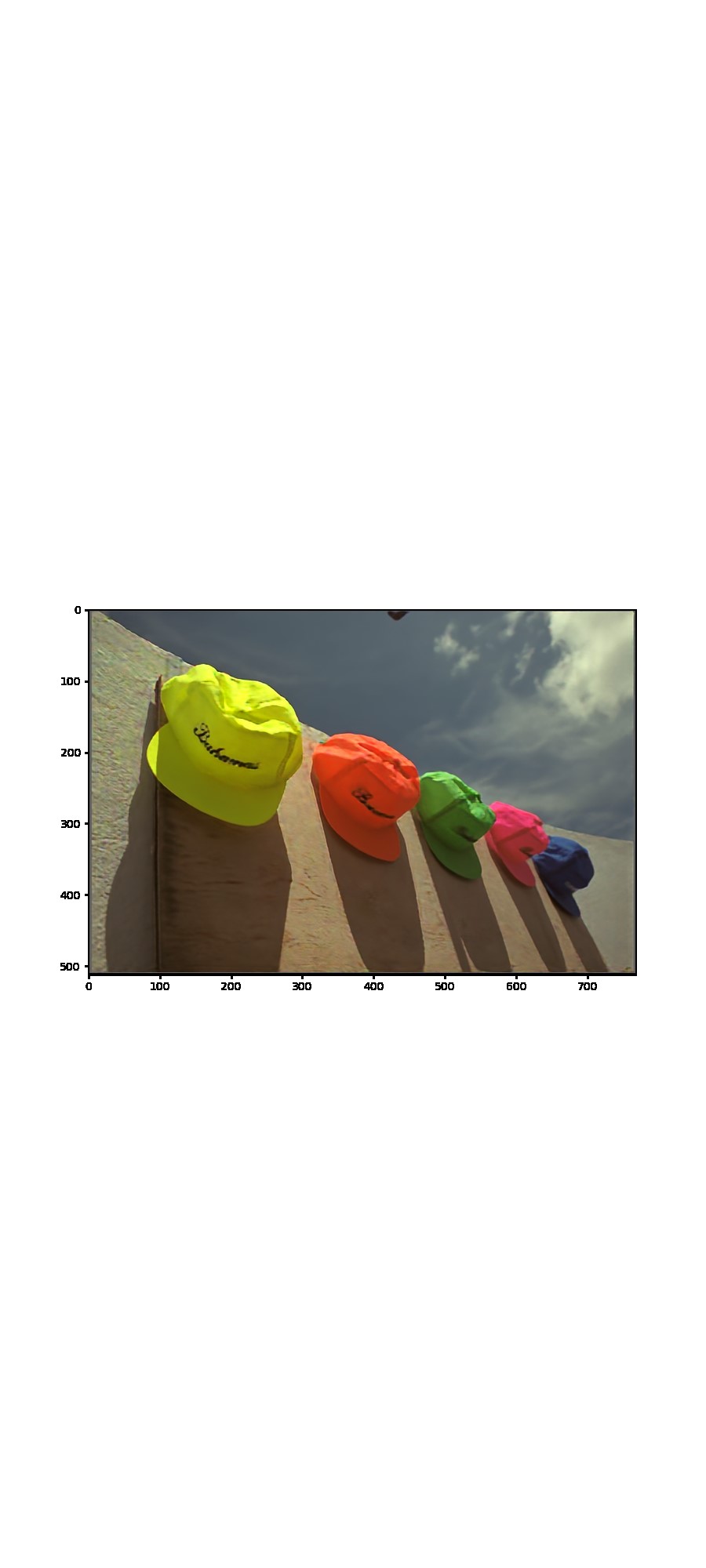}
			\caption{PSNR $=31.87$}
	\end{subfigure}			
	\begin{subfigure}{0.195\linewidth}
			\includegraphics[trim=5.8cm 19.2cm 5.2cm 19.3cm, clip,width=\linewidth]{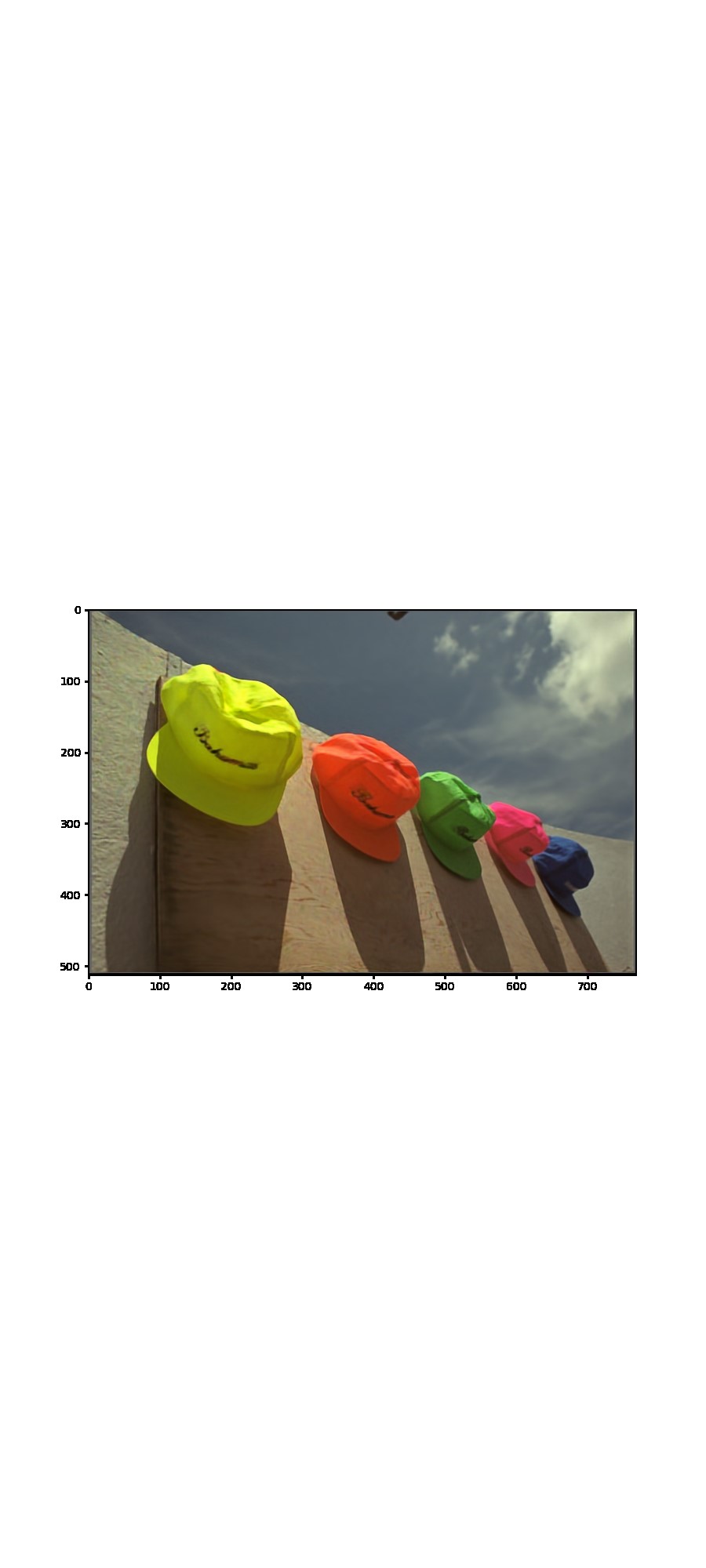}
			\caption{PSNR $=32.47$}
	\end{subfigure}
	\\
	\begin{subfigure}{0.195\linewidth}
			\caption{\small  Input}
	\end{subfigure}	
	\begin{subfigure}{0.195\linewidth}
			\caption{\small  Groundtruth}
	\end{subfigure}	
		\begin{subfigure}{0.195\linewidth}
			\caption{\small  DIP}
	\end{subfigure}	
		\begin{subfigure}{0.195\linewidth}
			\caption{\small  DIP-$\ell_1$}
	\end{subfigure}	
		\begin{subfigure}{0.195\linewidth}
			\caption{\small  Ours}
	\end{subfigure}	
	\caption{\small \textbf{Robust image recovery for salt-and-pepper corruption.} PSNR of the results is overlaid at the bottom of the images. For our method, all cases use the same network width, learning rate, and termination condition. For DIP and DIP-$\ell_1$, case-dependent early stopping is used which is essential for their good performance. Despite that, our method achieves the highest PSNRs and best visual quality.}
	\label{fig:dip}
\vspace{-12pt}
\end{figure}

\vspace{-0.05in}
\subsection{Robust Recovery of Natural Images}
\label{sec:experiment-image}
\vspace{-1mm}
Following \cite{ulyanov2018deep}, we evaluate the performance of our method for robust image recovery using images from a standard dataset\footnote{\scriptsize\url{http://www.cs.tut.fi/~foi/GCF-BM3D/index.html\#ref_results}}. Corruption for the images is synthesized by adding salt-and-pepper noise, where ratio $p$ of randomly chosen pixels are replaced with either $1$ or $0$ (each with $50\%$ probability). 

The network $\phi(\bm \theta)$ for our method in \eqref{eq:dip-double} is the same as the denoising network in \cite{ulyanov2018deep}, which has a U-shaped architecture with skip connections. Each layer of the network contains a convolutional layer, a nonlinear LeakyReLU layer and a batch normalization layer. We also follow \cite{ulyanov2018deep} on the initialization of network parameters $\bm \theta$, which is the Kaiming initialization. Meanwhile, we initialize $\vg$ and $\vh$ by i.i.d. zero-mean Gaussian distribution with a standard deviation of $10^{-5}$. We use learning rate $\tau = 1$ and set $\alpha = 500$ for all our experiments below.

\vspace{-0.05in}

\myparagraph{No need to tune model width or terminate early} We compare our method with a variant of DIP that we call DIP-$\ell_1$, which is based on solving $\min_{\bm \theta} \ell(\phi(\bm \theta) - \vy)$ with $\ell(\cdot)$ being $\norm{\cdot}{1}$. As shown in \Cref{fig:DIP-overview}, DIP-$\ell_1$ requires either choosing appropriate network width or early termination to avoid overfitting.
Note that neither of these can be carried out in practice without the true (clean) image.
On the other hand, our method does not require tuning network width or early termination. Its performance continues to improve as training proceeds until stabilises. 

\vspace{-0.05in}

\myparagraph{Handling different images and varying corruption levels}
The benefit mentioned above enables our method to handle different images and varying corruption levels \emph{without} the need to tune network width, termination condition and any other learning parameters.
In our experiments, we fix the number of channels in each convolutional layer to be $128$, run $150,\!000$ gradient descent iterations and display the output image. 
The results are shown in Figure~\ref{fig:dip} (for four different test images) and Figure~\ref{fig:dip_varying_noise} (for four corruption levels, see Appendix~\ref{sec:extra-exp}). 
For DIP and DIP-$\ell_1$, we report the result with the highest PSNR in the learning process (i.e., we perform the best early termination for evaluation purposes -- note that this cannot be carried out in practice).
Despite that, our method obtains better recovery quality in terms of PSNR for all cases.
We display the learning curves for these experiments in Figure~\ref{fig:different_figure} and Figure~\ref{fig:varying_p} (see Appendix~\ref{sec:extra-exp}), which show that our method does not overfit for all cases while DIP-$\ell_1$ requires a case-dependent early termination.

\vskip -0.15in
\section{Conclusion}\label{sec:conclusion}
\vskip -0.1in
In this work, we have shown both theoretically and empirically that the benefits of implicit bias of gradient descent can be extended to over-parameterization of two low-dimensional structures. The key to the success is the choice of discrepant learning rates that can properly regulate the optimization path so that it converges to the desired optimal solution. Such a framework frees us from the need of prior knowledge in the structures or from the lack of scalability of previous approaches. This has led to state of the art recovery results for both low-rank matrices and natural images. We hope this work may encourage people to investigate in the future if the same framework and idea generalize to a mixture of multiple and broader families of structures.




{
\medskip
\bibliographystyle{ieeetr}
\bibliography{nips,nonconvex}
}

\begin{appendices}

\section{Implicit Bias of Discrepant Learning Rates in Linear Regression}




In this part of the appendix, let us use a classical result for \emph{underdeterimined} linear regression problem to build up some intuitions behind the implicit bias of gradient descent for our problem formulation of robust learning problems. The high level message we aim to deliver through the simple example is that
\begin{itemize}[leftmargin=*]
    \item Gradient descent implicitly biases towards solutions with minimum $\ell_2$-norm.
    \item Discrepant learning rates lead to solutions with minimum \emph{weighted} $\ell_2$-norm. 
\end{itemize}

\paragraph{Underdeterimined linear regression.} Given observation $\mb b \in \bb R^{n_1} $ and wide data matrix $\mb W \in \bb R^{n_1 \times n_2}$ ($n_2>n_1$), we want to find $\bm \theta$ which is a solution to 
\begin{align}\label{eqn:ls-problem}
    \min_{ \bm \theta \in \bb R^{n_2} }  \varphi(\bm \theta) \;=\; \frac{1}{2} \norm{ \mb b -  \mb W \bm \theta }{2}^2.
\end{align}
For $n_2>n_1$ and full row-rank $\mb W$, the underdetermined problem \eqref{eqn:ls-problem} obviously has \emph{infinite} many solutions, which forms a set
\begin{align*}
    \mc S \;:=\; \Brac{ \; \bm \theta_{ln} + \mb n \; \mid\; \bm \theta_{ln} = \mb W^\dagger \mb b,\quad   \mb n \in \mc N(\mb W) \; }, 
\end{align*}
where $ \mb W^\dagger := \mb W^\top \paren{ \mb W \mb W^\top }^{-1}$ denotes the pseudo-inverse of $\mb W$, and
\begin{align*}
    \mc N(\mb W):= \Brac{ \mb n \mid \mb W \mb n = \mb 0 }, \quad \mc R(\mb W):= \Brac{ \mb z \mid \mb z = \mb W^\top \mb v  }
\end{align*}
are the null space and row space of $\mb W$, respectively. Simple derivation shows that $\bm \theta_{ln}$ is a particular \emph{least $\ell_2$-norm} solution to \eqref{eqn:ls-problem}, that minimizes 
\begin{align*}
    \min_{ \bm \theta \in \bb R^{n_2} }\; \frac{1}{2} \norm{ \bm \theta }{2}^2,\quad \text{s.t.} \quad \mb W \bm \theta \;=\; \mb b.
\end{align*}
\paragraph{Gradient descent biases towards $\bm \theta_{ln}$.} Starting from any initialization $\bm \theta_{0}$, gradient descent
\begin{align}\label{eqn:grad-descent-ls}
    \bm \theta_{k+1} \;=\; \bm \theta_{k} - \tau_{ls} \cdot \mb W^\top \paren{ \mb W \bm \theta_{k} - \mb b }
\end{align}
with a sufficiently small learning rate\footnote{This is because $\mb W \bm \theta_{k+1} - \mb b = \paren{ \mb I - \tau_{ls} \mb W \mb W^\top} \paren{ \mb W \bm \theta_{k} - \mb b } $. If we choose $\tau_{ls} < \norm{ \mb W \mb W^\top }{}^{-1} $, then $\norm{\mb W \bm \theta_{k} - \mb b}{}$ converges to $0$ geometrically.} $\tau_{ls}$ always finds one of the global solutions for \eqref{eqn:ls-problem}. Furthermore, it is now well-understood \cite{oymak2018overparameterized} that whenever the 
initialization $\bm \theta_{0}$ has zero component in $\mc N(\mb W)$ (i.e., $\mc P_{ \mc N(\mb W) }\paren{\bm \theta_{0}} = \mb 0$), one interesting phenomenon is that the iterates $\bm \theta_{\infty}$ in \eqref{eqn:grad-descent-ls} implicitly bias towards the minimium $\ell^2$-norm solution $\bm \theta_{ln}$. This happens because once initialized in $\mc R(\mb W)$, gradient descent \eqref{eqn:grad-descent-ls} implicitly biases towards iterates staying within $\mc R(\mb W)$, such that
\begin{align*}
  \mc P_{ \mc N(\mb W) }\paren{ \bm \theta_{\infty} } \;=\; \bm \theta_{ln},\quad  \mc P_{ \mc N(\mb W) }\paren{ \bm \theta_{\infty} } \;=\; \mc P_{ \mc N(\mb W) }\paren{ \bm \theta_{0} } = \mb 0.
\end{align*}
As we can see, a particular algorithm enforces specific regularization on the final solution. 

\myparagraph{Implicit bias of discrepant learning rates}
The gradient update in \eqref{eqn:grad-descent-ls} uses the same learning rate $\tau_{ls}$ for all coordinates of $\bm \theta$. 
If we use different learning rates for each coordinate (i.e., $\mb \Lambda$ is a diagonal matrix with positive diagonals)
\begin{align}\label{eqn:grad-descent-ls-w}
    \bm \theta_{k+1} \;=\; \bm \theta_{k} - \tau_{ls} \cdot \mb \Lambda \cdot \mb W^\top \paren{ \mb W \bm \theta_{k} - \mb b },
\end{align}
then by following a similar argument we conclude that the gradient update in \eqref{eqn:grad-descent-ls-w} converges to a weighted regularized solution for
\begin{align}
    \min_{ \bm \theta \in \bb R^{n_2} }\; \frac{1}{2} \norm{  \mb \Lambda^{-1/2} \bm \theta }{2}^2,\quad \text{s.t.} \quad \mb W \bm \theta \;=\; \mb b.
\label{eq:weighted-ls}\end{align}

\remark Let $\sigma_i$ be the $i$-th diagonal of $\mLambda$ and $\theta_i$ be the $i$-th element of $\bm \theta$. Then in \eqref{eqn:grad-descent-ls-w}, $\sigma_i\tau_{ls}$ is the learning rate for the variable $\theta_i$, which varies for different variables.
In words, the relation between \eqref{eqn:grad-descent-ls-w} and \eqref{eq:weighted-ls} implies that for one particular optimization variable (e.g., $\theta_i$) a large learning rate $\sigma_i\tau_{ls}$ in \eqref{eqn:grad-descent-ls-w} leads to a small implicit regularization effect in \eqref{eq:weighted-ls}. From a high-level perspective, this happens because a larger learning rate allows the optimization variable to move faster away from its initial point, resulting in a weaker regularization effect (which penalizes the distance of the variable to the 
initialization) on its solution path. 

An alternative explanation of this is through a change of variable $ {\bm \theta} = \bm \Lambda^{1/2}\wt{\bm \theta}$. Suppose we minimize 
\begin{align}\label{eqn:weighted-ls-problem}
    \min_{ \wt{\bm \theta} \in \bb R^{n_2} }  \wt\varphi(\wt{\bm \theta}) \;=\; \frac{1}{2} \norm{ \mb b -  \mb W \bm \Lambda^{1/2}\wt{\bm \theta} }{2}^2
\end{align}
via standard gradient descent with a single learning rate
\begin{align}\label{eqn:grad-descent-weighted-ls}
    \wt{\bm \theta}_{k+1} \;=\; \wt{\bm \theta}_{k} - \tau_{ls} \cdot \mb \Lambda^{1/2} \cdot \mb W^\top \paren{ \mb W \bm \Lambda^{1/2} \wt{\bm \theta}_{k} - \mb b }.
\end{align}
Thus, once initialized in $\mc R(\mb W \bm \Lambda^{1/2})$, gradient descent \eqref{eqn:grad-descent-weighted-ls} converges to the least $\ell_2$-norm solution to \eqref{eqn:weighted-ls-problem}, i.e., the solution of the following problem
\begin{align}\label{eqn:weighted-ls-min-norm}
    \min_{\wt{ \bm \theta} \in \bb R^{n_2} }\; \frac{1}{2} \norm{ \wt{\bm \theta} }{2}^2,\quad \text{s.t.} \quad \mb W \bm \Lambda^{1/2}\wt{\bm \theta} \;=\; \mb b.
\end{align}
Finally, plugging  $\wt{\bm \theta} = \bm \Lambda^{-1/2}{\bm \theta}$ into \eqref{eqn:grad-descent-weighted-ls} and \eqref{eqn:weighted-ls-min-norm} gives \eqref{eqn:grad-descent-ls-w} and \eqref{eq:weighted-ls}, respectively, also indicating the gradient update \eqref{eqn:grad-descent-ls-w} induces implicit weighted regularization towards the solution of \eqref{eq:weighted-ls}.

\section{Proof of \Cref{thm:dynamic-rms}}\label{sec:proof}

In this part of the appendix, we prove our main technical result (i.e., \Cref{thm:dynamic-rms}) in \Cref{sec:dynamical}. To make this part self-contained, we restate our result as follows.

\begin{theorem}\label{thm:main-app}
Assume that the measurement matrices $\mA_1,\mA_2,\ldots,\mA_m$ commute, i.e.
\begin{align*}
    \mb A_i \mb A_j \; =\; \mb A_j \mb A_i,\quad \forall \; 1 \leq i \not= j \leq m,
\end{align*}
and the gradient flows of $\mb U_t(\gamma)$, $\mb g_t(\gamma)$, and $\mb h_t(\gamma)$ satisfy 
\begin{align}
    \dot{ \mb U }_t(\gamma) \;&=\; \lim_{ \tau \rightarrow 0 } \frac{ \mb U_{t +\tau }(\gamma) - \mb U_{t }(\gamma) }{\tau } \;=\; - \mc A^* \paren{ \mb r_t(\gamma)  } \mb U_t(\gamma), \label{eqn:U-t-app} \\
     \begin{bmatrix} \dot{\mb g}_t(\gamma) \\ \dot{\mb h}_t(\gamma) \end{bmatrix} \;&=\; \lim_{ \tau \rightarrow 0 } \paren{ \begin{bmatrix} \mb g_{t+\tau}(\gamma) \\ \mb h_{t+\tau}(\gamma) \end{bmatrix} - \begin{bmatrix} \mb g_t(\gamma) \\ \mb h_t(\gamma) \end{bmatrix} } \;/\; {\tau } \;=\;  -\; \alpha \cdot \begin{bmatrix}
     \mb r_t(\gamma) \circ \mb g_t(\gamma) \\
     - \mb r_t(\gamma) \circ \mb h_t(\gamma),
    \end{bmatrix}, \label{eqn:g-h-t-app}
\end{align}
with $\mb r_t(\gamma) \;=\; \mc A( \mb U_t(\gamma) \mb U_t^\top(\gamma) ) + \mb g_t(\gamma) \circ \mb g_t(\gamma) - \mb h_t(\gamma) \circ \mb h_t(\gamma) - \mb y$, and they are initialized by 
\begin{align*}
    \mb U_0(\gamma) \;=\; \gamma \mb I, \quad \mb g_0(\gamma) \;=\; \gamma \mb 1,\quad h_0(\gamma) \;=\; \gamma \mb 1.
\end{align*}
Let $\mb X_t(\gamma) = \mb U_t(\gamma)\mb U_t^\top (\gamma)$, and let $ \mb X_\infty\paren{ \gamma }$, $\mb g_\infty(\gamma)$, and $\mb g_\infty(\gamma)$ be the limit points defined as
\begin{align}\label{eqn:X-h-g-infty-app}
    \mb X_\infty\paren{ \gamma } \;:=\; \lim_{ t \rightarrow +\infty } \mb X_t(\gamma),\quad \mb g_\infty(\gamma) \;:=\; \lim_{ t \rightarrow +\infty } \mb g_t(\gamma),\quad \mb h_\infty(\gamma) \;:=\; \lim_{ t \rightarrow +\infty } \mb h_t(\gamma).
\end{align}
Suppose that our initialization is infinitesimally small such that
\begin{align*}
    \wh{\mb X} \;:=\; \lim_{\gamma \rightarrow 0 } \mb X_{\infty}(\gamma),\quad \wh{\mb g} \;:=\; \lim_{\gamma \rightarrow 0 } \mb g_{\infty}(\gamma),\quad \wh{\mb h} \;:=\; \lim_{\gamma \rightarrow 0 } \mb h_{\infty}(\gamma)
\end{align*}
exist and $( \wh{\mb X}, \wh{\mb g}, \wh{\mb h} )$ is a global optimal solution to
\begin{align}\label{eq:rms-over-app}
    \min_{\mb U \in \bb R^{n \times r'} , \{\mb g, \mb h\} \subseteq \bb R^m }\; f(\mb U, \mb g,\mb h) \;:=\; \frac{1}{4}  \norm{  \mc A\paren{\mb U \mb U^\top} + \left(\mb g \circ \mb g - \mb h \circ \mb h\right) - \mb y }{2}^2,
\end{align}
with $ \mc A(\wh{\mb X}) + \wh{\mb s} \;=\; \mb y$ and $ \wh{\mb s}\;=\; \wh{\mb g} \circ \wh{\mb g} - \wh{\mb h} \circ \wh{\mb h}$.
Then we have $\wh{\mb g} \circ \wh{\mb h} = \mb 0$, and $(\wh{\mb X},\wh{\mb s} )$ is also a global optimal solution to 
\begin{align}\label{eqn:app-rpca-cvx}
    \min_{ \mb X \in \bb R^{n\times n}, \mb s \in \bb R^m }\; \norm{ \mb X }{*} \;+\; \lambda \cdot \norm{ \mb s }{1},\quad \text{s.t.}\quad \mc A(\mb X)+ \mb s \;=\; \mb y,\;\; \mb X\succeq \mb 0.
\end{align}
with $\lambda = \alpha^{-1}$ and $\alpha>0$ being the balancing parameter in \eqref{eqn:g-h-t-app}.
\end{theorem}

\begin{proof}
From \eqref{eqn:U-t-app}, we can derive the gradient flow for $\mb X_t(\gamma)$ via chain rule
\begin{align}\label{eqn:X-dot-app}
    \dot{\mb X}_t(\gamma) \;=\; \dot{\mb U}_t(\gamma) \mb U_t^\top(\gamma) \;+\; \mb U_t(\gamma) \dot{\mb U}_t^\top(\gamma) \;=\; - \mc A^* (\mb r_t(\gamma)) \mb X_t(\gamma) - \mb X_t(\gamma) \mc A^*(\mb r_t(\gamma)).
\end{align}
We want to show that when the initialization is infinitesimally small (i.e., $\gamma \rightarrow 0$), the limit points of the gradient flows $\mb X_t(\gamma) = \mb U_t(\gamma) \mb U_t^\top(\gamma) $ and $\mb s_t(\gamma) = \mb g_t(\gamma) \circ \mb g_t(\gamma) - \mb h_t(\gamma) \circ \mb h_t(\gamma) $ are optimal solutions for \eqref{eqn:app-rpca-cvx} as $t \rightarrow +\infty $. Towards this goal, let us first look at the optimality condition for \eqref{eqn:app-rpca-cvx}. From Lemma \ref{lem:optimality}, we know that if $( \wh{\mb X},\wh{\mb s})$ with
\begin{align*}
    \wh{\mb X} \;=\; \lim_{\gamma \rightarrow 0 } \mb X_{\infty}(\gamma), \quad \wh{\mb s} \;=\; \wh{\mb g} \circ \wh{\mb g} - \wh{\mb h} \circ \wh{\mb h} \quad \text{with}\quad \wh{\mb g} \;:=\; \lim_{\gamma \rightarrow 0 } \mb g_{\infty}(\gamma),\quad \wh{\mb h} \;:=\; \lim_{\gamma \rightarrow 0 } \mb h_{\infty}(\gamma)
\end{align*}
is an optimal solution for \eqref{eqn:app-rpca-cvx} then there exists a \emph{dual certificate} $\bm \nu$ such that
\begin{align*}
    \mc A(\wh{\mb X}) + \wh{\mb s} \;=\; \mb y,\quad \paren{ \mb I - \mc A^*(\bm \nu) } \cdot \wh{\mb X} \;=\; \mb 0,\quad \mc A^*(\bm \nu)\;\preceq \; \mb I  ,\quad \bm \nu \in \lambda \cdot \sign(\wh{\mb s} ), \quad \wh{\mb X} \succeq \mb 0,
\end{align*}
where $\sign(\wh{\mb s} )$ is defined in \eqref{eqn:sign-operator}, which is the subdifferential of $\norm{ \cdot }{1}$. 
Thus, it suffices to construct a dual certificate $\bm \nu$ such that $(\wh{\mb X},\wh{\mb s})$ satisfies the equation above. 

Since $( \wh{\mb X}, \wh{\mb g}, \wh{\mb h} )$ is a global optimal solution to \eqref{eq:rms-over-app}, we automatically have $\mc A(\wh{\mb X}) + \wh{\mb s} \;=\; \mb y$ and $\wh{\mb X} \succeq \mb 0 $. On the other hand, given that $\Brac{\mb A_i}_{i=1}^m$ commutes and \eqref{eqn:X-dot-app} and \eqref{eqn:g-h-t-app} hold for $\mb X_t$, $\mb g_t$ and $\mb h_t$, by Lemma \ref{lem:X-g-h-t}, we know that
\begin{align}
    \mb X_t(\gamma) \; &=\; \exp\paren{ \mc A^*(\bm \xi_t(\gamma)) } \cdot \mb X_0(\gamma) \cdot \exp\paren{ \mc A^*(\bm \xi_t(\gamma)) } , \label{eqn:X-t-0} \\
    \mb g_t(\gamma) \;& =\; \mb g_0(\gamma) \circ \exp \paren{ \alpha \bm \xi_t(\gamma) },\quad \mb h_t(\gamma) \; =\; \mb h_0(\gamma) \circ \exp \paren{ - \alpha \bm \xi_t(\gamma) }, \label{eqn:g-h-t-0}
\end{align}
where $\bm \xi_t(\gamma) = - \int_0^t \mb r_\tau(\gamma) d \tau $. Let $\bm \xi_\infty(\gamma) := \lim_{t \rightarrow +\infty } \bm \xi_t(\gamma)$, by Lemma \ref{lem:dual-1} and Lemma \ref{lem:dual-2}, we can construct
\begin{align*}
    \bm \nu(\gamma) \;=\; \frac{\bm \xi_\infty(\gamma)}{ \log \paren{ 1/\gamma }},
\end{align*}
such that
\begin{align*}
    \lim_{ \gamma \rightarrow 0 } \mc A^* \paren{ \bm \nu(\gamma) } \; \preceq \; \mb I, \quad \lim_{ \gamma \rightarrow 0 } \brac{ \mb I - \mc A^* \paren{ \bm \nu(\gamma) } } \cdot \wh{\mb X} \;=\; \mb 0,
\end{align*}
and
\begin{align*}
    \lim_{\gamma\rightarrow 0} \; \bm \nu(\gamma) \;\in\; \alpha^{-1} \cdot \sign( \wh{\mb s} ),  \quad \lim_{\gamma \rightarrow 0 } \mb g(\gamma) \circ \mb h(\gamma) \;=\; \mb 0.
\end{align*}
This shows the exists of the dual certificate $\bm \nu(\gamma)$ such that the optimality condition holds for $(\wh \mX, \wh \vs)$. Hence, $(\wh \mX, \wh \vs)$ is also a global optimal solution to \eqref{eqn:app-rpca-cvx}.
\end{proof}

\begin{lemma}\label{lem:optimality}
If $( \wh{\mb X}, \wh{\mb s} )$ is an optimal solution for \eqref{eqn:app-rpca-cvx}, then there exists a dual certificate $\bm \nu \in \bb R^m$, that
\begin{align}\label{eqn:cvx-kkt}
    \mc A(\wh{\mb X}) + \wh{\mb s} \;=\; \mb y,\quad \paren{ \mb I - \mc A^*(\bm \nu) } \cdot \wh{\mb X} \;=\; \mb 0,\quad \bm \nu \in \lambda \cdot \sign(\wh{\mb s} ), \quad \mb I \succeq \mc A^*(\bm \nu),\; \wh{\mb X} \succeq \mb 0,
\end{align}
where $\sign(\mb s)$ is the subdifferential of $\norm{\mb s}{1}$ with each entry
\begin{align}\label{eqn:sign-operator}
    \sign(s) \;:=\; \begin{cases} s/\abs{s} & s \not = 0, \\ [-1,1] & s =0.  \end{cases}
\end{align}

\end{lemma}

\begin{proof}
The Lagrangian function of the problem can be written as
\begin{align*}
    \mc L\paren{ \mb X, \mb s, \bm \nu, \mb \Gamma } \;=\; \trace\paren{ \mb X } + \lambda \norm{ \mb s }{1} + \bm \nu^\top \paren{ \mb y - \mc A(\mb X) - \mb s } - \innerprod{\mb X}{ \mb \Gamma }, 
\end{align*}
with $\bm \nu \in \bb R^m $ and $\mb \Gamma \in \bb R^{n\times n}$ being the dual variables, where $\mb \Gamma \succeq \mb 0 $. Thus, we can derive the Karush-Kuhn-Tucker (KKT) optimality condition for \eqref{eqn:app-rpca-cvx} as
\begin{align*}
    \mb 0 \in \partial \mc L\;: & \quad \mb I - \mb \Gamma - \mc A^*(\bm \nu) \;=\; \mb 0, \quad \bm \nu \in  \lambda \cdot \partial \norm{\mb s}{1} = \lambda \cdot \sign \paren{ \mb s } , \\
    \text{feasibility\;:} & \quad  \mc A(\mb X) + \mb s \;=\; \mb y,\; \mb X \succeq \mb 0, \; \mb \Gamma \succeq \mb 0, \\
    \text{complementary slackness\;:} & \quad \mb \Gamma \cdot \mb X \;=\; \mb 0,
\end{align*}
where $\partial(\cdot)$ denotes the subdifferential operator and $\sign(\mb s)$ is the subdifferential of $\norm{\mb s}{1}$ with each entry
\begin{align*}
    \sign(s) \;=\; \begin{cases} s/\abs{s} & s \not = 0, \\ [-1,1] & s =0.  \end{cases}
\end{align*}
Thus, we know that $\paren{ \wh{\mb X}, \wh{\mb s} }$ is global solution to \eqref{eqn:app-rpca-cvx} as long as there exists a \emph{dual certificate} $\bm \nu$ such that \eqref{eqn:cvx-kkt} holds, where we eliminated $\mb \Gamma$ by plugging in $\mb \Gamma = \mb I - \mc A^*(\bm \nu)$. 
\end{proof}

\begin{lemma}\label{lem:X-g-h-t}
Suppose that $\Brac{\mb A_i}_{i=1}^m$ commutes. Suppose \eqref{eqn:X-dot} and \eqref{eqn:g-h-t-app} hold for $\mb X_t$, $\mb g_t$ and $\mb h_t$, then
\begin{align}
    \mb X_t \; &=\; \exp\paren{ \mc A^*(\bm \xi_t) } \cdot \mb X_0 \cdot \exp\paren{ \mc A^*(\bm \xi_t) } \label{eqn:X-t}  \\
    \mb g_t \;& =\; \mb g_0 \circ \exp \paren{ \alpha \bm \xi_t },\quad \mb h_t \; =\; \mb h_0 \circ \exp \paren{ - \alpha \bm \xi_t },
\end{align}
where $\bm \xi_t = - \int_0^t \mb r_\tau d \tau $.
\end{lemma}

\begin{proof}
From \eqref{eqn:g-h-t-app}, we know that
\begin{align*}
    \frac{d \mb g_t }{ dt } \;=\; - \alpha \mb r_t \circ \mb g_t,
\end{align*}
where the differentiation $\frac{d \mb g_t  }{dt}$ is entrywise for $\mb g_t$. Thus, we have
\begin{align*}
    \int_0^t \frac{d \mb g_\tau}{ \mb g_\tau } \;=\; - \alpha \int_0^t  \mb r_{\tau} d\tau \quad  \Longrightarrow \quad  \log \mb g_t - \log \mb g_0 \;=\;  \alpha \bm \xi_t \quad \Longrightarrow \quad \mb g_t \;=\; \mb g_0 \circ \exp \paren{ \alpha \bm \xi_t },
\end{align*}
where all the operators are entrywise. Similarly, $\mb h_t \; =\; \mb h_0 \circ \exp \paren{ - \alpha \bm \xi_t }$ holds. 

For \eqref{eqn:X-t}, by using \eqref{eqn:X-dot-app} and the fact that $\Brac{\mb A_i}_{i=1}^m$ commutes, we can derive it with an analogous argument.
\end{proof}

\begin{lemma}\label{lem:dual-1}
Under the settings of \Cref{thm:main-app} and Lemma \ref{lem:X-g-h-t}, for any $\gamma>0$ there exists  
\begin{align}\label{eqn:nu-construct}
    \bm \nu ( \gamma ) \;=\;  \frac{\bm \xi_\infty(\gamma)}{ \log \paren{ 1/\gamma }},
\end{align}
such that
\begin{align*}
    \lim_{ \gamma \rightarrow 0 } \mc A^* \paren{ \bm \nu(\gamma) } \; \preceq \; \mb I, \quad \lim_{ \gamma \rightarrow 0 } \brac{ \mb I - \mc A^* \paren{ \bm \nu(\gamma) } } \cdot \wh{\mb X} \;=\; \mb 0,
\end{align*}
where $\bm \xi_\infty(\gamma) = \lim_{t \rightarrow 0 } \bm \xi_t(\gamma) $ with  $\bm \xi_t(\gamma) = - \int_0^t \mb r_\tau(\gamma) d \tau $.
\end{lemma}

\begin{proof}
Given $\mb U_0 = \gamma \mb I$, we have $\mb X_0 = \mb U_0 \mb U_0^\top = \gamma^2 \mb I$. By the expression for $\mb X_t$ in \eqref{eqn:X-t-0}, we have
\begin{align}
    \mb X_\infty(\gamma) \;=\; \gamma^2 \cdot \exp\paren{ 2 \mc A^*\paren{ \bm \xi_\infty(\gamma ) } }  
\end{align}
where $\bm \xi_\infty(\gamma ) = \lim_{ t \rightarrow \infty } \bm \xi_t(\gamma) $. Because $\Brac{\mb A_i}_{i=1}^m$ are symmetric and they commute, we know that they are simultaneously diagonalizable by an orthonormal basis $\mb \Omega = [ \bm \omega_1, \ldots,  \bm \omega_n ]
\in \bb R^{n \times n}$, i.e., 
\begin{align*}
    \mb \Omega \mb A_i \mb \Omega^\top = \mb \Lambda_i, \quad \mb \Lambda_i \text{ diagonal }, \quad \forall\; i = 1,2, \ldots, m,
\end{align*}
and so is $\mc A^*(\mb b)$ for any $\mb b \in \bb R^m$. 
Therefore, we have
\begin{align}\label{eqn:lambda-k-form}
    \lambda_k \paren{ \mb X_\infty(\gamma) } \;=\; \gamma^2 \cdot \exp\paren{ 2 \lambda_k\paren{ \mc A^*\paren{ \bm \xi_\infty(\gamma ) } }  } \;=\; \exp\paren{ 2 \lambda_k\paren{ \mc A^*\paren{ \bm \xi_\infty(\gamma ) } } + 2 \log \gamma },
\end{align}
where $\lambda_k(\cdot)$ denotes the $k$-th eigenvalue with respect to the $k$-th basis $\bm \omega_k$. 
Moreover $\mb X_\infty(\gamma)$ and its limit $\wh{\mb X}$ have the same eigen-basis $\mb \Omega$. Since we have $\mb X_\infty (\gamma)$ converges to $\wh{\mb X}$ when $\gamma \rightarrow 0$, then we have the eigenvalues
\begin{align}\label{eqn:X-gamma-converge}
   \lambda_k \paren{ \mb X_\infty(\gamma) } \; \rightarrow \; \lambda_k( \wh{\mb X} ),\quad \forall\; k\;=\;1,2,\ldots,n,
\end{align}
whenever $\gamma \rightarrow 0$.

\paragraph{Case 1: $\lambda_k(\wh{\mb X})>0$.} For any $k$ such that $\lambda_k(\wh \mX)>0$, from \eqref{eqn:lambda-k-form} and \eqref{eqn:X-gamma-converge}, we have
\begin{align*}
    \exp\paren{ 2 \lambda_k\paren{ \mc A^*\paren{ \bm \xi_\infty(\gamma ) } } + 2 \log \gamma } \; \rightarrow \; \lambda_k( \wh{\mb X} ),
\end{align*}
so that
\begin{align*}
    2 \lambda_k\paren{ \mc A^*\paren{ \bm \xi_\infty(\gamma ) } } + 2 \log \gamma - \log \lambda_k( \wh{\mb X} ) \; \rightarrow \; 0,
\end{align*}
which further implies that
\begin{align*}
    \lambda_k\paren{ \mc A^*\paren{ \frac{ \bm \xi_\infty(\gamma ) }{ \log(1/\gamma)  }  }  } - 1 - \frac{ \log \lambda_k(\wh{\mb X}) }{2 \log (1/\gamma) } \; \rightarrow \; 0.
\end{align*}
Now if we construct $\bm \nu ( \gamma ) $ as \eqref{eqn:nu-construct}, so that we conclude
\begin{align}\label{eqn:lambda-1}
    \lim_{\gamma \rightarrow 0 } \; \lambda_k \paren{ \mc A^*( \bm \nu(\gamma) ) } \;=\; 1, 
\end{align}
for any $k$ such that $\lambda_k( \wh{\mb X} )>0$.

\paragraph{Case 2: $\lambda_k(\wh{\mb X})=0$.} On the other hand, for any $k$ such that $ \lambda_k( \wh{\mb X} )=0 $, similarly from \eqref{eqn:lambda-k-form} and \eqref{eqn:X-gamma-converge}, we have
\begin{align*}
    \exp\paren{ 2 \lambda_k\paren{ \mc A^*\paren{ \bm \xi_\infty(\gamma ) } } + 2 \log \gamma } \;\rightarrow \; 0,
\end{align*}
when $\gamma \rightarrow 0$. Thus, for any small $\epsilon \in (0,1)$, there exists some $\gamma_0 \in (0,1)$ such that for all $\gamma < \gamma_0$,
\begin{align*}
    \exp\paren{ 2 \lambda_k\paren{ \mc A^*\paren{ \bm \xi_\infty(\gamma ) } } + 2 \log \gamma } \;< \; \epsilon,
\end{align*}
which implies that
\begin{align*}
    \lambda_k\paren{ \mc A^*\paren{ \frac{ \bm \xi_\infty(\gamma ) }{ \log(1/\gamma) } } } \;-\;1 \;<\; \frac{\log \epsilon }{ 2 \log \paren{ 1/\gamma } } \;<\; 0.
\end{align*}
Thus, given the construction of $\bm \nu(\gamma)$ in \eqref{eqn:nu-construct}, we have
\begin{align*}
    \lambda_k \paren{ \mc A^*( \bm \nu(\gamma) ) } \;<\;1, \quad \forall \; \gamma \;<\; \gamma_0,
\end{align*}
which further implies that for any $k$ with $\lambda_k(\wh \mX) =0$, we have
\begin{align}\label{eqn:lambda-2}
    \lim_{\gamma \rightarrow 0 } \lambda_k \paren{ \mc A^*( \bm \nu(\gamma) ) } \;<\;1.
\end{align}

\paragraph{Putting things together.} Combining our analysis in \eqref{eqn:lambda-1} and \eqref{eqn:lambda-2}, we obtain
\begin{align*}
    \lim_{ \gamma \rightarrow 0 } \mc A^*( \bm \nu(\gamma) ) \;\preceq \; \mb I.
\end{align*}
On the other hand, per our analysis, we know that there exists an orthogonal matrix $\mb \Omega \in \bb R^{n \times n }$, such that $\mc A^*( \bm \nu(\gamma) )$ and  $\wh{\mb X}$ can be simultaneously diagonalized. Thus, we have
\begin{align*}
   \brac{ \mb I - \mc A^*( \bm \nu(\gamma) ) } \cdot \wh{\mb X} \;=\; \mb \Omega \cdot \paren{ \mb I - \mb \Lambda_{ \mc A^*( \bm \nu(\gamma) ) } }  \cdot  \mb \Lambda_{ \wh{\mb X} }  \cdot \mb \Omega^\top,
\end{align*}
where $\mb \Lambda_{ \mc A^*( \bm \nu(\gamma) ) }$ and $\mb \Lambda_{ \wh{\mb X} }$ are diagonal matrices, with entries being the eigenvalues of $\mc A^*( \bm \nu(\gamma) )$ and $\mb \Lambda_{ \wh{\mb X} }$, respectively. From our analysis for Case 1 and Case 2, we know that $ \lim_{ \gamma \rightarrow 0 }  \paren{ \mb I - \mb \Lambda_{ \mc A^*( \bm \nu(\gamma) ) } }  \cdot  \mb \Lambda_{ \wh{\mb X} } = \mb 0 $. Therefore, we also have
\begin{align*}
   \lim_{ \gamma \rightarrow 0 }\; \brac{ \mb I - \mc A^*( \bm \nu(\gamma) ) } \cdot \wh{\mb X} \;=\; \mb 0,
\end{align*}
as desired.
\end{proof}

\begin{lemma}\label{lem:dual-2}
Under the settings of \Cref{thm:main-app} and Lemma \ref{lem:X-g-h-t}, for any $\gamma>0$ there exists  
\begin{align}
    \bm \nu ( \gamma ) \;=\;  \frac{\bm \xi_\infty(\gamma)}{ \log \paren{ 1/\gamma }},
\end{align}
such that
\begin{align}\label{eqn:nu-sign}
    \lim_{\gamma\rightarrow 0} \; \bm \nu(\gamma) \;\in\; \alpha^{-1} \cdot \sign( \wh{\mb s} ),  \quad \lim_{\gamma \rightarrow 0 } \mb g(\gamma) \circ \mb h(\gamma) \;=\; \mb 0,
\end{align}
where $\wh{\mb s} = \wh{\mb g} \circ \wh{\mb g} - \wh{\mb h} \circ \wh{\mb h} $, and $\bm \xi_\infty(\gamma) = \lim_{t \rightarrow 0 } \bm \xi_t(\gamma) $ with  $\bm \xi_t(\gamma) = - \int_0^t \mb r_\tau(\gamma) d \tau $.
\end{lemma}

\begin{proof}
Let $g_\infty^i(\gamma) $ and $h_\infty^i(\gamma) $ be the $i$th coordinate of $ \mb g_\infty(\gamma) $ and $ \mb h_\infty(\gamma) $ defined in \eqref{eqn:X-h-g-infty-app}, respectively. It follows from \eqref{eqn:g-h-t-0} that
\begin{align*}
    g_\infty^i(\gamma) \;=\; \gamma \cdot \exp\paren{ \alpha \cdot \xi_\infty^i(\gamma) } ,\quad h_\infty^i(\gamma) \;=\; \gamma \cdot \exp\paren{- \alpha \cdot \xi_\infty^i(\gamma) },\quad \forall \; i \;=\; 1,2, \ldots, m.
\end{align*}
When $\gamma \rightarrow 0$, we have
\begin{align*}
    g_\infty^i(\gamma) \cdot h_\infty^i(\gamma) \;=\; \gamma^2 \; \rightarrow \; 0,\quad \forall \; i \;=\; 1,2,\ldots,m,
\end{align*}
so that $\lim_{\gamma \rightarrow 0 } \mb g(\gamma) \circ \mb h(\gamma) \;=\; \mb 0$. This also implies that either $g_\infty^i(\gamma)$ or $h_\infty^i(\gamma)$ for any $i=1,2,\ldots,m$.

On the other hand, let us define 
\begin{align*}
    \mb s_\infty(\gamma)\;=\; \mb h_\infty(\gamma) \circ \mb h_\infty(\gamma) \; - \; \mb g_\infty(\gamma) \circ \mb g_\infty(\gamma),
\end{align*}
and let $ s_\infty^i(\gamma)$ be the $i$th coordinate of $\mb s_\infty(\gamma)$ with
\begin{align}\label{eqn:s-infty-i}
    s_\infty^i(\gamma) \;=\; \gamma^2 \cdot \exp\paren{ 2 \alpha \cdot  \xi_\infty^i(\gamma) } \; - \; \gamma^2 \cdot \exp\paren{ - 2 \alpha \cdot  \xi_\infty^i(\gamma) } .
\end{align}
Correspondingly, we know that $\wh{\mb s} = \lim_{\gamma\rightarrow 0 }  \mb s_\infty(\gamma) $ and let $s^i$ be the $i$th coordinate of $\wh{\mb s} = \wh{\mb g} \circ \wh{\mb g} - \wh{\mb h } \circ \wh{\mb h}$. In the following, we leverage on these to show that our construction of $\bm \nu(\gamma)$ satisfies \eqref{eqn:nu-sign}. We classify the entries $\wh{s}_i$ of $\wh{\mb s}$ $(i=1,2,\ldots,m)$ into three cases and analyze as follows.
\begin{itemize}[leftmargin=*]
    \item \textbf{Case 1: $\wh{s}_i>0$.} Since $\lim_{\gamma \rightarrow 0} s_\infty^i(\gamma) = \wh{s}_i>0$, from \eqref{eqn:s-infty-i} we must have $\xi_\infty^i(\gamma) \rightarrow +\infty $ when $\gamma \rightarrow 0$, so that $\exp\paren{ 2 \alpha \cdot  \xi_\infty^i(\gamma) } \rightarrow +\infty$ and $\exp\paren{ - 2 \alpha \cdot  \xi_\infty^i(\gamma) } \rightarrow 0$. Therefore, when $\gamma \rightarrow 0$, we have
    \begin{align*}
        \gamma^2 \exp\paren{2\alpha \cdot \xi_{\infty}^i(\gamma)} \;\rightarrow\; \wh s_i \quad &\Longrightarrow \quad  2\alpha \cdot \xi_{\infty}^i(\gamma)  - 2 \log \paren{1/\gamma} - \log \wh{s}_i  \;\rightarrow\; 0,  \\
        &\Longrightarrow \quad  \nu_i(\gamma) \;=\; \frac{\xi_{\infty}^i(\gamma) }{ \log \paren{1/\gamma} } \;\rightarrow \; \frac{1}{ \alpha  }. \;\; (\text{given}\;\; \log\paren{1/\gamma} \rightarrow +\infty)
    \end{align*}
    \item \textbf{Case 2: $\wh{s}_i<0$.} Since $\lim_{\gamma \rightarrow 0} s_\infty^i(\gamma) = \wh{s}_i<0$, from \eqref{eqn:s-infty-i} we must have $\xi_\infty^i(\gamma) \rightarrow -\infty $ when $\gamma \rightarrow 0$, so that $\exp\paren{ 2 \alpha \cdot  \xi_\infty^i(\gamma) } \rightarrow 0$ and $\exp\paren{ - 2 \alpha \cdot  \xi_\infty^i(\gamma) } \rightarrow +\infty$. Therefore, when $\gamma \rightarrow 0$, we have
    \begin{align*}
        -\gamma^2 \exp\paren{-2\alpha \cdot \xi_{\infty}^i(\gamma)} \;\rightarrow\; \wh s_i \quad &\Longrightarrow \quad  -2\alpha \cdot \xi_{\infty}^i(\gamma)  + 2 \log \paren{1/\gamma} - \log \wh{s}_i  \;\rightarrow\; 0,  \\
        &\Longrightarrow \quad \nu_i(\gamma) \;=\; \frac{\xi_{\infty}^i(\gamma) }{ \log \paren{1/\gamma} } \;\rightarrow \; -\frac{1}{ \alpha  }. 
    \end{align*}
    \item \textbf{Case 3: $\wh{s}_i=0$.} Since $\lim_{\gamma \rightarrow 0} s_\infty^i(\gamma) = \wh{s}_i=0$, from \eqref{eqn:s-infty-i} we must have $\gamma^2\cdot \exp\paren{2 \alpha \cdot \xi_\infty^i(\gamma) } \rightarrow 0 $ and $\gamma^2\cdot \exp\paren{-2 \alpha \cdot \xi_\infty^i(\gamma) } \rightarrow 0 $, when $\gamma \rightarrow 0$. Therefore, for any small $\epsilon \in (0,1)$, there exists some $\gamma_0>0$, such that for all $\gamma \in (0,\gamma_0)$, we have
    \begin{align*}
       & \gamma^2\cdot \max\Brac{ \exp\paren{2 \alpha \cdot \xi_\infty^i(\gamma) },\; \exp\paren{-2 \alpha \cdot \xi_\infty^i(\gamma) } }\; \leq \; \epsilon \\
       \Longrightarrow \quad & 2 \alpha \cdot \max \Brac{  \frac{\xi_{\infty}^i(\gamma) }{ \log \paren{1/\gamma} }, -\frac{\xi_{\infty}^i(\gamma) }{ \log \paren{1/\gamma} } } -2 \;<\; \frac{ \log \epsilon }{ \log \paren{1/\gamma} } \;<\; 0,
    \end{align*}
    which further implies that
    \begin{align*}
        \abs{  \nu_i(\gamma) } \;=\; \max \Brac{  \frac{\xi_{\infty}^i(\gamma) }{ \log \paren{1/\gamma} }, -\frac{\xi_{\infty}^i(\gamma) }{ \log \paren{1/\gamma} } } \;<\; \frac{1}{\alpha}.
    \end{align*}
\end{itemize}
Therefore, combining the results in the three cases above we obtain that
\begin{align*}
    \lim_{\gamma \rightarrow 0} \nu_i(\gamma) \;=\; \frac{1}{\alpha} \sign( \wh{s}_i )  \;=\; \begin{cases} 
    \frac{\wh{s}_i}{ \alpha \abs{\wh{s}_i} } & \wh{s}_i \not= 0, \\
    [-1/\alpha,1/\alpha] & \wh{s}_i =0,
    \end{cases}
\end{align*}
so that we have \eqref{eqn:nu-sign} holds.
\end{proof}

\section{Extra Experiments}\label{sec:extra-exp}

Due to limited space in the main body, we here provide extra results for our experiments on robust image recovery presented in Section~\ref{sec:experiment-image}.

\myparagraph{Varying corruption levels}
In Figure~\ref{fig:dip_varying_noise}, we display results of our method for robust image recovery with varying levels of salt-and-pepper corruption. 

\myparagraph{Learning curves for different test images and varying corruption levels}
In Figure~\ref{fig:different_figure} and Figure~\ref{fig:varying_p}, we provide learning curves for the results in Figure~\ref{fig:dip} and Figure~\ref{fig:dip_varying_noise}, respectively. 
While DIP-$\ell_1$ requires a case-dependent early termination to obtain the best performance, our method does not overfit and does not require any early termination. Overall our method achieves both better PSNR and visual quality, especially when noise level is high.

\begin{figure}
\captionsetup[subfigure]{labelformat=empty,font=footnotesize,skip=-0.9em}
	\begin{subfigure}{0.24\linewidth}
			\includegraphics[trim=13cm 20.2cm 2.5cm 20.5cm, clip,width=\linewidth]{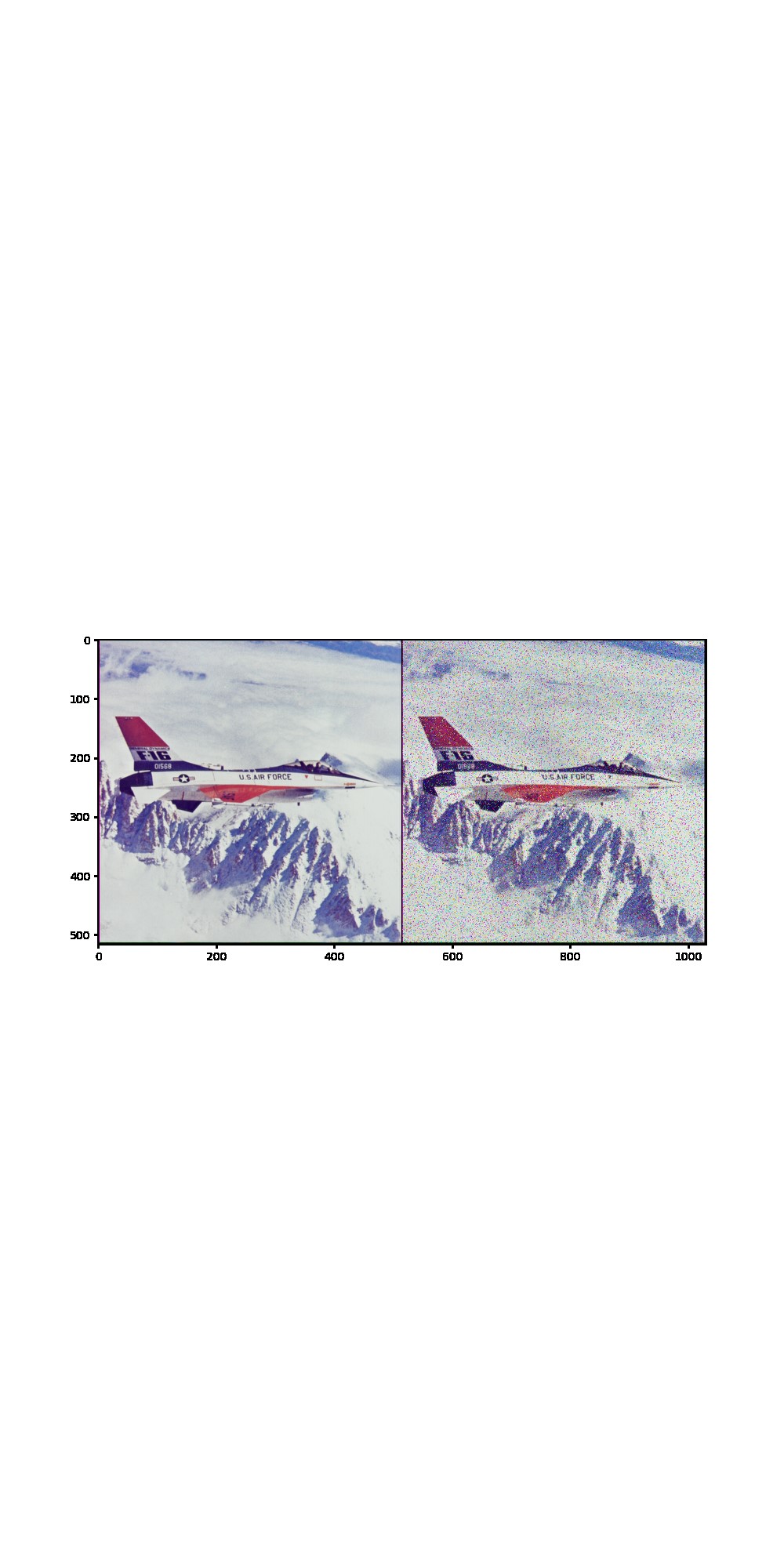}
	\end{subfigure}
	\begin{subfigure}{0.24\linewidth}
	        \includegraphics[trim=2.8cm 16.2cm 2.2cm 16.3cm, clip,width=\linewidth]{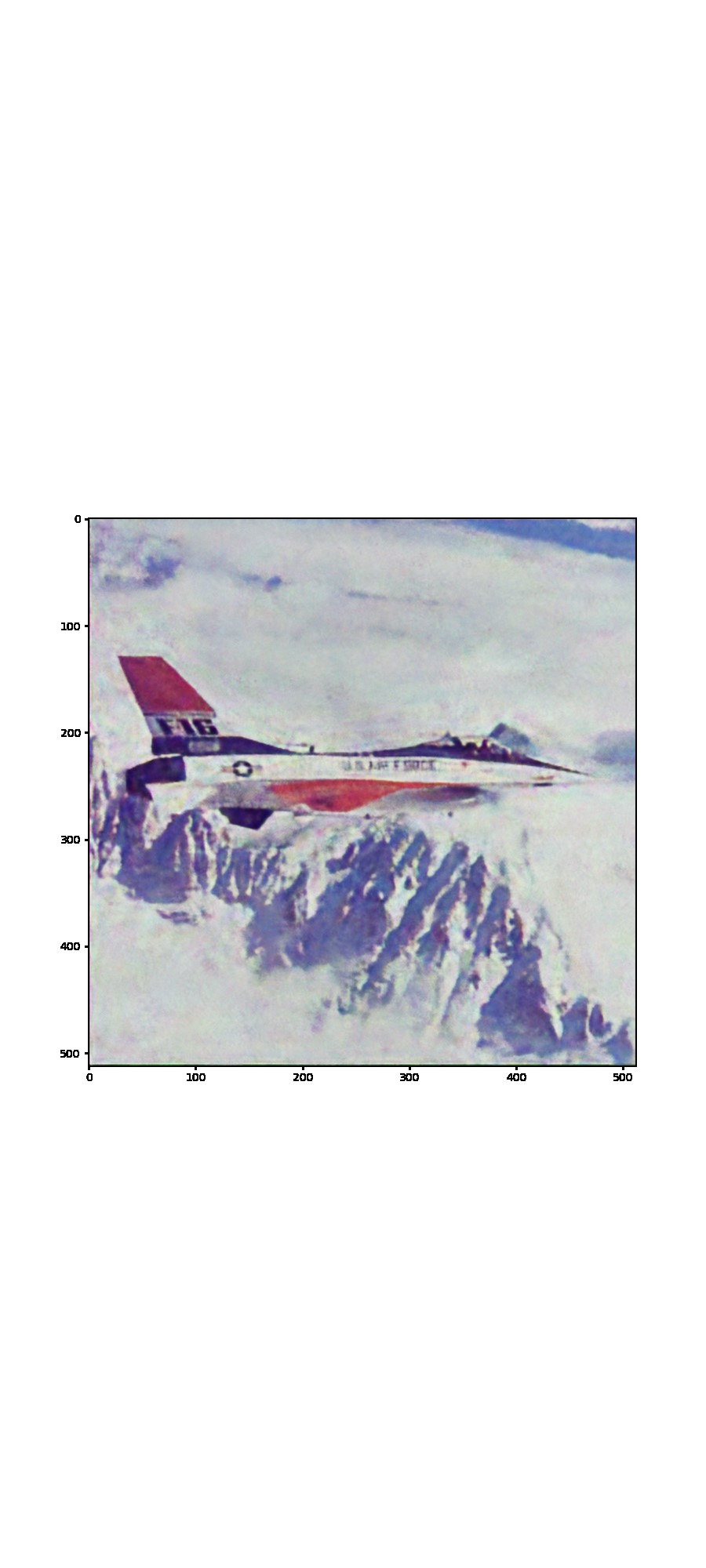}
	        \caption{PSNR $=26.19$}
	\end{subfigure}	
	\begin{subfigure}{0.24\linewidth}
	        \includegraphics[trim=2.8cm 16.2cm 2.2cm 16.3cm, clip,width=\linewidth]{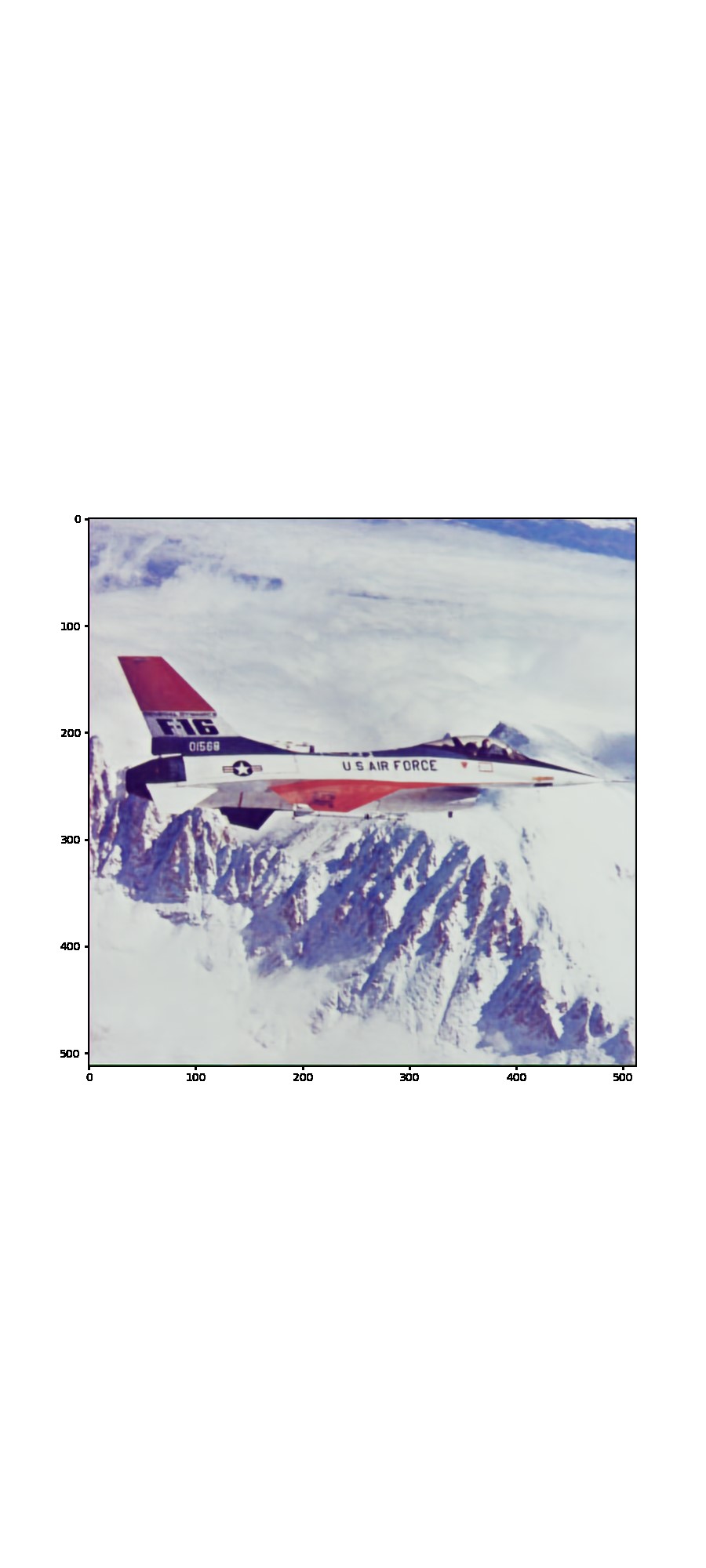}
	        \caption{PSNR $=36.15$}
	\end{subfigure}
	\begin{subfigure}{0.24\linewidth}
	        \includegraphics[trim=2.8cm 16.2cm 2.2cm 16.3cm, clip,width=\linewidth]{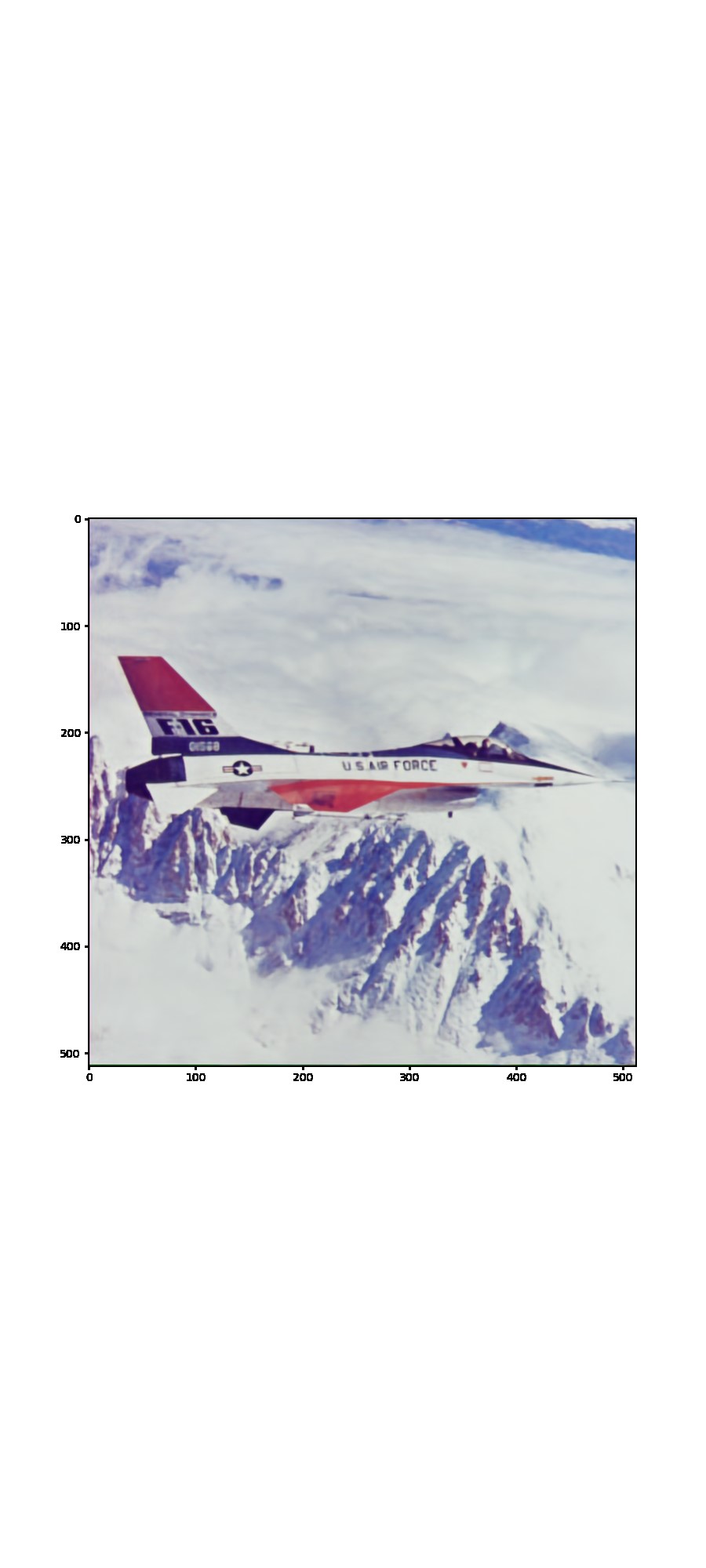}
	        \caption{PSNR $=34.91$}
	\end{subfigure}	
	\\[-0.2em]
	\begin{subfigure}{0.24\linewidth}
			\includegraphics[trim=13cm 20.2cm 2.5cm 20.5cm, clip,width=\linewidth]{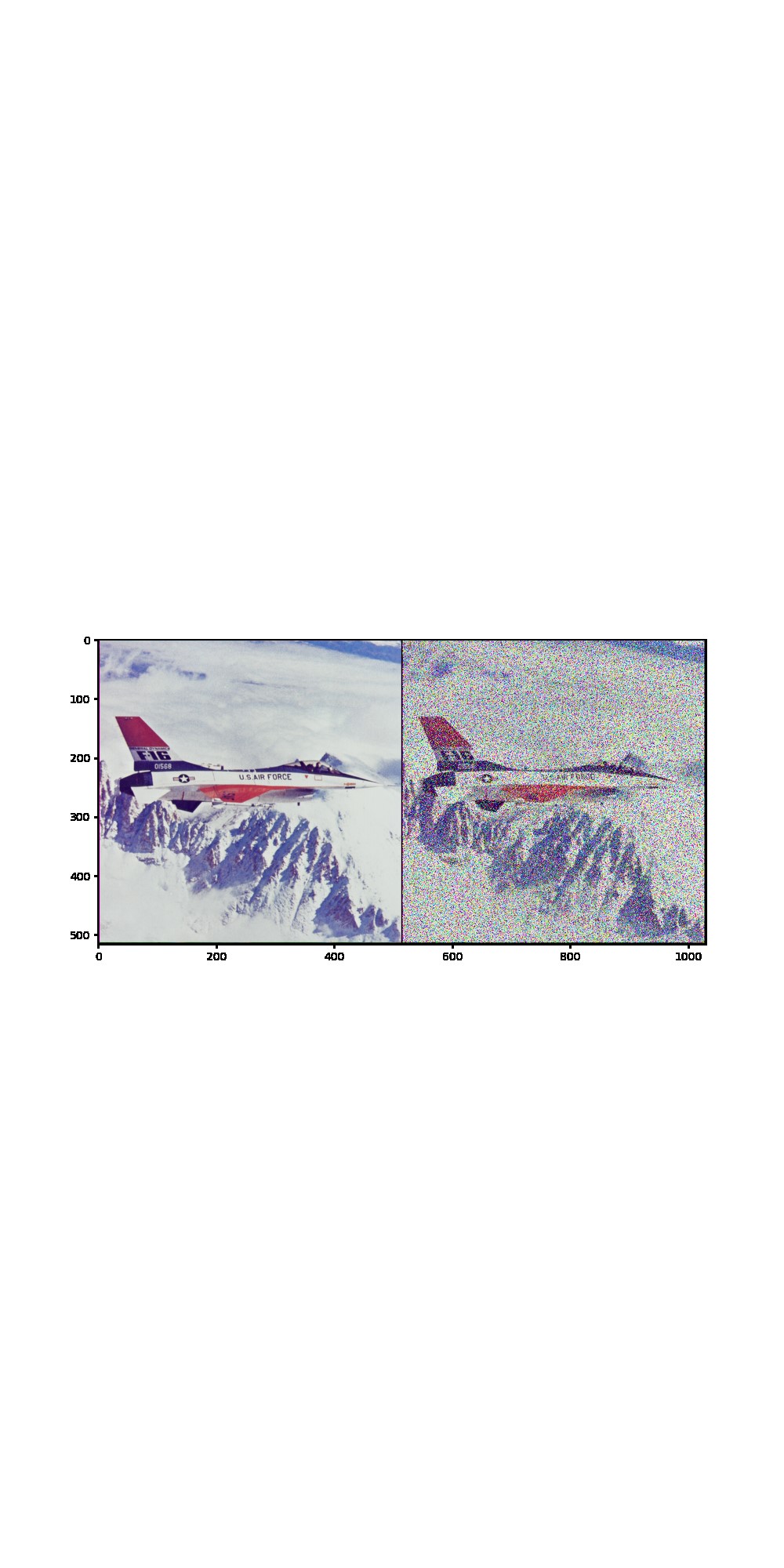}
	\end{subfigure}
	\begin{subfigure}{0.24\linewidth}
			\includegraphics[trim=2.8cm 16.2cm 2.2cm 16.3cm, clip,width=\linewidth]{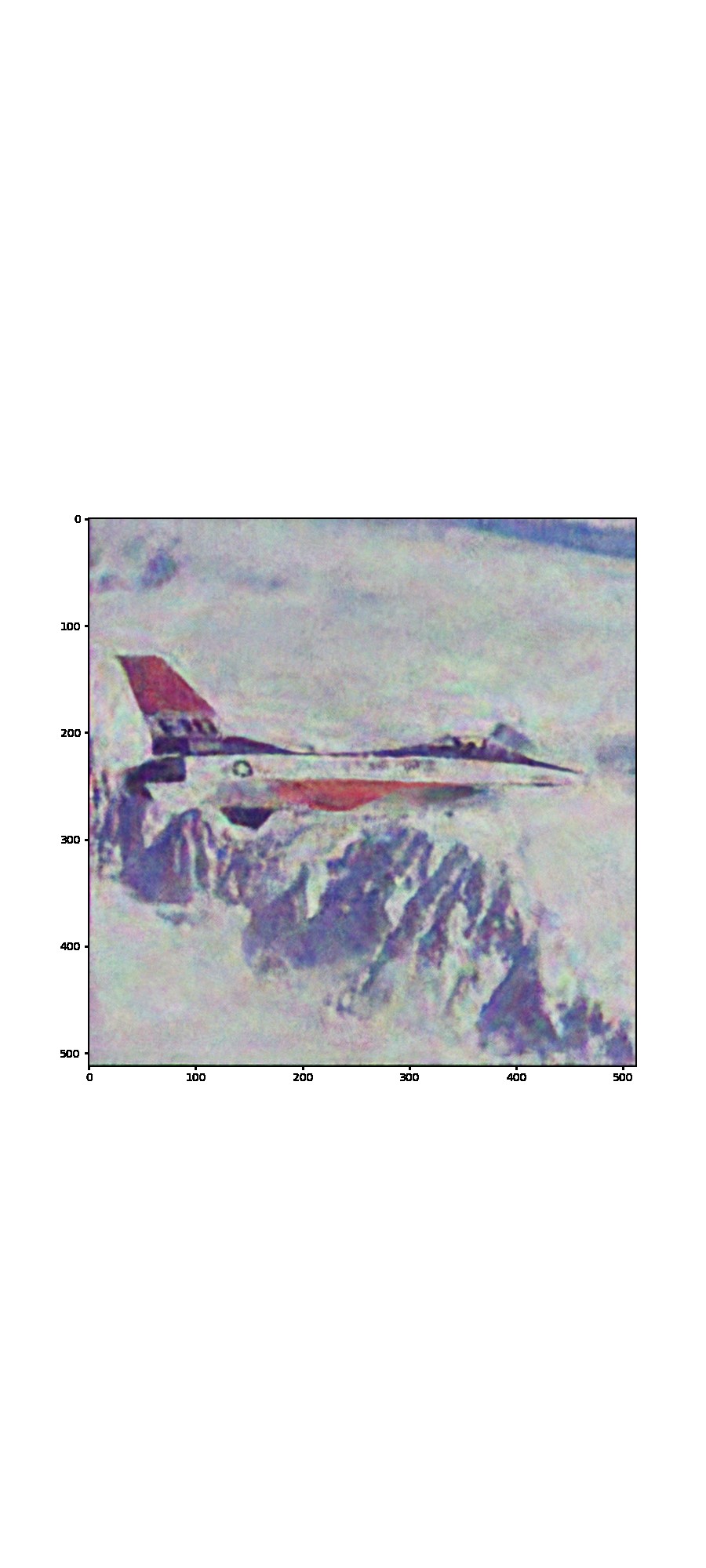}
		    \caption{PSNR $=20.15$}
	\end{subfigure}	
		\begin{subfigure}{0.24\linewidth}
			\includegraphics[trim=2.8cm 16.2cm 2.2cm 16.3cm, clip,width=\linewidth]{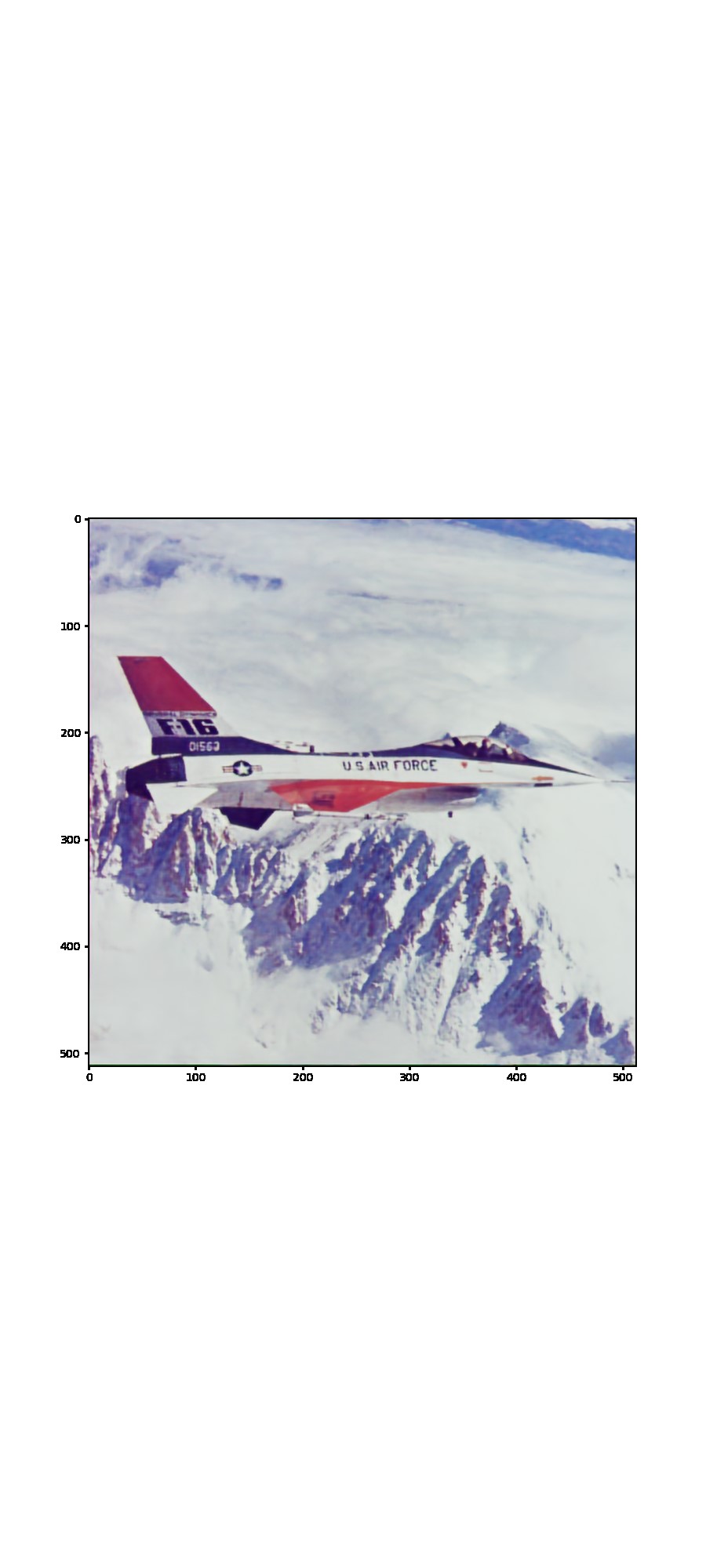}
			\caption{PSNR $=33.64$}
	\end{subfigure}
	\begin{subfigure}{0.24\linewidth}
			\includegraphics[trim=2.8cm 16.2cm 2.2cm 16.3cm, clip,width=\linewidth]{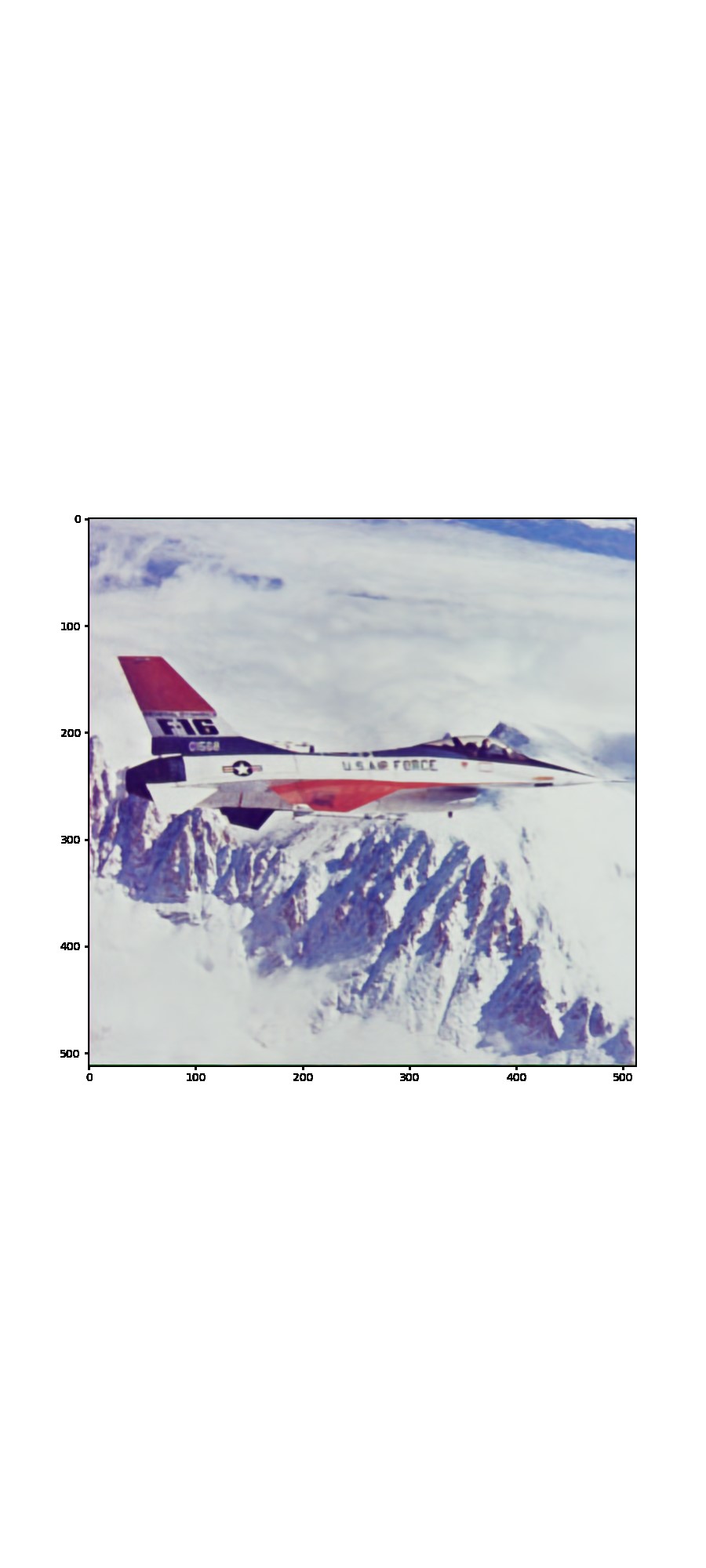}
			\caption{PSNR $=34.39$}
	\end{subfigure}	
	\\[-0.2em]
	\begin{subfigure}{0.24\linewidth}
			\includegraphics[trim=13cm 20.2cm 2.5cm 20.5cm, clip,width=\linewidth]{figs/Ours_n05_true.jpg}
	\end{subfigure}
	\begin{subfigure}{0.24\linewidth}
			\includegraphics[trim=2.8cm 16.2cm 2.2cm 16.3cm, clip,width=\linewidth]{figs/DIP_n05_best.jpg}
			\caption{PSNR $=16.46$}
	\end{subfigure}	
		\begin{subfigure}{0.24\linewidth}
			\includegraphics[trim=2.8cm 16.2cm 2.2cm 16.3cm, clip,width=\linewidth]{figs/DIPL1_n05_best.jpg}
			\caption{PSNR $=29.06$}
	\end{subfigure}
	\begin{subfigure}{0.24\linewidth}
			\includegraphics[trim=2.8cm 16.2cm 2.2cm 16.3cm, clip,width=\linewidth]{figs/Ours_n05_best.jpg}
			\caption{PSNR $=31.60$}
	\end{subfigure}	
	\\[-0.2em]
	\begin{subfigure}{0.24\linewidth}
			\includegraphics[trim=13cm 20.2cm 2.5cm 20.5cm, clip,width=\linewidth]{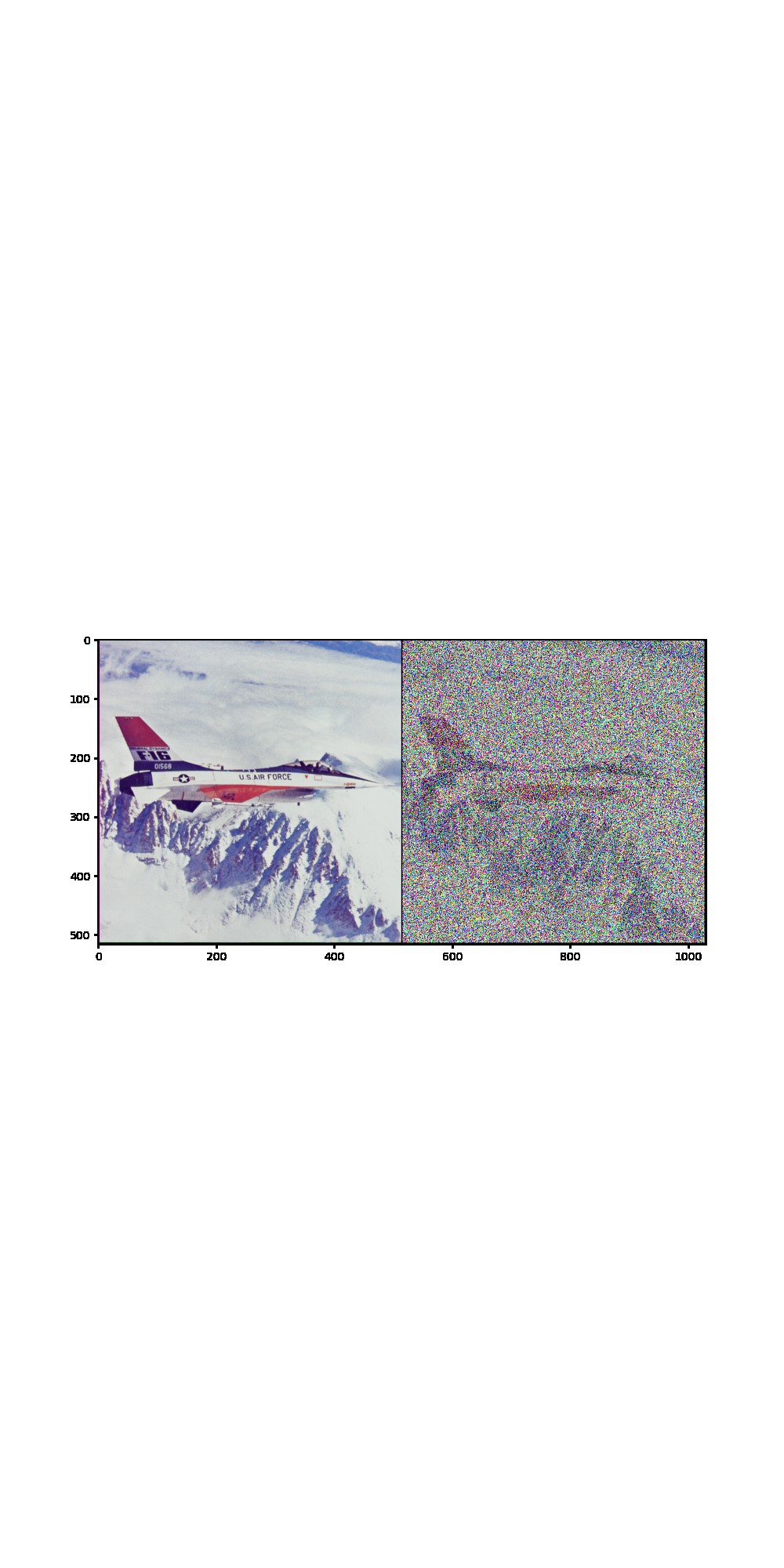}
	\end{subfigure}
	\begin{subfigure}{0.24\linewidth}
			\includegraphics[trim=2.8cm 16.2cm 2.2cm 16.3cm, clip,width=\linewidth]{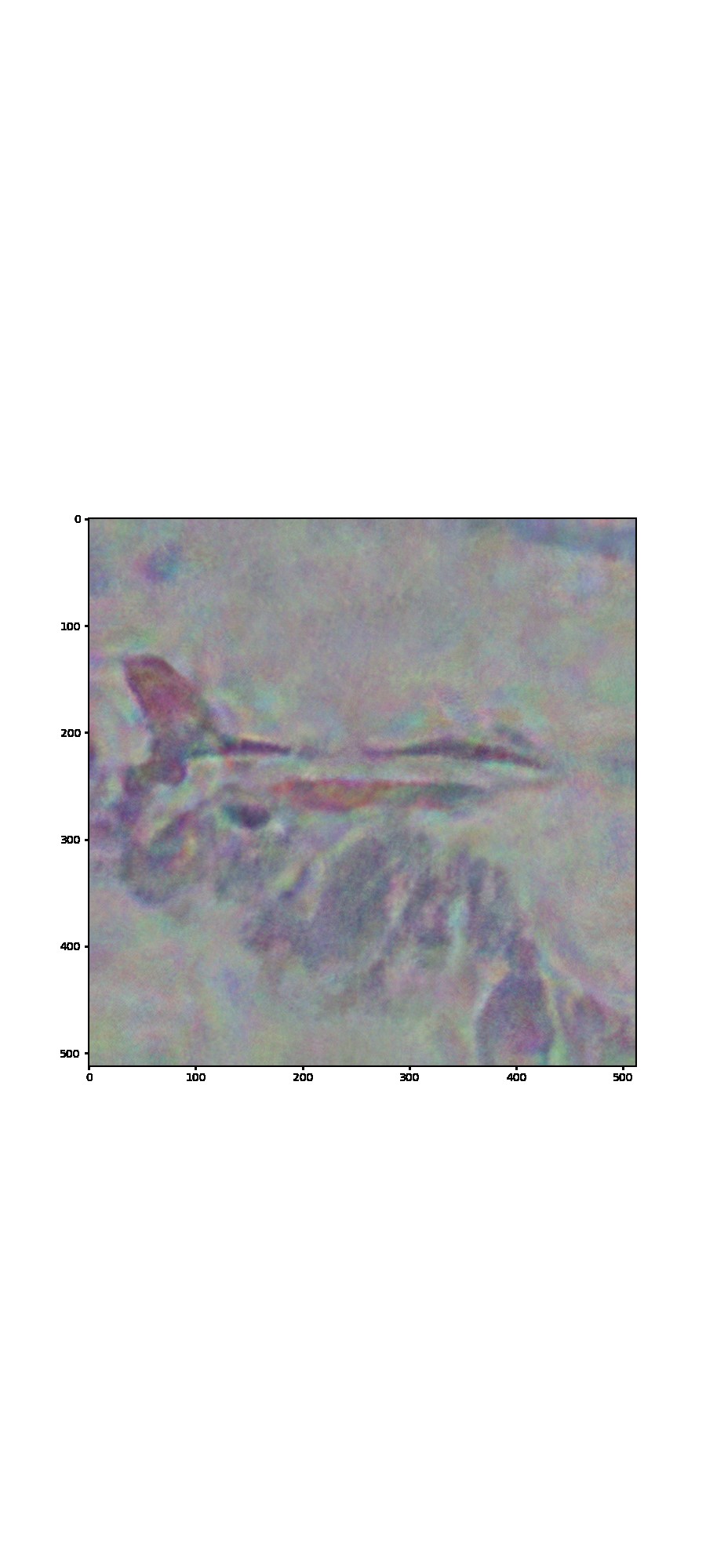}
			\caption{PSNR $=13.91$}
	\end{subfigure}	
	\begin{subfigure}{0.24\linewidth}
			\includegraphics[trim=2.8cm 16.2cm 2.2cm 16.3cm, clip,width=\linewidth]{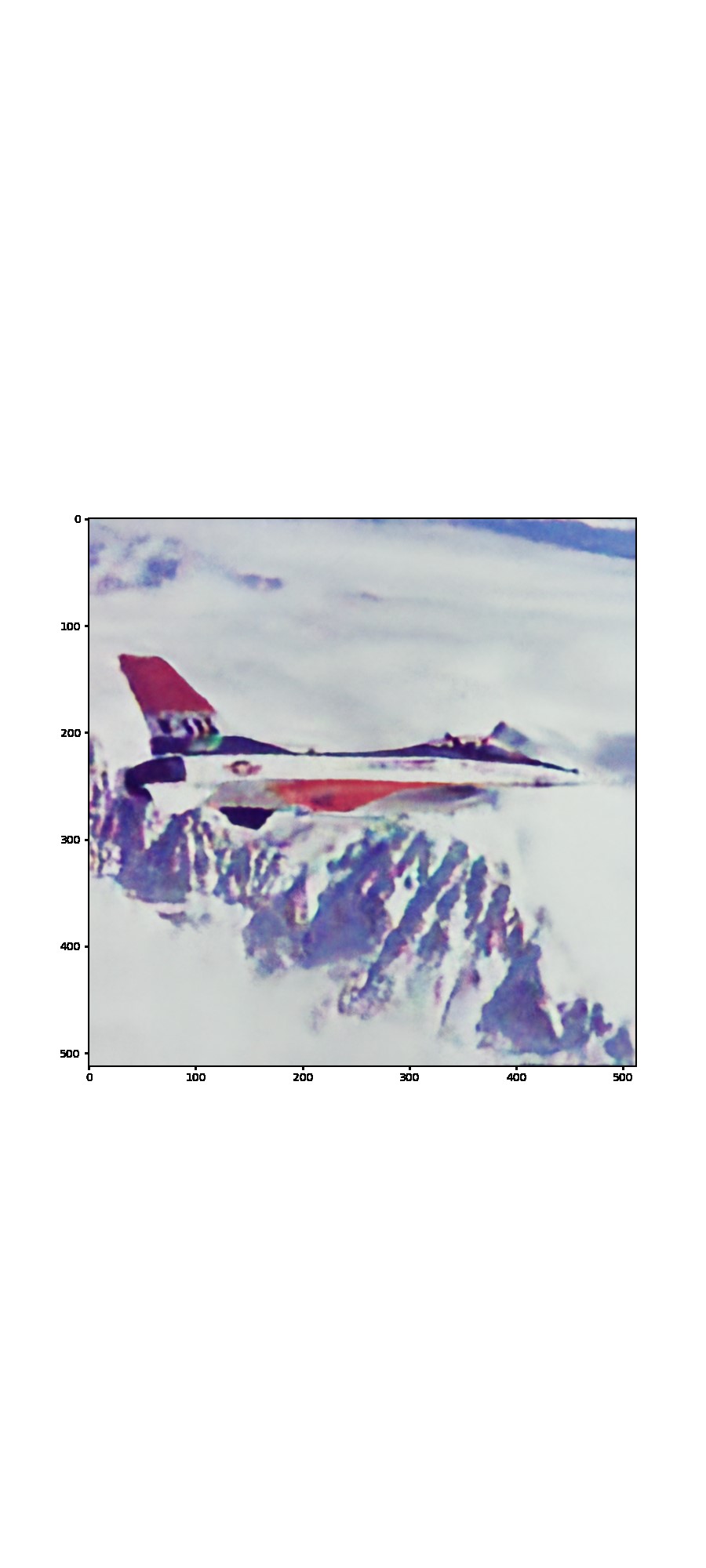}
			\caption{PSNR $=24.80$}
	\end{subfigure}
	\begin{subfigure}{0.24\linewidth}
			\includegraphics[trim=2.8cm 16.2cm 2.2cm 16.3cm, clip,width=\linewidth]{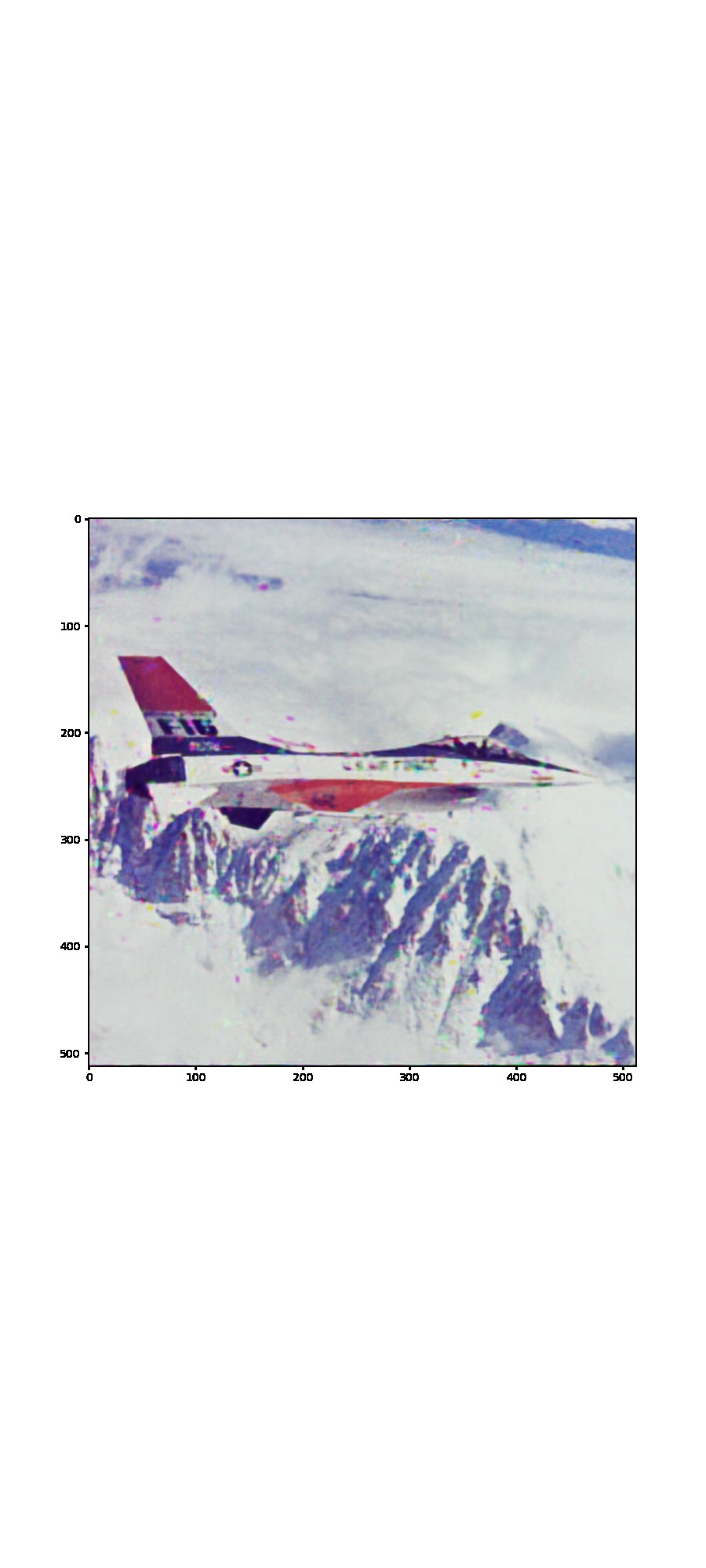}
			\caption{PSNR $=26.99$}
	\end{subfigure}
	\\
	\begin{subfigure}{0.24\linewidth}
			\caption{\small  Input}
	\end{subfigure}	
	\begin{subfigure}{0.24\linewidth}
			\caption{\small  DIP}
	\end{subfigure}	
		\begin{subfigure}{0.24\linewidth}
			\caption{\small  DIP-L1}
	\end{subfigure}	
		\begin{subfigure}{0.24\linewidth}
			\caption{\small  Ours}
	\end{subfigure}	
	\caption{\small \textbf{Robust image recovery with $10\%, 30\%, 50\%$, and $70\%$ salt-and-pepper noise. } PSNR of the results is overlaid at the bottom of the images. For our method, all cases use the same network width, learning rate, and termination condition. For DIP and DIP-$\ell_1$, case-dependent early stopping is used which is essential for their good performance. Despite that, our method achieves the highest PSNRs and best visual quality.}
	\label{fig:dip_varying_noise}
\vspace{-8pt}
\end{figure}

\begin{figure}
	\begin{subfigure}{0.45\linewidth}
		\centerline{
			\includegraphics[width=\linewidth]{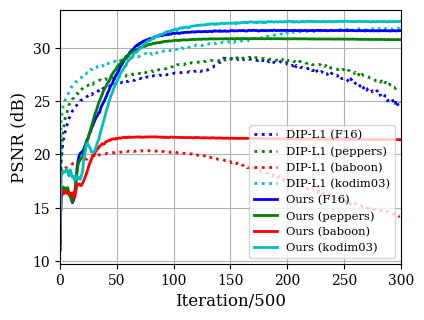}}
		\caption{Different test images}\label{fig:different_figure}
\end{subfigure}
	~~~
	\begin{subfigure}{0.45\linewidth}
		\centerline{\includegraphics[width=\linewidth]{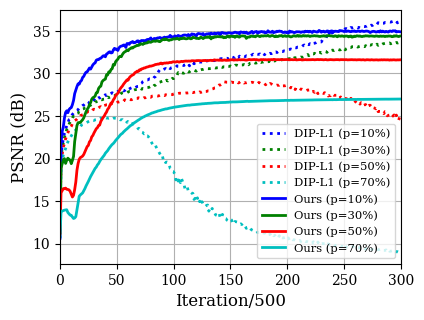}}
\caption{Varying corruption levels}\label{fig:varying_p}
	\end{subfigure}\vspace{-5pt}
	\caption{\footnotesize \textbf{Learning curves for robust image recovery with different test images (a) and varying corruption levels (b).} DIP-$\ell_1$ requires case-dependent early stopping while our method does not. 
	}\label{fig:result-learning-curves}
\vspace{-10pt}
\end{figure}

\end{appendices}

\end{document}